\pdfoutput=1
\documentclass{article}
\usepackage{iclr2027_conference,times}

\usepackage[utf8]{inputenc}
\usepackage[T1]{fontenc}
\usepackage[hypertexnames=false]{hyperref}
\hypersetup{hidelinks}
\usepackage{url}
\usepackage{nicefrac}
\usepackage{microtype}
\usepackage{xcolor}
\usepackage{amsmath,amssymb,amsfonts,bm,amsthm}
\usepackage{graphicx}
\usepackage{booktabs}
\usepackage{longtable}
\usepackage{multirow}
\usepackage{algorithm}
\usepackage{algorithmic}
\usepackage{float}
\usepackage{tikz}

\usetikzlibrary{positioning,arrows.meta,fit,backgrounds}

\newcommand{\R}{\mathbb{R}}
\newcommand{\norm}[1]{\left\lVert #1 \right\rVert}
\newcommand{\ip}[2]{\left\langle #1,#2\right\rangle}
\newcommand{\sg}{\mathrm{stopgrad}}
\newcommand{\Tr}{\mathrm{Tr}}
\newcommand{\diag}{\mathrm{diag}}
\newcommand{\1}{\mathbf{1}}
\newcommand{\cP}{\mathcal{P}}
\newcommand{\cH}{\mathcal{H}}
\newcommand{\cF}{\mathcal{L}}
\newcommand{\cA}{\mathcal{A}}
\newcommand{\cR}{\mathcal{R}}

\theoremstyle{plain}
\newtheorem{theorem}{Theorem}[section]
\newtheorem{proposition}[theorem]{Proposition}
\newtheorem{lemma}[theorem]{Lemma}

\theoremstyle{remark}
\newtheorem{remark}[theorem]{Remark}

\title{Routing in Gradient Space:\\ Balanced Usage Is Not Expert Specialization}

\author{Yuchen Li$^{1}$, Mingyu Du$^{1,2}$, Zongqi Fan$^{1}$, Nguyen H. Tran$^{1,*}$, Ken-Tye Yong$^{1,*}$ \\ $^{1}$University of Sydney \quad $^{2}$University of New South Wales \\ $^{*}$Corresponding authors}

\iclrfinalcopy

\begin{document}
\maketitle
\lhead{Preprint. Under review.}

\begin{abstract}
Sparse expert models can distribute traffic evenly while still grouping incompatible training signals within the same experts. We study routing as a gradient-partitioning problem and introduce gradient-aligned routing (GAR), whose load-normalized router objective rewards grouping observations with aligned gradients. On five multi-task text-classification mixtures, we compare GAR with task-loss-only routing, gradient-combination and gradient-conflict methods, and load-balancing losses. With a fully trainable RoBERTa backbone and classification-head experts, GAR has the highest aggregate validation accuracy, $1.07$ percentage points above task-loss-only routing. With frozen DeBERTa and Qwen3-1.7B backbones and low-rank adapter experts, it again ranks first, $1.10$ points above task-loss-only routing, with better-balanced expert load and higher gradient-mass purity, the share of each expert's gradient-norm mass from its dominant task; the load-balancing losses flatten load further but leave this purity near its task-loss-only level. Top-1 routing, trainable full-parameter feed-forward network (FFN) experts, and a larger backbone also show positive aggregate gains. The results distinguish expert-load balance from gradient-based routing organization and indicate the predictive value of gradient-informed routing in multi-task text classification.
\end{abstract}

\section{Introduction}

Multi-task learning exposes shared parameters to heterogeneous optimization signals. When gradients conflict, improving one objective can impede another. Gradient-level methods address this problem by reweighting losses or modifying the shared update, as in gradient normalization (GradNorm), multiple-gradient descent algorithm (MGDA), and conflict-averse gradient descent (CAGrad) \citep{gradnorm_chen2018,mgda_sener2018,cagrad_liu2021}. Mixture-of-Experts (MoE) introduces another decision: which observations should contribute to the same expert. This makes the organization of gradient contributions an explicit routing problem.

Solving Token Gradient Conflict (STGC) \citep{stgc_yang2024}, our closest comparison, detects tokens whose gradients conflict with their current expert mean and penalizes those assignments. Its conflict term discourages current assignments without explicitly comparing gradient-association gains across alternative groupings. We study how to evaluate an entire assignment. We call an expert coherent when the gradient observations routed to it point in similar directions, so their sum is large relative to the load they bring (Section~\ref{sec:variational_formulation}). Moving probability mass between two experts changes the coherence of both and can improve a partition even when no observation triggers a conflict test.

\paragraph{A criterion for gradient partitions.}
We formulate routing as a soft partition of gradient observations, scoring positive and negative associations in a single functional. Dividing each expert's score by its assigned mass controls reward scaling with traffic. The resulting marginal assignment scores compare the net change in partition reward at the source and receiving experts. We call routing trained with this auxiliary objective gradient-aligned routing (GAR); the task loss still trains the model.

The criterion uses gradient observations in a common expert-parameter template and differentiable routing probabilities, accommodating both classification-head and FFN experts. We obtain the observations by summing corresponding expert gradients, detach them, and optimize only the router through the auxiliary objective. This reuses task-gradient computations and requires no second-order derivatives.

We derive the marginal assignment score and relate the objective to gradient-space clustering. With fixed observations, nonempty hard assignments at zero stabilizer recover k-means, while an orthogonal relaxation yields the spectral trace objective. An idealized analysis further characterizes how load normalization controls reward scaling and favors coherent gradient groups.

\paragraph{Balanced usage and predictive value.}
\label{sec:intro_balance_not_specialization_ref}
Load uniformity concerns the marginal distribution over experts, whereas association with task or gradient structure concerns a joint distribution. Independent uniform routing is balanced in the population but contains no task information. We call this joint structure expert specialization and measure it by gradient-mass purity (task dominance in expert-gradient norms) and task-label association; balance, specialization, and accuracy answer different questions (Appendix~\ref{app:experimental_protocol}), and we compare them by empirical Pareto dominance. Task-label association is a diagnostic: aligned observations from different tasks may usefully share experts.

\paragraph{Evidence across MoE settings.}
A task mixture jointly trains on several classification datasets; the five mixtures containing five to eight tasks are defined in Appendix~\ref{app:hyperparameters_reproducibility}, Table~\ref{tab:task_mixtures}. With fully trainable backbones, a seven-method RoBERTa comparison places low-rank adaptation (LoRA) experts in the classification head with sequence-level routing, and a DeBERTa comparison uses full-parameter feed-forward network (FFN) experts with token-level routing. GAR has the highest aggregate validation accuracy in both, and also under top-1 routing of the classification-head setting.

A second group freezes DeBERTa and Qwen3-1.7B backbones and trains LoRA experts in FFN blocks. Across the five mixtures, GAR achieves the highest aggregate validation accuracy and gradient-mass purity among seven methods. It simultaneously improves accuracy, purity, and utilization and reduces load variance relative to task-loss-only routing (Baseline) and CAGrad, Pareto-dominating both in the four-metric aggregate; gains also extend to Qwen3-8B, and DeBERTa single-task controls provide complementary adaptation checks. A fixed-configuration classification-head coefficient sweep and denominator ablation probe the auxiliary objective at a shorter training budget.

\paragraph{Contributions.}
\textbf{(1) A gradient-space routing objective.} We introduce a load-normalized partition criterion that scores signed gradient associations and captures assignment improvements beyond binary conflict detection.
\textbf{(2) Theory and first-order optimization.} We derive the marginal assignment score and the criterion's connections to k-means and spectral partitioning, and develop a detached auxiliary update that trains the router using ordinary task gradients.
\textbf{(3) Evidence across MoE architectures.} We report accuracy gains across trainable and frozen backbones, head and FFN experts, and top-1 and top-4 routing. Routing trajectories and coefficient ablations further characterize its effects on expert organization.

\section{Related Work}

Prior methods address heterogeneous optimization signals through loss weighting, gradient updates, or data allocation. Our focus is an auxiliary objective that scores the allocation itself using gradient observations.

\paragraph{Gradient-level conflict mitigation.}
GradNorm \citep{gradnorm_chen2018} balances task-gradient magnitudes, PCGrad \citep{pcgrad_yu2020} projects conflicting gradients, MGDA \citep{mgda_sener2018} seeks a common descent direction, and CAGrad \citep{cagrad_liu2021} balances average descent with per-task improvement. These methods modify the shared update; a gradient-informed router changes which observations contribute to each expert. We use CAGrad as a representative gradient-combination comparison.

\paragraph{MoE routing and assignment.}
Switch Transformers \citep{switch_transformer} and DeepSpeed-MoE \citep{deepspeed_moe_rajbhandari2022} use auxiliary objectives to encourage balanced expert utilization. Routing Transformers \citep{routing_transformers_roy2021}, Sinkhorn-style sparse attention \citep{sinkhorn_transformer_tay2020}, balanced assignment layers \citep{base_layers_lewis2021}, and expert-choice routing \citep{expert_choice_zhou2022} use representation structure or assignment constraints to organize sparse computation. We study an auxiliary objective based on current-model gradient associations to score the composition of expert assignments, complementing controls on their marginal loads.

\paragraph{STGC and gradient-informed routing.}
STGC \citep{stgc_yang2024} motivates conflict detection with a first-order loss expansion, compares token gradients with the current expert mean, and penalizes conflicting assignments. Its conflict term encourages departure from the current expert without explicitly scoring alternative groupings by their gradient-association gains.

Our continuous partition criterion compares marginal coherence and normalization changes at both source and receiving experts, including when no conflict is flagged. Its fixed-observation analysis characterizes these preferences; the experiments evaluate the complete criterion and detached router pathway.

STGC adds the adapted conflict-elimination loss to the task loss; an update-level load-balancing loss (LoadPen), a micro-batch Switch-style loss (SwitchAux), and STGC with the LoadPen term (STGC+Load) are further controls (Appendix~\ref{app:hyperparameters_reproducibility}).

\paragraph{Partition geometry and positioning.}
The normalized Gram objective connects to classical clustering and spectral partitioning \citep{dhillon2004kernelkm}. These standard identities characterize our proposal to train a parameterized router from current-model gradient observations. Gradient-mass purity, task--expert normalized mutual information (NMI), and adjusted Rand index (ARI) provide distinct diagnostics \citep{strehl2002nmi,hubert1985ari}; frozen LoRA-FFN and trainable classification-head settings test predictive value across model configurations.

\section{Variational Formulation of Gradient-Space Partitioning}
\label{sec:variational_formulation}
We formalize routing as a partitioning problem over gradient space. A single variational objective connects clustering and spectral partitioning under the stated idealized restrictions; our practical training loss uses its stabilized, zero-entropy-coefficient form. Proofs are in Appendices~\ref{app:variational_derivations} and~\ref{app:additional_derivations}.

\subsection{Assignment simplex, expert loads, and affinity}
Consider $M$ routed items, each associated with a gradient observation. An item is a unit of training data with one routing distribution and one gradient observation: a token, an example, or, in all our experiments, a same-task group of up to eight training examples with group-averaged routing probabilities (Section~\ref{sec:training_detached}). The router assigns each item to a distribution over $K$ experts, giving
\[
P := [p_{mk}] \in \R_+^{M\times K},
\qquad
\cP_{M,K}
:=
\left\{
P\in \R_+^{M\times K}
\;\middle|\;
P\1_K=\1_M
\right\},
\]
so each row lies on the probability simplex. The load of expert $k$ and its stabilized diagonal matrix are
\[
d_k(P):=\sum\nolimits_{m=1}^M p_{mk},
\qquad
D_\epsilon(P):=\diag \big(d_1(P)+\epsilon,\dots,d_K(P)+\epsilon\big),
\qquad \epsilon\geq0 .
\]
Let $W\in\R^{M\times M}$ be a symmetric affinity matrix, where $W_{ij}$ is large when routed items $i$ and $j$ induce similar gradient directions and small or negative when their gradients conflict; Section~\ref{Sec:Euc} instantiates it from gradients. Together, $P$ and $W$ determine how much similarity mass each expert collects,
$Q(P;W):=P^\top W P \in \R^{K\times K}$, whose diagonal entry $Q_{kk}$ is the affinity mass assigned to surrogate group $k$. Under the gradient Gram instantiation below, it measures the squared norm of an aggregate in a common parameter template, rather than the actual update to physical expert $k$.

\subsection{Load-normalized partition objective}
The unnormalized affinity $Q_{kk}$ grows with both gradient agreement and the number of assigned observations. Dividing by assigned mass separates the per-observation reward from this quadratic load scaling. Summing over experts gives the load-normalized association
\begin{equation}
\cA_\epsilon(P;W)
:=
\Tr\!\left(D_\epsilon(P)^{-1}P^\top W P\right)
=
\sum_{k=1}^K \frac{Q_{kk}(P;W)}{d_k(P)+\epsilon},
\label{eq:association_reward}
\end{equation}
which favors assignments where routed items sharing an expert have similar gradient observations, measured per unit load. Adding an entropic relaxation yields the variational objective
\begin{equation}
\cF_{\tau,\epsilon}(P;W)
:=
-\cA_\epsilon(P;W)
+
\tau \sum_{m=1}^M \sum_{k=1}^K p_{mk}\log p_{mk},
\qquad
P\in \cP_{M,K},
\label{eq:free_energy}
\end{equation}
where $\tau \ge 0$ controls entropic smoothing. We set $\tau=0$, omitting the
explicit entropy term while still using differentiable router probabilities as
the continuous surrogate through which the router receives gradients. Hard
assignments are analyzed separately by restricting $P$ to the nonempty
hard-assignment set; $\tau=0$ alone does not impose one-hot routing.

\subsection{Euclidean gradient instantiation} \label{Sec:Euc}
We instantiate $W$ using a common expert-parameter template. Each expert has selected parameters $\theta_e\in\R^d$ with the same names, shapes, and local ordering. We sum corresponding gradient entries across experts to obtain the observation of item $m$
\[
g_m:=\sum_{e=1}^{K}\nabla_{\theta_e}\ell_m\in\R^d,
\qquad \tilde g_m=\sg(g_m).
\]
Here $\ell_m$ is the mean task loss over the examples in item $m$, and $d$ is the selected parameter dimension of one expert. For LoRA, the $A$ and $B$ matrices occupy separate template slots. Equivalently, $g_m$ is the loss derivative with respect to a common additive perturbation $\Delta\in\R^d$ applied to every $\theta_e$, evaluated at $\Delta=0$; the experts remain independently parameterized.

The affinity is $W_{ij}:=\ip{\tilde g_i}{\tilde g_j}$, including cross-expert inner products in matching template coordinates. Different routing supports therefore need not yield zero affinity. The criterion forms $G_k:=\sum_m p_{mk}\tilde g_m$, a \emph{surrogate aggregate in the template space}. If routed unit $n$ (a token or an example) has output $y_n$ containing $\sum_e\pi_{ne}u_{ne}$, with gate $\pi_{ne}$ and expert output $u_{ne}$, the chain rule gives $\nabla_{\theta_e}\ell_m=\sum_{n\in m}(\partial u_{ne}/\partial\theta_e)^\top\pi_{ne}\nabla_{y_n}\ell_m$. Thus $\tilde g_m$ is evaluated at the current gates and expert Jacobians $\partial u_{ne}/\partial\theta_e$; detachment holds it fixed during the auxiliary update, and it is not the gradient another assignment would produce.

\begin{proposition}[Euclidean instantiation]
\label{prop:euclidean_instantiation}
Under the gradient Gram affinity $W$, the association reward in \eqref{eq:association_reward} reduces to
\begin{equation}
\cA_\epsilon(P;W)
=
\sum_{k=1}^K \frac{\norm{G_k}^2}{d_k(P)+\epsilon}.
\label{eq:euclidean_instantiation}
\end{equation}
\end{proposition}

\noindent This follows from $Q_{kk}=\sum_{i,j} p_{ik}p_{jk}\ip{\tilde g_i}{\tilde g_j}=\norm{G_k}^2$. The full Gram matrix retains its diagonal, so the criterion reflects gradient norms and self-association as well as signed agreement between distinct observations (Appendix~\ref{app:self_cross_association}). Aligned contributions increase $\|G_k\|^2$, while conflicting contributions can cancel within the surrogate aggregate.

\paragraph{Partition preference beyond conflict detection.}
Consider three fixed unit-norm observations $\tilde g_1,\tilde g_2,\tilde g_3$ and two nonempty hard-assignment clusters with $\epsilon=0$. Placing $\tilde g_i$ and $\tilde g_j$ together and the third observation alone gives $\cA_0=2+\ip{\tilde g_i}{\tilde g_j}$. Every partition has the same cluster-size profile and diagonal contribution $2$, so the three partitions differ only through the cross inner products $\ip{\tilde g_1}{\tilde g_2}$, $\ip{\tilde g_1}{\tilde g_3}$, and $\ip{\tilde g_2}{\tilde g_3}$. An STGC-style conflict test flags observation $m$ in cluster $C$ when
\[
\cos(\tilde g_m,\mu_C)<0,
\qquad
\mu_C:=\frac{1}{|C|}\sum_{j\in C}\tilde g_j .
\]
If all three cross inner products are positive, then $\ip{\tilde g_m}{\mu_C}=|C|^{-1}\big(1+\sum_{j\in C\setminus\{m\}}\ip{\tilde g_m}{\tilde g_j}\big)>0$ for every $m$ and every partition, so no observation is flagged. The objective still ranks the partitions: $\ip{\tilde g_1}{\tilde g_2}=0.9$ and $\ip{\tilde g_1}{\tilde g_3}=\ip{\tilde g_2}{\tilde g_3}=0.2$ give $\cA_0=2.9$ for $\{\{1,2\},\{3\}\}$ and $2.2$ for the other two partitions. This establishes a distinct preference at equal gradient norms and self-association.

\paragraph{Marginal assignment preference and soft routing.}
With observations fixed, the marginal reward is
$s_{mk}:=\partial\cA_\epsilon/\partial p_{mk}=2\ip{G_k}{\tilde g_m}/(d_k+\epsilon)-\norm{G_k}^2/(d_k+\epsilon)^2$.
A feasible transfer of mass $\delta$ from expert $a$ to $b$ changes the reward by $\delta(s_{mb}-s_{ma})+O(\delta^2)$. It compares compatibility and normalization changes at both experts, including when no conflict is flagged. This score evaluates surrogate partition reward at $\tau=0$. For $\tau>0$, every interior stationary point has Gibbs probabilities proportional to $\exp(s_{mk}/\tau)$ (Appendix~\ref{app:gibbs_form_proof}, Proposition~\ref{prop:gibbs_form}). If the scores converge with a unique maximizer as $\tau\to0$, its probability tends to $1$.

\section{Theoretical Analysis}
We analyze the variational objective from three angles---optimization geometry,
clustering structure, and routing dynamics. For the discrete geometric results,
we set $\tau=0$ and restrict the domain to nonempty hard assignments. All proofs are in
Appendices~\ref{app:additional_derivations} and~\ref{app:dynamics}.

\subsection{Alignment favors coherent observation clusters}

\begin{proposition}[Alignment maximizes within-expert coherence] \label{prop:pairwise}
Under hard assignments, minimizing $\cF_{0, \epsilon}(P;W)=-\cA_{\epsilon}(P; W)$ is equivalent to maximizing load-normalized within-cluster coherence of the gradient observations.
\end{proposition}

\noindent Pairwise inner products reward compatible gradient observations and penalize conflicting ones within each surrogate aggregate $G_k$ defined in Section~\ref{Sec:Euc}.

\subsection{Hard-routing and spectral cases}

\begin{proposition}[k-means equivalence; Thm.~\ref{thm:kmeans_equiv}]\label{prop:kmeans_equiv_informal}
Under Euclidean gradient-Gram affinity, for nonempty hard assignments, minimizing $\cF_{0, 0}(P;W)$ is equivalent to minimizing the k-means objective over gradient vectors.
\end{proposition}

\begin{proposition}[Spectral relaxation; Thm.~\ref{thm:spectral_relaxation}]\label{prop:spectral_relaxation_informal}
For nonempty hard assignments, minimizing $\cF_{0, 0}(P;W)$ corresponds to a ratio-association objective whose standard orthogonal relaxation is
$\max_{Y^\top Y = I_K} \mathrm{Tr}(Y^\top W Y)$,
with solution given by the top-$K$ eigenvectors of $W$.
\end{proposition}

\noindent For fixed observations, nonempty hard assignments, and $\epsilon=0$, these identities characterize the geometry of the partition objective. Our method trains a parameterized router with its differentiable, stabilized form alongside the task loss.

\subsection{Routing dynamics and alignment geometry}
\label{sec:routing_dynamics_geometry}
Appendix~\ref{app:dynamics} analyzes load scaling and assignment geometry under a symmetric prototype model. First, \emph{load-normalized anti-amplification} (Proposition~\ref{prop:case1_antiamplification}) converts the quadratic expected aggregate norm into asymptotically linear utility under independent finite-window sampling, with a finite limiting expected marginal score. Second, \emph{static directional tilt and coherence preference} (Proposition~\ref{prop:formal_self_reinforcing_specialization}) shows that enriching a fixed mixture toward one prototype increases its directional alignment and expected coherence relative to uniform mixing. Third, at $\epsilon=0$, a mode-separating perturbation of continuous assignments strictly improves the leading-order reward at a collapsed assignment (Theorem~\ref{thm:vertex_unstable}). These results characterize the load and coherence preferences of the auxiliary criterion; Section~\ref{sec:empirical_summary} evaluates the jointly trained model.
\section{Method: Gradient-Aligned Routing (GAR)}
\label{sec:methodology_implementation}

The implementation separates the standard task computation from a detached
alignment pathway whose auxiliary loss has no direct gradient to non-router
parameters.

\subsection{Implementation Overview}

Figure~\ref{fig:method_overview_compact} shows the ordinary task-loss channel
and the detached alignment channel, whose auxiliary gradient is directed to
the router. The latter implements a
stochastic first-order surrogate of \eqref{eq:free_energy} under the Euclidean
gradient instantiation. Appendix~\ref{app:training_flow} gives the full training flow.

\begin{figure}[!ht]
\centering
\resizebox{\linewidth}{!}{%
\begin{tikzpicture}[
box/.style={draw, rounded corners, align=center, minimum width=25mm, minimum height=6.5mm, font=\small},
branch/.style={draw, rounded corners, align=center, minimum width=31mm, minimum height=8mm, font=\small},
arrow/.style={-Stealth, thick},
aux/.style={-Stealth, thick, dashed}]
\node[box] (batch) at (0,0) {Batch};
\node[box] (router) at (3.5,0) {Router probabilities};
\node[box] (experts) at (7.0,0) {Expert dispatch};
\node[box] (task) at (10.5,0) {Task loss};
\node[branch] (align) at (3.5,-1.15) {$\mathcal L_{\mathrm{norm}}(p,\tilde g)$\\Direct gradient: router};
\node[branch] (grad) at (7.5,-1.15) {Detached template gradients\\$\tilde g_m=\sg(\sum_e\nabla_{\theta_e}\ell_m)$};
\draw[arrow] (batch) -- (router);
\draw[arrow] (router) -- (experts);
\draw[arrow] (experts) -- (task);
\draw[aux] (task.south) |- (grad.east);
\draw[aux] (grad.west) -- (align.east);
\draw[aux] (router.south) -- node[right,font=\scriptsize] {group mean $p_m$} (align.north);
\end{tikzpicture}%
}
\caption{Shared task pathway (solid) and added alignment pathway (dashed). Expert dispatch uses top-4 routing over eight experts. Detached group gradients and recomputed group probabilities form an auxiliary loss with no direct gradient to non-router parameters; task loss trains the predictive model. All resulting gradients subsequently undergo shared global clipping.}
\label{fig:method_overview_compact}
\end{figure}
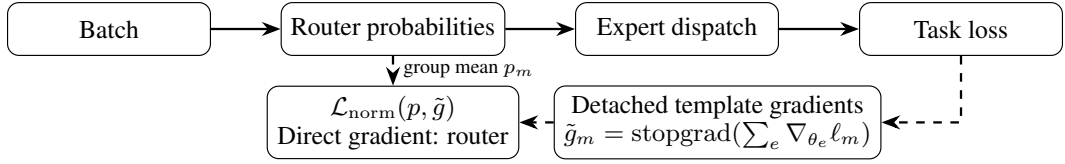

We use $\tau=0$ with differentiable router probabilities and no explicit
entropy term. The assignment score in Section~\ref{sec:variational_formulation}
accounts for both gradient compatibility and changes in assigned mass.

\subsection{Training with First-Order Detached Alignment}
\label{sec:training_detached}

Within each setting, the compared methods share the routed forward computation
and supervised task losses, which weight tasks equally and examples equally
within each task (Appendix~\ref{appsubsec:training_semantics}).
Each control applies its stated auxiliary objective or gradient-update rule.
The criterion takes paired observations $(\tilde g_m,p_m)$: a detached gradient
vector and a routing-probability vector for the same observation unit. Its
definition accommodates a specified token, example, or group granularity
(Appendix~\ref{app:route_aligned_observations}).
Our experiments use same-task micro-batches, reusing the gradients computed
during task-gradient accumulation. Thus $m$ indexes one example group within
the update. Let
$\ell_m$ be the ordinary loss averaged over that chunk. For LoRA experts, the
gradient observation $g_m\in\R^d$ sums group-loss gradients across experts
in the common adapter-parameter template defined in
Section~\ref{Sec:Euc}. The
observations are obtained within the same ordinary training batch and from the
same task loss used by Baseline. Let $q_n$ be the router summary for
example $n$. The token-routed FFN settings use top-$k$ masked softmax;
the top-1 extension uses a hard one-hot forward gate with a full-softmax
straight-through derivative (Appendix~\ref{appsubsec:top1_roberta_protocol}).
In both cases, $q_n$ averages the configured gates over non-padding tokens.
For sequence-routed classification-head experts,
$q_n$ is the configured per-example router probability vector. The auxiliary
group probability is
\[
p_m=\frac{1}{|m|}\sum_{n\in m}q_n.
\]
The group mean is used directly, so $\sum_k p_{mk}=1$ and its support is the
union of the contributing routes. Token- or example-level top-$k$ gating
and group-level gradient observations therefore have distinct roles. Details are in
Appendix~\ref{app:hyperparameters_reproducibility}.

Differentiating through $g_m$ would introduce second-order terms. We detach
$\tilde{g}_m=\text{stopgrad}(g_m)$ and recompute the auxiliary router summaries
from detached routing inputs, as detailed in
Appendix~\ref{app:hyperparameters_reproducibility}. The first operation avoids
higher-order derivatives; the second restricts this auxiliary pathway to router
parameters. Each auxiliary update treats the current-model observations as
fixed. The alignment regularizer is
\[
\mathcal{L}_{\text{norm}}
=
-\sum_{k=1}^{K}
\frac{\left\|\sum_{m=1}^{M}p_{mk}\tilde g_m\right\|^2}{d_k(P)+\epsilon},
\qquad
d_k(P)=\sum_{m=1}^{M}p_{mk}.
\]
This objective favors assignments whose current-model gradient observations are coherent within each surrogate aggregate.

The final update uses one optimizer step. Before clipping, all trainable
non-router parameters receive gradients only from $\mathcal{L}_{\text{task}}$,
while router parameters receive gradients from the combined objective
$\mathcal{L}_{\text{task}}+\lambda\mathcal{L}_{\text{norm}}$, where
$\lambda\equiv\lambda_{\mathrm{align}}$ is the alignment coefficient reported
in the configuration tables (see
Algorithm~\ref{alg:training}). The resulting gradients are then globally
clipped over all trainable parameters before the optimizer step. This produces
a first-order, Hessian-free alignment signal for the router.

Algorithm~\ref{alg:training} and Figure~\ref{fig:training_flow} in
Appendix~\ref{app:training_flow} state and illustrate the resulting
training step. At the objective level, setting $\lambda=0$ recovers Baseline
under the shared architecture, task-batch construction, routed forward
computation, and update semantics. The comparator configurations and
the RoBERTa classification-head and LoRA-FFN coefficient-ablation protocols are specified together in
Appendix~\ref{app:hyperparameters_reproducibility}.

\paragraph{Computational cost.}
Detachment avoids second-order differentiation, while the alignment pathway
adds group-gradient observations and router-summary recomputation. In
DeBERTa measurements, GAR takes $1.04\times$ Baseline's mean step time on the
frozen LoRA-FFN five-task mixture (five seeds, 20 gradient observations per
update) and $1.12\times$ and $1.03\times$ in single-task MRPC frozen-LoRA and
full-FFN runs (Appendix~\ref{appsubsec:runtime_overhead_deberta_single},
Table~\ref{tab:runtime_overhead_deberta_single}).

\section{Empirical Results}
\label{sec:empirical_summary}

\subsection{Protocol and reporting scope}

We compare Baseline, CAGrad, GAR, STGC, LoadPen (load-balancing loss), SwitchAux (Switch auxiliary loss), and STGC+Load on five task mixtures with five seeds, eight experts, and top-4 routing (E8K4). The mixtures contain five tasks, six tasks, two distinct sets of seven tasks, and eight tasks. Table~\ref{tab:task_mixtures} in Appendix~\ref{app:hyperparameters_reproducibility} lists the datasets; the two seven-task mixtures are distinguished by their PAWS or MRPC task. Each mixture is trained jointly with one optimizer configuration across its tasks. Compared methods share data, expert topology, update budget, and evaluation protocol.

\paragraph{Trainable-backbone settings.}
A seven-method RoBERTa evaluation uses a fully trainable backbone, classification-head LoRA experts over a shared base head, and sequence-level routing on pooled representations. This setting records ten-checkpoint routing trajectories, including structure purity from the selection counts used for NMI/ARI (Appendices~\ref{appsubsec:implb_protocol} and~\ref{appsubsec:implb_mechanism_trajectories}); a top-1 variant keeps the same router and experts but dispatches each example to one expert through a straight-through gate (Appendix~\ref{appsubsec:top1_head_protocol}). A second setting trains DeBERTa fully with full-parameter FFN experts and token-level routing (Appendix~\ref{appsubsec:implc_protocol}). DeBERTa single-task controls (four experts, top-2 routing; E4K2) compare frozen-backbone LoRA-FFN experts with full-parameter FFN experts at the same sites in an unfrozen backbone.

\paragraph{Frozen LoRA-FFN settings.}
The seven methods are also compared on frozen DeBERTa and Qwen3-1.7B with final-layer LoRA-FFN experts and token-level routing, with a three-method Qwen3-8B extension. Frozen RoBERTa-base has limited downstream accuracy under the shared configuration; it is reported separately and hosts a top-1 extension and a coefficient sweep, with results in Appendices~\ref{appsubsec:roberta_multitask_tables}, \ref{appsubsec:top1_roberta_results}, and~\ref{appsubsec:lambda_ablation_results}.

\paragraph{Evaluation and metrics.}
Evaluation splits and selected configurations are in Appendix~\ref{app:hyperparameters_reproducibility}. Equal-task macro validation accuracy weights tasks equally regardless of validation-set size. Load variance (LVar) measures marginal imbalance; utilization is the fraction of experts whose load reaches at least half the uniform-load level; gradient-mass purity measures per-expert concentration of task-gradient norms. NMI/ARI count expert selections per non-padding token for FFN routing and per example for classification-head routing. FFN loads summarize probability mass, while classification-head loads summarize selection frequency; their absolute values are interpreted within each setting. Appendix~\ref{app:experimental_protocol} gives the formula and aggregations.

\subsection{Joint performance across backbones and adaptation settings}
\label{subsec:conflict_heavy}

Table~\ref{tab:main_endpoint_controls}(a) reports five-mixture mean accuracy and the paired GAR--Baseline gain in each setting; panel (b) gives seven-method accuracy and routing statistics for the frozen two-backbone mean. The top-4 frozen LoRA-FFN and classification-head comparisons include all seven methods. The top-1 extensions, trainable DeBERTa-FFN, Qwen3-8B, and single-task controls compare Baseline, CAGrad, and GAR.

\begin{table}[!t]
\centering
\begingroup
\centering
\caption{Endpoint comparisons over the five mixtures in Table~\ref{tab:task_mixtures} (five seeds). (a) Mean accuracy and paired GAR$-$Baseline gains ($\Delta_{\text{Base}}$, pp; 95\% Student-$t$ intervals across seeds). $\dagger$: trainable backbone; RoBERTa-head uses sequence-routed classification-head LoRA experts (three of seven methods shown); Head, top-1 uses straight-through one-hot gating (Appendix~\ref{appsubsec:top1_head_results}); DeBERTa-FFN uses token-routed full-parameter FFN experts. Frozen-LoRA averages frozen DeBERTa and Qwen3-1.7B; Qwen3-8B is a frozen model-scale extension. $\ddagger$: DeBERTa single-task E4K2 controls average [CoLA, MRPC, RTE, SST-2, QQP], weighting frozen-LoRA and trainable-FFN adaptations equally. (b) Seven-method frozen LoRA-FFN two-backbone means. Bold marks the best accuracy in (a) and best metric values in (b). Full details are in Appendix~\ref{app:full_empirical_results}.}
\label{tab:main_endpoint_controls}
\medskip
\setlength{\tabcolsep}{2.4pt}
\begin{tabular}{@{}lrrrrrr@{}}
\toprule
(a) Method & \shortstack{RoBERTa-\\head$^{\dagger}$} & \shortstack{Head,\\top-1$^{\dagger}$} & \shortstack{DeBERTa-\\FFN$^{\dagger}$} & \shortstack{Frozen-\\LoRA} & Qwen3-8B & \shortstack{Single-\\task$^{\ddagger}$} \\
\midrule
Baseline & 0.7902 & 0.8039 & 0.8478 & 0.7593 & 0.7328 & 0.8820 \\
CAGrad & 0.7824 & 0.8066 & 0.8507 & 0.7625 & 0.7401 & 0.8799 \\
GAR & \textbf{0.8008} & \textbf{0.8145} & \textbf{0.8594} & \textbf{0.7702} & \textbf{0.7452} & \textbf{0.8913} \\
\midrule
$\Delta_{\text{Base}}$ (pp) & $+1.065$ & $+1.056$ & $+1.155$ & $+1.095$ & $+1.239$ & $+0.929$ \\
95\% CI & $[0.697{,}1.433]$ & $[0.831{,}1.282]$ & $[0.730{,}1.579]$ & $[0.661{,}1.529]$ & $[0.514{,}1.965]$ & $[0.207{,}1.651]$ \\
\bottomrule
\end{tabular}
\par\medskip
\setlength{\tabcolsep}{3pt}
\begin{tabular}{@{}lrrrrrrr@{}}
\toprule
(b) Metric & Baseline & CAGrad & \textbf{GAR} & STGC & LoadPen & SwitchAux & STGC+Load \\
\midrule
Acc & 0.7593 & 0.7625 & \textbf{0.7702} & 0.7438 & 0.7556 & 0.7573 & 0.7412 \\
LVar & 0.03118 & 0.02994 & 0.01542 & \textbf{0.00101} & 0.01032 & 0.00846 & 0.00139 \\
Purity & 0.4368 & 0.4446 & \textbf{0.5645} & 0.4200 & 0.4427 & 0.4424 & 0.4071 \\
Util. & 0.5200 & 0.5250 & 0.6225 & \textbf{0.9725} & 0.7300 & 0.7525 & 0.9525 \\
\bottomrule
\end{tabular}
\endgroup

\end{table}

\paragraph{Trainable backbones.}
In classification-head RoBERTa, GAR has the highest five-mixture mean accuracy among seven methods and gains $+1.07$ points over Baseline and $+1.84$ over CAGrad (Table~\ref{tab:appendix_implb_paired}). GAR leads accuracy, structure purity, and NMI, while STGC reaches near-uniform load. On accuracy, LVar, gradient-mass purity, structure purity, and utilization, GAR and CAGrad are both non-dominated: CAGrad has higher mean gradient-mass purity, while GAR improves the other four means. Trainable DeBERTa full-parameter FFN gives $+1.155$ points over Baseline and $+0.873$ points over CAGrad, with lower LVar and higher utilization than both (Table~\ref{tab:appendix_implc_paired}). With top-1 routing in the classification-head setting, GAR gains $+1.06$ $[+0.83,+1.28]$ points over Baseline and $+0.79$ $[+0.30,+1.27]$ over CAGrad and leads all five mixtures; task-loss-only routing concentrates on a single expert in 17 of 25 runs, against 1 of 25 for GAR, which has lower LVar and higher gradient-mass purity and utilization than both comparators (Appendix~\ref{appsubsec:top1_head_results}).

GAR and Baseline select the same learning rate in the five-, six-, and both
seven-task classification-head mixtures; clipping and weight decay come from
Baseline's search in every setting. In trainable
DeBERTa-FFN, the five- and six-task mixtures also share the learning rate and
yield gains of $+0.65$ and $+1.45$ points. These matched-configuration comparisons complement the
five-mixture summaries (Appendix~\ref{app:hyperparameters_reproducibility}).

DeBERTa single-task controls yield GAR$-$Baseline gains of $+1.13$ points with LoRA and $+0.73$ with FFN over five tasks. Averaged over both adaptations, GAR gains $+0.93$ $[+0.21,+1.65]$ points over Baseline and $+1.13$ $[+0.42,+1.85]$ over CAGrad (Appendix~\ref{appsubsec:single_task_paired_uncertainty}).

\paragraph{Frozen LoRA-FFN backbones.}
In the frozen two-backbone aggregate, GAR combines the highest accuracy and gradient-mass purity among seven methods with lower LVar and higher utilization than Baseline and CAGrad (Table~\ref{tab:main_endpoint_controls}b). It also has the highest five-mixture mean accuracy on both backbones. Its paired accuracy gains are $+1.10$ $[+0.66,+1.53]$ percentage points over Baseline and $+0.78$ $[+0.27,+1.29]$ over CAGrad (Table~\ref{tab:paired_supervised_uncertainty}). Per-backbone GAR--Baseline intervals are positive on both backbones (Appendix~\ref{appsubsec:paired_supervised_uncertainty}).

Across the 10 frozen backbone--mixture combinations, GAR has higher mean accuracy than Baseline in all 10, lower LVar in 8, and higher utilization in 7. Appendix~\ref{appsubsec:task_grouped_multitask_results} gives the per-combination values.

STGC+Load is tuned for each setting and mixture over model-training learning rate, $\beta_{\mathrm{STGC}}$, and $\lambda_{\mathrm{load}}$ under the common search procedure; its results appear in Table~\ref{tab:main_endpoint_controls}(b), Figure~\ref{fig:appendix_implb_mechanism_main}, and the appendix tables.

\paragraph{Empirical Pareto comparison.}
Using accuracy, gradient-mass purity, and utilization as increasing axes and LVar as a decreasing axis, GAR Pareto-dominates Baseline and CAGrad in the panel-(b) frozen aggregate: all four reported means improve simultaneously. The non-dominated methods are GAR, STGC, LoadPen, and SwitchAux. GAR leads accuracy and gradient-mass purity, while the other three favor load uniformity and utilization.

\paragraph{Training-process evidence.}
Across ten checkpoints on the five classification-head mixtures, GAR has lower LVar and higher utilization, NMI, and structure purity than Baseline and CAGrad (Figures~\ref{fig:appendix_implb_mechanism_main}--\ref{fig:appendix_implb_mechanism_extra}). Its paired endpoint intra-expert coherence gain over Baseline is $+0.0315$ $[+0.0019,+0.0611]$ under the all-task-pair definition in Appendix~\ref{app:experimental_protocol} (Table~\ref{tab:appendix_implb_paired}).

\subsection{Observed gains across task mixtures}

\begin{figure}[!ht]
\centering
\begin{minipage}[t]{0.49\linewidth}
\vspace{0pt}
\begingroup
\scriptsize
\centering
\resizebox{\linewidth}{!}{%
\definecolor{barbase}{HTML}{666666}
\definecolor{barcagrad}{HTML}{0072B2}
\definecolor{barours}{HTML}{D55E00}
\definecolor{barstgc}{HTML}{009E73}
\definecolor{barload}{HTML}{CC79A7}
\definecolor{barswitch}{HTML}{E69F00}
\definecolor{barstgcload}{HTML}{56B4E9}
\begin{tikzpicture}[x=1cm,y=1cm,font=\scriptsize]
\def\panelW{5.20}\def\panelH{2.17}
\def\ymin{-7.50}\def\ymax{3.50}
\newcommand{\gy}[1]{\pgfmathsetmacro{\yy}{(#1-\ymin)/(\ymax-\ymin)*\panelH}}
\pgfmathsetmacro{\yzero}{(0-\ymin)/(\ymax-\ymin)*\panelH}
\newcommand{\statbar}[6]{%
  \gy{#4}%
  \filldraw[draw=white,line width=.25pt,fill=#3] (#1,\yzero) rectangle (#2,\yy);%
  \pgfmathsetmacro{\xc}{(#1+#2)/2}%
  \pgfmathsetmacro{\ylo}{(#5-\ymin)/(\ymax-\ymin)*\panelH}%
  \pgfmathsetmacro{\yhi}{(#6-\ymin)/(\ymax-\ymin)*\panelH}%
  \draw[black!70,line width=.35pt] (\xc,\ylo)--(\xc,\yhi);%
  \draw[black!70,line width=.35pt] (\xc-.028,\ylo)--(\xc+.028,\ylo);%
  \draw[black!70,line width=.35pt] (\xc-.028,\yhi)--(\xc+.028,\yhi);%
}
\newcommand{\cellpoints}[2]{%
  \foreach \dx/\val in {#2}{%
    \pgfmathsetmacro{\xp}{#1+\dx}%
    \gy{\val}%
    \filldraw[fill=white,fill opacity=.88,draw=black!75,line width=.18pt]
      (\xp,\yy) circle (.013);%
  }%
}
\foreach \v/\lab in {-7/$-7$,-5/$-5$,-3/$-3$,-1/$-1$,1/$1$,3/$3$}{
  \gy{\v}
  \draw[gray!20,line width=.35pt] (0,\yy) -- (\panelW,\yy);
  \draw[gray!55,line width=.35pt] (0,\yy) -- (-.05,\yy) node[left,font=\scriptsize] {\lab};
}
\draw[gray!65,line width=.55pt] (0,\yzero) -- (\panelW,\yzero);
\draw[gray!60,line width=.35pt] (0,0) rectangle (\panelW,\panelH);
\node[below,align=center,font=\scriptsize] at (1.595,-.06) {5--6\\$n=4$};
\node[below,align=center,font=\scriptsize] at (3.795,-.06) {7--8\\$n=6$};

\draw[barbase,line width=1.0pt] (1.0800,\yzero) -- (1.2100,\yzero);
\statbar{1.2300}{1.3600}{barcagrad!70}{0.2100}{-0.1069}{0.5269}
\statbar{1.3800}{1.5100}{barstgc!70}{-2.8950}{-7.0683}{1.2783}
\statbar{1.5300}{1.6600}{barload!70}{-0.7600}{-2.6001}{1.0801}
\statbar{1.6800}{1.8100}{barswitch!70}{-0.3650}{-1.1855}{0.4555}
\statbar{1.8300}{1.9600}{barours!88}{0.7275}{-0.1933}{1.6483}
\statbar{1.9800}{2.1100}{barstgcload!80}{-2.5950}{-4.3717}{-0.8183}
\cellpoints{1.2950}{-.024/.18,-.008/.35,.008/-.06,.024/.37}
\cellpoints{1.4450}{-.024/-4.70,-.008/-2.03,.008/-5.27,.024/.42}
\cellpoints{1.5950}{-.024/-.25,-.008/-.33,.008/-2.48,.024/.02}
\cellpoints{1.7450}{-.024/-.44,-.008/-.19,.008/-1.03,.024/.20}
\cellpoints{1.8950}{-.024/1.47,-.008/.06,.008/.74,.024/.64}
\cellpoints{2.0450}{-.024/-1.79,-.008/-1.55,.008/-3.90,.024/-3.14}

\draw[barbase,line width=1.0pt] (3.2800,\yzero) -- (3.4100,\yzero);
\statbar{3.4300}{3.5600}{barcagrad!70}{0.3883}{0.0393}{0.7374}
\statbar{3.5800}{3.7100}{barstgc!70}{-0.6417}{-1.2817}{-0.0016}
\statbar{3.7300}{3.8600}{barload!70}{-0.1017}{-0.6204}{0.4170}
\statbar{3.8800}{4.0100}{barswitch!70}{-0.0783}{-0.5144}{0.3578}
\statbar{4.0300}{4.1600}{barours!88}{1.3417}{1.0345}{1.6489}
\statbar{4.1800}{4.3100}{barstgcload!80}{-1.2817}{-2.9007}{0.3374}
\cellpoints{3.4950}{-.040/.06,-.024/.48,-.008/.22,.008/.55,.024/.09,.040/.93}
\cellpoints{3.6450}{-.040/-.66,-.024/-.38,-.008/.10,.008/-.75,.024/-.43,.040/-1.73}
\cellpoints{3.7950}{-.040/.25,-.024/-.02,-.008/.23,.008/.24,.024/-.30,.040/-1.01}
\cellpoints{3.9450}{-.040/-.03,-.024/.57,-.008/-.13,.008/-.08,.024/-.06,.040/-.74}
\cellpoints{4.0950}{-.040/1.72,-.024/1.33,-.008/1.30,.008/1.28,.024/1.56,.040/.86}
\cellpoints{4.2450}{-.040/-.26,-.024/1.32,-.008/-1.66,.008/-1.92,.024/-2.32,.040/-2.85}
\node[rotate=90,font=\scriptsize] at (-.62,\panelH/2) {$\Delta$ acc. vs. Baseline (pp)};
\begin{scope}[shift={(.05,-.82)},font=\tiny]
  \fill[barbase!70] (0,0) rectangle (.12,.12); \node[right=1.5pt] at (.12,.06) {Baseline};
  \fill[barcagrad!70] (1.26,0) rectangle (1.38,.12); \node[right=1.5pt] at (1.38,.06) {CAGrad};
  \fill[barstgc!70] (2.38,0) rectangle (2.50,.12); \node[right=1.5pt] at (2.50,.06) {STGC};
  \fill[barstgcload!80] (3.52,0) rectangle (3.64,.12); \node[right=1.5pt] at (3.64,.06) {STGC+Load};
  \fill[barload!70] (.64,-.24) rectangle (.76,-.12); \node[right=1.5pt] at (.76,-.18) {LoadPen};
  \fill[barswitch!70] (2.18,-.24) rectangle (2.30,-.12); \node[right=1.5pt] at (2.30,-.18) {SwitchAux};
  \fill[barours!88] (3.52,-.24) rectangle (3.64,-.12); \node[right=1.5pt,font=\tiny\bfseries] at (3.64,-.18) {GAR};
\end{scope}
\end{tikzpicture}%
}%
\caption{Accuracy gains over Baseline by task-count bin. Bars are equal-cell means; dots are backbone--mixture cells ($n=4,6$; frozen DeBERTa and Qwen3-1.7B), and whiskers are 95\% Student-$t$ intervals across those cells. All runs use E8K4; Table~\ref{tab:task_mixtures} defines the mixtures in each bin.}
\label{fig:task_scale_accuracy_gain}
\endgroup
\end{minipage}\hfill
\begin{minipage}[t]{0.485\linewidth}
\vspace{0pt}
\makeatletter\def\@captype{table}\makeatother
\begingroup
\centering
\caption{RoBERTa classification-head fixed-configuration ablations on the five E8K4 mixtures in Table~\ref{tab:task_mixtures} at 1,000 updates (accuracy, \%). (a) Coefficient sweep. (b) Denominator ablation at $\lambda=10^{-3}$; $\lambda=0$ disables the auxiliary term. Each mixture reuses its selected GAR configuration without further search. Values average tasks, mixtures, and seeds equally. Appendices~\ref{appsubsec:lambda_ablation_protocol} and~\ref{appsubsec:lambda_ablation_results} give full details.}
\label{tab:lambda_ablation_means}
\textbf{Fixed-configuration ablation}\par\vspace{4pt}
\begin{minipage}[t]{0.40\linewidth}
\centering
\setlength{\tabcolsep}{2.0pt}
\renewcommand{\arraystretch}{1.0}
\begin{tabular}[t]{@{}cr@{}}
\toprule
\multicolumn{2}{c}{\shortstack{\textbf{(a) Coefficient}\\\textbf{sweep}}} \\
\cmidrule(lr){1-2}
$\lambda$ & Acc. (\%) \\
\midrule
$0$ & 78.97 \\
$10^{-5}$ & 79.17 \\
$10^{-4}$ & 79.46 \\
$10^{-3}$ & \textbf{79.87} \\
$10^{-2}$ & 79.46 \\
\bottomrule
\end{tabular}
\end{minipage}\hfill
\begin{minipage}[t]{0.575\linewidth}
\centering
\setlength{\tabcolsep}{2.0pt}
\renewcommand{\arraystretch}{1.0}
\begin{tabular}[t]{@{}lr@{}}
\toprule
\multicolumn{2}{c}{\textbf{(b) Denominator}} \\
\cmidrule(lr){1-2}
Objective & Acc. (\%) \\
\midrule
$\lambda=0$ & 78.97 \\
Numerator only & 79.39 \\
Load-normalized & \textbf{79.87} \\
\bottomrule
\end{tabular}
\end{minipage}
\endgroup

\end{minipage}
\end{figure}

Across the five frozen LoRA-FFN mixtures, GAR's equal-backbone gains over Baseline are
$+0.73$ points for five--six tasks and $+1.34$ for seven--eight
(Figure~\ref{fig:task_scale_accuracy_gain}). The bins
describe different task counts and compositions. Classification-head RoBERTa
also has positive gains in both groups (Appendix~\ref{appsubsec:implb_results}).

\paragraph{Fixed-configuration coefficient ablations.}
Table~\ref{tab:lambda_ablation_means} reports a 1,000-update RoBERTa classification-head sweep on the same five E8K4 mixtures. It reuses each mixture's selected GAR configuration and varies only $\lambda$. At $10^{-3}$, the equal-mixture mean rises from $78.97\%$ to $79.87\%$, a paired gain of $+0.90$ $[+0.43,+1.37]$ points ($n=5$). Removing the load denominator at the same coefficient gives $79.39\%$; the load-normalized objective exceeds this numerator-only arm by $+0.47$ $[+0.08,+0.87]$ points. Appendix~\ref{appsubsec:lambda_ablation_results} gives per-task results and the corresponding frozen RoBERTa LoRA-FFN sweep.

\section{Conclusion}
\label{sec:conclusion}

We presented a load-normalized objective for routing in gradient space, together with a first-order implementation that directs auxiliary gradients to the router. The formulation connects expert assignment to gradient clustering and provides a criterion for grouping compatible optimization signals. Experiments show accuracy gains across trainable and frozen backbones, head and FFN experts, and top-1 and top-4 routing. GAR achieves the highest aggregate accuracy in the seven-method comparisons; the frozen LoRA-FFN aggregate also shows higher utilization and gradient-mass purity and lower load variance than Baseline and CAGrad. These findings support gradient information as a useful basis for expert routing beyond load balance.

\section{Limitations}
\label{sec:limitations}

Our experiments focus on supervised fine-tuning of pretrained language models for English text classification. Evaluation in large-scale pretraining and other modalities remains for future work.

\clearpage
\subsubsection*{Ethics Statement}

This work studies routing objectives for sparse Mixture-of-Experts models and
uses publicly released English text-classification datasets
(GLUE, SuperGLUE, PAWS, and ANLI) under their
original licenses. No new data were collected, no human subjects were involved,
and we did not intentionally collect or inspect personal information. Public
text datasets can nevertheless retain incidental personal or sensitive
content despite upstream filtering; users should follow the source datasets'
documentation and apply deployment-appropriate filtering. The method changes how
training signal is grouped across experts and does not introduce a new
generative capability; we are not aware of a direct dual-use concern beyond
those already inherent to fine-tuning pretrained language models.

\subsubsection*{Reproducibility Statement}

The paper and appendices specify the objective, implementation, experimental
settings, hyperparameter selection procedure, and evaluation metrics.
Reported endpoint results summarize five independent runs.
Selected configurations are listed in the hyperparameter tables of
Appendix~\ref{app:hyperparameters_reproducibility}; code implementing
the objective and the reported settings will be released at
\url{https://github.com/lyclyq/MoE_arxiv}.

\subsubsection*{Use of Large Language Models}

Large language models were used to retrieve relevant literature; draft and edit
manuscript text; refactor code and support research execution; and check and
revise mathematical derivations. The research questions, core method, and final
scientific decisions were developed and verified by the authors. Experimental
measurements were produced by the reported code. The authors reviewed all
AI-assisted content and take full responsibility for the contents of this paper.

\bibliographystyle{iclr2027_conference}
\bibliography{bib/references}

\begin{thebibliography}{23}
\providecommand{\natexlab}[1]{#1}
\providecommand{\url}[1]{\texttt{#1}}
\expandafter\ifx\csname urlstyle\endcsname\relax
  \providecommand{\doi}[1]{doi: #1}\else
  \providecommand{\doi}{doi: \begingroup \urlstyle{rm}\Url}\fi

\bibitem[Chen et~al.(2018)Chen, Badrinarayanan, Lee, and
  Rabinovich]{gradnorm_chen2018}
Zhao Chen, Vijay Badrinarayanan, Chen-Yu Lee, and Andrew Rabinovich.
\newblock {GradNorm}: Gradient normalization for adaptive loss balancing in
  deep multitask networks.
\newblock In \emph{International Conference on Machine Learning (ICML)}, 2018.

\bibitem[Dhillon et~al.(2004)Dhillon, Guan, and Kulis]{dhillon2004kernelkm}
Inderjit~S. Dhillon, Yuqiang Guan, and Brian Kulis.
\newblock Kernel k-means, spectral clustering and normalized cuts.
\newblock In \emph{Proceedings of the Tenth ACM SIGKDD International Conference
  on Knowledge Discovery and Data Mining}, pp.\  551--556, 2004.

\bibitem[Fedus et~al.(2022)Fedus, Zoph, and Shazeer]{switch_transformer}
William Fedus, Barret Zoph, and Noam Shazeer.
\newblock Switch transformers: Scaling to trillion parameter models with simple
  and efficient sparsity.
\newblock \emph{Journal of Machine Learning Research}, 23\penalty0
  (120):\penalty0 1--39, 2022.

\bibitem[He et~al.(2021)He, Liu, Gao, and Chen]{deberta_he2021}
Pengcheng He, Xiaodong Liu, Jianfeng Gao, and Weizhu Chen.
\newblock {DeBERTa}: Decoding-enhanced {BERT} with disentangled attention.
\newblock In \emph{International Conference on Learning Representations
  (ICLR)}, 2021.

\bibitem[He et~al.(2023)He, Gao, and Chen]{deberta_v3_he2023}
Pengcheng He, Jianfeng Gao, and Weizhu Chen.
\newblock {DeBERTaV3}: Improving {DeBERTa} using {ELECTRA}-style pre-training
  with gradient-disentangled embedding sharing.
\newblock In \emph{International Conference on Learning Representations
  (ICLR)}, 2023.

\bibitem[Hu et~al.(2022)Hu, Shen, Wallis, Allen-Zhu, Li, Wang, Wang, and
  Chen]{lora_hu2022}
Edward~J. Hu, Yelong Shen, Phillip Wallis, Zeyuan Allen-Zhu, Yuanzhi Li, Shean
  Wang, Lu~Wang, and Weizhu Chen.
\newblock {LoRA}: Low-rank adaptation of large language models.
\newblock In \emph{International Conference on Learning Representations
  (ICLR)}, 2022.

\bibitem[Hubert \& Arabie(1985)Hubert and Arabie]{hubert1985ari}
Lawrence Hubert and Phipps Arabie.
\newblock Comparing partitions.
\newblock \emph{Journal of Classification}, 2\penalty0 (1):\penalty0 193--218,
  1985.

\bibitem[Lewis et~al.(2021)Lewis, Bhosale, Dettmers, Goyal, and
  Zettlemoyer]{base_layers_lewis2021}
Mike Lewis, Shruti Bhosale, Tim Dettmers, Naman Goyal, and Luke Zettlemoyer.
\newblock {BASE} layers: Simplifying training of large, sparse models.
\newblock In \emph{International Conference on Machine Learning (ICML)}, 2021.

\bibitem[Liu et~al.(2021)Liu, Liu, Jin, Stone, and Liu]{cagrad_liu2021}
Bo~Liu, Xingchao Liu, Xiaojie Jin, Peter Stone, and Qiang Liu.
\newblock Conflict-averse gradient descent for multi-task learning.
\newblock In \emph{Advances in Neural Information Processing Systems
  (NeurIPS)}, pp.\  18878--18890, 2021.

\bibitem[Liu et~al.(2019)Liu, Ott, Goyal, Du, Joshi, Chen, Levy, Lewis,
  Zettlemoyer, and Stoyanov]{roberta_liu2019}
Yinhan Liu, Myle Ott, Naman Goyal, Jingfei Du, Mandar Joshi, Danqi Chen, Omer
  Levy, Mike Lewis, Luke Zettlemoyer, and Veselin Stoyanov.
\newblock {RoBERTa}: A robustly optimized {BERT} pretraining approach.
\newblock \emph{arXiv preprint arXiv:1907.11692}, 2019.

\bibitem[Nie et~al.(2020)Nie, Williams, Dinan, Bansal, Weston, and
  Kiela]{anli_nie2020}
Yixin Nie, Adina Williams, Emily Dinan, Mohit Bansal, Jason Weston, and Douwe
  Kiela.
\newblock Adversarial {NLI}: A new benchmark for natural language
  understanding.
\newblock In \emph{Proceedings of the 58th Annual Meeting of the Association
  for Computational Linguistics (ACL)}, 2020.

\bibitem[Rajbhandari et~al.(2022)Rajbhandari, Li, Yao, Zhang, Aminabadi, Awan,
  Rasley, and He]{deepspeed_moe_rajbhandari2022}
Samyam Rajbhandari, Conglong Li, Zhewei Yao, Minjia Zhang, Reza~Yazdani
  Aminabadi, Ammar~Ahmad Awan, Jeff Rasley, and Yuxiong He.
\newblock {DeepSpeed-MoE}: Advancing mixture-of-experts inference and training
  to power next-generation {AI} scale.
\newblock In \emph{International Conference on Machine Learning (ICML)}, 2022.

\bibitem[Roy et~al.(2021)Roy, Saffar, Vaswani, and
  Grangier]{routing_transformers_roy2021}
Aurko Roy, Mohammad Saffar, Ashish Vaswani, and David Grangier.
\newblock Efficient content-based sparse attention with routing transformers.
\newblock \emph{Transactions of the Association for Computational Linguistics},
  9:\penalty0 53--68, 2021.

\bibitem[Sener \& Koltun(2018)Sener and Koltun]{mgda_sener2018}
Ozan Sener and Vladlen Koltun.
\newblock Multi-task learning as multi-objective optimization.
\newblock In \emph{Advances in Neural Information Processing Systems
  (NeurIPS)}, 2018.

\bibitem[Strehl \& Ghosh(2002)Strehl and Ghosh]{strehl2002nmi}
Alexander Strehl and Joydeep Ghosh.
\newblock Cluster ensembles -- a knowledge reuse framework for combining
  multiple partitions.
\newblock \emph{Journal of Machine Learning Research}, 3:\penalty0 583--617,
  2002.

\bibitem[Tay et~al.(2020)Tay, Bahri, Yang, Metzler, and
  Juan]{sinkhorn_transformer_tay2020}
Yi~Tay, Dara Bahri, Liu Yang, Donald Metzler, and Da-Cheng Juan.
\newblock Sparse sinkhorn attention.
\newblock In \emph{International Conference on Machine Learning (ICML)}, 2020.

\bibitem[Wang et~al.(2019{\natexlab{a}})Wang, Pruksachatkun, Nangia, Singh,
  Michael, Hill, Levy, and Bowman]{superglue_wang2019}
Alex Wang, Yada Pruksachatkun, Nikita Nangia, Amanpreet Singh, Julian Michael,
  Felix Hill, Omer Levy, and Samuel~R. Bowman.
\newblock {SuperGLUE}: A stickier benchmark for general-purpose language
  understanding systems.
\newblock In \emph{Advances in Neural Information Processing Systems
  (NeurIPS)}, 2019{\natexlab{a}}.

\bibitem[Wang et~al.(2019{\natexlab{b}})Wang, Singh, Michael, Hill, Levy, and
  Bowman]{glue_wang2019}
Alex Wang, Amanpreet Singh, Julian Michael, Felix Hill, Omer Levy, and
  Samuel~R. Bowman.
\newblock {GLUE}: A multi-task benchmark and analysis platform for natural
  language understanding.
\newblock In \emph{International Conference on Learning Representations
  (ICLR)}, 2019{\natexlab{b}}.

\bibitem[Yang et~al.(2025{\natexlab{a}})]{qwen3_yang2025}
An~Yang et~al.
\newblock {Qwen3} technical report.
\newblock \emph{arXiv preprint arXiv:2505.09388}, 2025{\natexlab{a}}.

\bibitem[Yang et~al.(2025{\natexlab{b}})Yang, Shen, Cai, Yang, Gao, Zhang, and
  Li]{stgc_yang2024}
Longrong Yang, Dong Shen, Chaoxiang Cai, Fan Yang, Tingting Gao, Di~Zhang, and
  Xi~Li.
\newblock Solving token gradient conflict in mixture-of-experts for large
  vision-language model.
\newblock In \emph{International Conference on Learning Representations
  (ICLR)}, 2025{\natexlab{b}}.

\bibitem[Yu et~al.(2020)Yu, Kumar, Gupta, Levine, Hausman, and
  Finn]{pcgrad_yu2020}
Tianhe Yu, Saurabh Kumar, Abhishek Gupta, Sergey Levine, Karol Hausman, and
  Chelsea Finn.
\newblock Gradient surgery for multi-task learning.
\newblock In \emph{Advances in Neural Information Processing Systems
  (NeurIPS)}, 2020.

\bibitem[Zhang et~al.(2019)Zhang, Baldridge, and He]{paws_zhang2019}
Yuan Zhang, Jason Baldridge, and Luheng He.
\newblock {PAWS}: Paraphrase adversaries from word scrambling.
\newblock In \emph{Proceedings of the 2019 Conference of the North American
  Chapter of the Association for Computational Linguistics (NAACL)}, 2019.

\bibitem[Zhou et~al.(2022)Zhou, Lei, Liu, Du, Huang, Zhao, Dai, Chen, Le, and
  Laudon]{expert_choice_zhou2022}
Yanqi Zhou, Tao Lei, Hanxiao Liu, Nan Du, Yanping Huang, Vincent Zhao,
  Andrew~M. Dai, Zhifeng Chen, Quoc~V. Le, and James Laudon.
\newblock Mixture-of-experts with expert choice routing.
\newblock In \emph{Advances in Neural Information Processing Systems
  (NeurIPS)}, 2022.

\end{thebibliography}

\appendix

\makeatletter
\patchcmd{\LT@output}{\vss}{\vfil}{}{}
\patchcmd{\LT@output}{\vss}{\vfil}{}{}
\newcommand{\AppendixTableFontCap}{%
  \ifdim\f@size pt>9pt\relax\small\fi}
\AddToHook{env/tabular/begin}{\AppendixTableFontCap}
\AddToHook{env/tabular*/begin}{\AppendixTableFontCap}
\AddToHook{env/longtable/begin}{\AppendixTableFontCap}
\AddToHook{env/longtable/begin}{\renewcommand{\theHtable}{longtable.\arabic{table}}}
\newsavebox{\AppendixTableSavebox}
\newcommand{\AppendixTable}[1]{%
  \begingroup
  \sbox{\AppendixTableSavebox}{#1}%
  \Gscale@div\AppendixTableScale{9pt}{\f@size pt}%
  \@tempdima=\AppendixTableScale\wd\AppendixTableSavebox\relax
  \ifdim\@tempdima>\textwidth
    \resizebox{\textwidth}{!}{\usebox{\AppendixTableSavebox}}%
  \else
    \scalebox{\AppendixTableScale}{\usebox{\AppendixTableSavebox}}%
  \fi
  \endgroup}
\makeatother
\raggedbottom
\section{Additional Variational Derivations}
\label{app:variational_derivations}

Figure~\ref{fig:gradient_space_intuition} illustrates the gradient-grouping
preference underlying the following fixed-observation derivations.

\begin{figure}[!ht]
\centering
\includegraphics[width=0.72\linewidth,trim=10 8 10 8,clip]{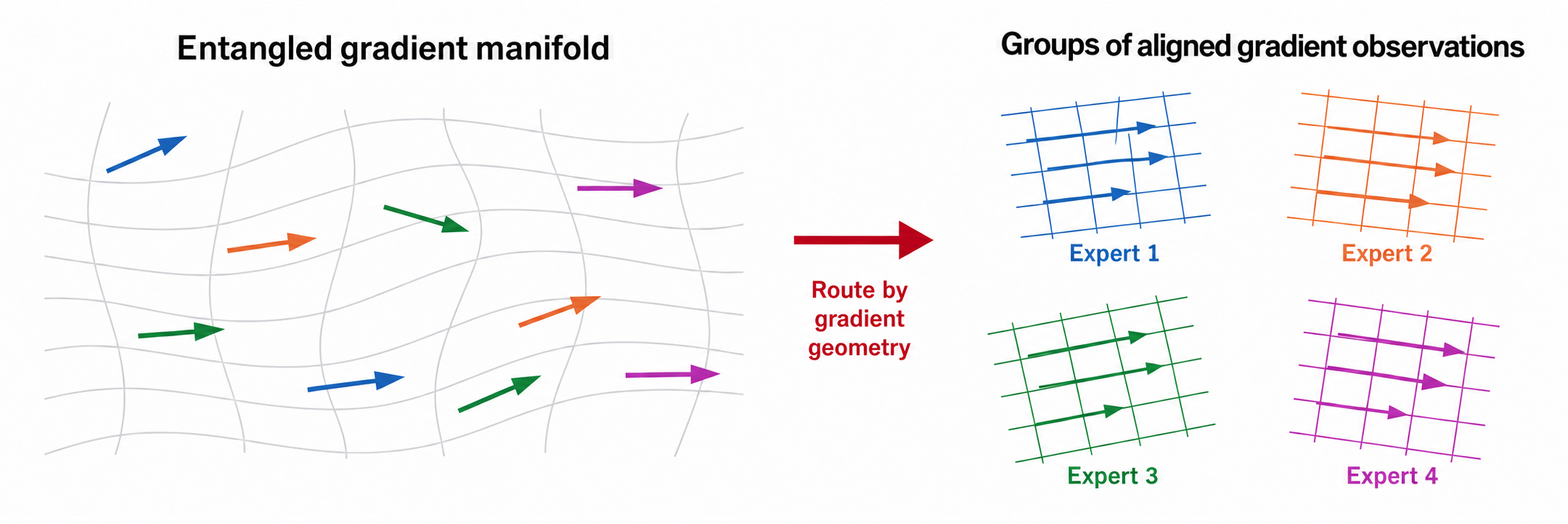}
\caption{Gradient-space routing intuition. The criterion groups gradient observations by signed associations. The expert clusters illustrate the grouping preference, rather than learned linear subspaces.}
\label{fig:gradient_space_intuition}
\end{figure}

\subsection{Stationary Gibbs form}
\label{app:gibbs_form_proof}

The following is the formal statement summarized in Section~\ref{sec:variational_formulation}.

\begin{proposition}[Stationary Gibbs form]
\label{prop:gibbs_form}
Fix $\tau>0$. If $P^\star$ is an interior stationary point of $\cF_{\tau,\epsilon}(\cdot;W)$ over $\cP_{M,K}$, then for every routed item $m$,
\begin{equation}
p^\star_{mk}
=
\frac{\exp\!\big(s_{mk}(P^\star)/\tau\big)}
{\sum_{\ell=1}^K \exp\!\big(s_{m\ell}(P^\star)/\tau\big)},
\qquad s_{mk}(P)
:=
\frac{2\ip{G_k}{\tilde g_m}}{d_k(P)+\epsilon}
-
\frac{\norm{G_k}^2}{(d_k(P)+\epsilon)^2}.
\label{eq:gibbs_routing}
\end{equation}
\end{proposition}

\begin{proof}
By Proposition~\ref{prop:euclidean_instantiation},
\[
\cF_{\tau,\epsilon}(P;W)
=
-
\sum_{k=1}^K \frac{\norm{G_k}^2}{d_k(P)+\epsilon}
+
\tau \sum_{m=1}^M \sum_{k=1}^K p_{mk}\log p_{mk}.
\]
Differentiating with respect to $p_{mk}$ gives
\begin{equation}
\frac{\partial \cF_{\tau,\epsilon}}{\partial p_{mk}}
=
-
\left(
\frac{2\ip{G_k}{\tilde g_m}}{d_k(P)+\epsilon}
-
\frac{\norm{G_k}^2}{(d_k(P)+\epsilon)^2}
\right)
+
\tau(1+\log p_{mk}).
\label{eq:free_energy_grad}
\end{equation}
Introducing a Lagrange multiplier $\alpha_m$ for the row constraint
$\sum_{k=1}^K p_{mk}=1$, the stationarity condition is
\[
\frac{\partial \cF_{\tau,\epsilon}}{\partial p_{mk}}+\alpha_m=0.
\]
Substituting \eqref{eq:free_energy_grad}, rearranging, and normalizing over $k$
yields \eqref{eq:gibbs_routing}.
\end{proof}

\section{Optimization geometry, Clustering structure, and Spectral view}
\label{app:additional_derivations}

\subsection{Proof of Proposition~\ref{prop:pairwise}} \label{app:pairwise}
For fixed gradient observations under hard assignments, each sample $m$ is assigned to exactly one expert,
forming clusters $\mathcal{C}_k = \{ m : p_{mk} = 1 \}$.

The surrogate aggregate of observations assigned to expert $k$ is:
\[
G_k = \sum_{m \in \mathcal{C}_k} \tilde g_m.
\]

Its squared norm expands as:
\[
\|G_k\|^2
=
\sum_{i,j \in \mathcal{C}_k}
\langle \tilde g_i, \tilde g_j \rangle.
\]

Therefore, the objective
\[
\sum_k \frac{\|G_k\|^2}{|\mathcal{C}_k|+\epsilon}
=
\sum_k \frac{1}{|\mathcal{C}_k|+\epsilon}
\sum_{i,j\in\mathcal{C}_k}\langle \tilde g_i,\tilde g_j\rangle
\]
is a load-normalized association within each observation cluster. Maximizing it
rewards positive co-routing affinity and discourages cancellation within the
surrogate aggregates.
\qed

\paragraph{Self-association and cross-observation agreement.}
\label{app:self_cross_association}
For general soft assignments, the full Gram reward decomposes as
\[
\cA_\epsilon(P;W)
=
\sum_k
\frac{
\sum_m p_{mk}^2\|\tilde g_m\|^2
+
\sum_{i\ne j}p_{ik}p_{jk}\langle\tilde g_i,\tilde g_j\rangle
}{d_k(P)+\epsilon}.
\]
The first term retains observation self-association; the second scores
agreement between distinct observations. The diagonal can favor concentrated
assignments even when the off-diagonal associations vanish. For example, with
$M=K=2$ and $W=I_2$, the uniform assignment has reward $1/(1+\epsilon)$,
whereas $P=I_2$ has reward $2/(1+\epsilon)$.
The full Gram choice underlies the Euclidean aggregate identity and the
k-means characterization below; removing its diagonal defines a different
objective. The experiments evaluate the full-Gram criterion.

\subsection{Hard-Routing Restriction and k-Means Equivalence}
\label{app:kmeans_equiv}

We now restrict the variational free energy to hard assignment matrices
\[
\cH_{M,K}
:=
\left\{
P\in \{0,1\}^{M\times K}
\;\middle|\;
P\1_K=\1_M
\right\}
\subset \cP_{M,K}.
\]
Under $P\in \cH_{M,K}$, each routed item is assigned to exactly one expert.
For the discrete equivalence below we set $\epsilon=0$ and restrict to nonempty
clusters, matching the usual k-means convention; the stabilizer only handles
numerical empty-cluster cases.

\begin{theorem}[Hard-routing restriction = k-means]
\label{thm:kmeans_equiv}
Under the Euclidean gradient-Gram affinity
\[
W_{ij}=\ip{\tilde g_i}{\tilde g_j},
\]
minimizing $\cF_{0,0}(P;W)$ over nonempty hard clusters is equivalent to minimizing
the classical k-means objective on $\{\tilde g_m\}_{m=1}^M$.
\end{theorem}

\begin{proof}
Let $\{\tilde g_m\}_{m=1}^M \subset \R^d$ be the detached gradient observations and assume hard routing
\[
p_{mk} \in \{0,1\},
\qquad
\mathcal{C}_k = \{ m \mid p_{mk}=1 \}.
\]
Let $|\mathcal{C}_k|$ denote cluster size.
Classical k-means minimizes
\[
\mathcal{J}_{k\text{-means}}
=
\sum_{k=1}^{K}
\sum_{m \in \mathcal{C}_k}
\norm{\tilde g_m - \mu_k}^2,
\qquad
\mu_k
=
\frac{1}{|\mathcal{C}_k|}
\sum_{m \in \mathcal{C}_k} \tilde g_m.
\]
Expanding,
\[
\norm{\tilde g_m - \mu_k}^2
=
\norm{\tilde g_m}^2
-2\ip{\tilde g_m}{\mu_k}
+\norm{\mu_k}^2.
\]
Summing over $m \in \mathcal{C}_k$ and substituting $\mu_k$ yields
\[
\sum_{m \in \mathcal{C}_k}
\norm{\tilde g_m - \mu_k}^2
=
\sum_{m \in \mathcal{C}_k}
\norm{\tilde g_m}^2
-
\frac{1}{|\mathcal{C}_k|}
\norm{
\sum_{m \in \mathcal{C}_k} \tilde g_m
}^2.
\]
The first term $\sum_{m=1}^{M}\norm{\tilde g_m}^2$ is independent of the clustering;
therefore minimizing $\mathcal{J}_{k\text{-means}}$ is equivalent to maximizing
\[
\sum_{k=1}^{K}
\frac{1}{|\mathcal{C}_k|}
\norm{
\sum_{m \in \mathcal{C}_k} \tilde g_m
}^2.
\]
Under hard routing,
\[
G_k=\sum_{m\in\mathcal{C}_k}\tilde g_m,
\qquad
d_k(P)=|\mathcal{C}_k|,
\]
hence
\[
\cF_{0,0}(P;W)
=
-
\sum_{k=1}^{K}
\frac{\norm{G_k}^2}{d_k(P)}
=
-
\sum_{k=1}^{K}
\frac{1}{|\mathcal{C}_k|}
\norm{
\sum_{m \in \mathcal{C}_k} \tilde g_m
}^2.
\]
Thus minimizing $\cF_{0,0}(P;W)$ is exactly equivalent to minimizing the k-means objective.
\end{proof}

\paragraph{Remark.}
Experts correspond to clusters in gradient space, and routing corresponds to assignments.
Soft routing probabilities $p_{mk}\in[0,1]$ therefore induce a relaxed k-means formulation.

\subsection{Spectral Relaxation}
\label{app:spectral_relaxation}

\begin{theorem}[Spectral relaxation of load-normalized alignment]
\label{thm:spectral_relaxation}
Under nonempty hard routing,
\[
\cF_{0,0}(P;W)
=
-
\sum_{k=1}^{K}
\frac{1}{|\mathcal{C}_k|}
\left\|
\sum_{m\in \mathcal{C}_k} \tilde g_m
\right\|^2.
\]
Define the gradient affinity matrix
$
W \in \R^{M\times M},
W_{ij} := \ip{\tilde g_i}{\tilde g_j}.
$
Then minimizing $\cF_{0,0}(P;W)$ is equivalent to maximizing the ratio association objective
\[
\max_{\{\mathcal{C}_k\}_{k=1}^K}
\sum_{k=1}^K
\frac{1}{|\mathcal{C}_k|}
\sum_{i,j\in \mathcal{C}_k} W_{ij}.
\]
Moreover, its standard continuous relaxation yields a top-$K$ eigenvector solution.
\end{theorem}

\begin{proof}
First expand each cluster term:
\[
\left\|
\sum_{m\in \mathcal{C}_k} \tilde g_m
\right\|^2
=
\sum_{i\in \mathcal{C}_k}
\sum_{j\in \mathcal{C}_k}
\ip{\tilde g_i}{\tilde g_j}
=
\sum_{i,j\in \mathcal{C}_k} W_{ij}.
\]
Hence
\[
-\cF_{0,0}(P;W)
=
\sum_{k=1}^K
\frac{1}{|\mathcal{C}_k|}
\sum_{i,j\in \mathcal{C}_k} W_{ij},
\]
which is exactly the ratio association objective.

To express it in trace form, let
$H\in\{0,1\}^{M\times K}$ be the hard assignment matrix,
where $H_{ik}=1$ iff $i\in\mathcal{C}_k$.
Define
\[
S := H^\top H
=
\diag(|\mathcal{C}_1|,\dots,|\mathcal{C}_K|).
\]
Let $h_k$ denote the $k$-th column of $H$.
Then
\[
\sum_{i,j\in\mathcal{C}_k} W_{ij}
=
h_k^\top W h_k,
\qquad
|\mathcal{C}_k|
=
h_k^\top h_k.
\]
Therefore
\[
-\cF_{0,0}(P;W)
=
\sum_{k=1}^K
\frac{h_k^\top W h_k}{h_k^\top h_k}
=
\Tr\!\left(S^{-1}H^\top W H\right).
\]
Define the normalized embedding
\[
Y := H S^{-1/2}
\in \R^{M\times K},
\qquad
Y^\top Y = I_K.
\]
Then
\[
-\cF_{0,0}(P;W)
=
\Tr\!\left(Y^\top W Y\right),
\]
where $Y$ is constrained to arise from a discrete partition.

Relaxing the discreteness constraint and optimizing over all
$Y \in \R^{M\times K}$ satisfying $Y^\top Y = I_K$ gives
\[
\max_{Y^\top Y=I_K}
\Tr(Y^\top W Y).
\]
By Ky Fan's maximum principle, a maximizer is given by choosing the columns
of $Y$ as the top-$K$ eigenvectors of $W$.
\end{proof}

\paragraph{Remark.}
The load-normalized alignment term admits a spectral relaxation:
it corresponds to ratio association graph partitioning on the gradient affinity graph $W$.
Hard routing recovers a discrete partition, while the standard orthogonal relaxation
corresponds to a spectral embedding spanned by the top-$K$ eigenvectors of $W$.
The practical soft router is a separate differentiable relaxation of the hard partition.

\section{Routing Dynamics and Alignment Geometry}
\label{app:dynamics}

\paragraph{Guide to Appendix~\ref{app:dynamics}.}
This appendix analyzes load scaling, directional preference, and mode separation under a symmetric prototype model. An illustrative task-loss feedback model supplies a reference for the load comparison. Table~\ref{tab:appendix_c_map} maps each property to its assumptions and formal result.

\begin{table}[H]
\centering
\small
\caption{Map of load-scaling and routing-geometry analyses in Appendix~\ref{app:dynamics}.}
\label{tab:appendix_c_map}
\begin{tabular}{p{0.30\linewidth}p{0.27\linewidth}p{0.33\linewidth}}
\toprule
Analyzed property & Appendix subsection & Formal support \\
\midrule
Idealized setting & Appendix~\ref{sec:idealized_setting} & Assumptions on manifold priors, gradient concentration, finite-window sampling, and the baseline dynamics. \\
Idealized baseline feedback & Appendix~\ref{sec:baseline_collapse_derivation} & Proposition~\ref{prop:baseline_collapse}. \\
Load-normalized anti-amplification & Appendix~\ref{sec:ours_case1} & Proposition~\ref{prop:case1_antiamplification}. \\
Static directional tilt and coherence preference & Appendix~\ref{sec:ours_case2} and Appendix~\ref{app:directional_tilt_lemmas} & Proposition~\ref{prop:formal_self_reinforcing_specialization} plus Lemmas~\ref{lem:tilt}, \ref{lem:tilt_alignment}, and~\ref{lem:purity_increases_U}. \\
Collapsed assignments are not local maxima & Appendix~\ref{sec:ours_case3} & Theorem~\ref{thm:vertex_unstable} for the leading-order alignment surrogate. \\
\bottomrule
\end{tabular}
\end{table}

\subsection{Idealized Theoretical Assumptions}
\label{sec:idealized_setting}

We use a symmetric prototype model to isolate load feedback and the geometric
preference for separating mixed gradient observations. The baseline result
concerns its specified continuous-time feedback law; the alignment results
concern fixed observation distributions and the stated assignment objectives.

\paragraph{Expert--manifold matching ($K=C$).}
We assume the number of experts equals the number of latent gradient manifolds:
\[
K = C.
\]
Each manifold $c \in \{1,\dots,C\}$ is associated with a prototype
unit direction $u_c \in \mathbb{R}^d$, $\|u_c\|=1$.
The ideal optimal configuration assigns one expert to each manifold.

\paragraph{Uniform manifold prior.}
Training observations are drawn from the $C$ manifolds
with equal probability:
\[
\Pr(c) = \frac{1}{C}.
\]
Thus, no manifold has intrinsic frequency advantage.
Any imbalance arises purely from routing dynamics.

\paragraph{Gradient concentration model.}
Detached gradient observations follow $\tilde g\mid c=u_c+\xi$, where
\[
\mathbb{E}[\xi\mid c]=0,\qquad \mathrm{Cov}(\xi\mid c)\preceq\sigma^2 I,
\qquad \mathbb{E}[\tilde g\mid c]=u_c .
\]
We assume the diagonal second moment
$\nu:=\mathbb{E}[\|\tilde g\|^2\mid c]$ is mode-independent, as in the symmetric noise model.

\paragraph{Pairwise-homogeneous prototype geometry.}
Let $\Gamma_{ab}:=\langle u_a,u_b\rangle$ be the prototype Gram matrix.  We use the
equiangular specialization
\[
\Gamma_{ab}=\begin{cases}
1,&a=b,\\
\rho,&a\neq b,
\end{cases}
\qquad -\frac{1}{C-1}\leq \rho<1,
\]
and assume that this Gram matrix is realizable in $\mathbb{R}^d$
($\operatorname{rank}(\Gamma)\leq d$).  The lower bound is the positive-semidefinite
constraint for an equicorrelation Gram matrix, and $\rho<1$ supplies the
strict same-mode advantage. Prototype geometry and mode priors are separate
assumptions of the model.

The coherence and directional-overlap quantities are then consequences of
the same Gram matrix, rather than separate free parameters.  Conditional
independence and $\mathbb{E}[\tilde g\mid c]=u_c$ give, for $i\neq j$,
\[
\mathbb{E}[\langle\tilde g_i,\tilde g_j\rangle\mid c_i=a,c_j=b]
=\mathbb{E}[\langle u_a,\tilde g\rangle\mid c=b]
=\Gamma_{ab}.
\]
Thus both quantities equal $1$ for $a=b$ and $\rho$ for $a\neq b$.
All scalar-coherence and purity reductions below inherit this specialization.
For a general prototype Gram matrix, the pairwise term is instead
$w_k^\top\Gamma w_k$ and need not depend only on $\sum_cw_{kc}^2$.
The case $\rho<0$, where gradients from different manifolds interfere
destructively in expectation, is the strongest instance of the gap
$1-\rho>0$ and is not assumed separately.

\paragraph{Load and mixture within an expert.}
For the finite-window calculations, fix an integer count $L_k\geq1$ and draw
$L_k$ observations independently from a mixture with mode probabilities
$w_{kc}$, where $\sum_cw_{kc}=1$. Here $G_k$ is the unweighted sum of
these observations; $w_{kc}$ specifies their sampling distribution, rather
than the realized empirical proportions in a particular window. These exact
finite-window formulas use this fixed-count, independent sampling model:
\begin{equation}
\mathbb{E}\|G_k\|^2
=
L_k\nu+
L_k(L_k-1)
\left(
\sum_{c}w_{kc}^2
+
\sum_{c\neq c'}w_{kc}w_{kc'}\rho
\right).
\label{eq:coherence_expansion}
\end{equation}
Thus, to leading pairwise order in $L_k$,
\begin{equation}
\mathbb{E}\|G_k\|^2
=
L_k^2
\left(
\rho
+
(1-\rho)
\sum_cw_{kc}^2
\right)
+O(L_k).
\label{eq:coherence_simplified}
\end{equation}

\paragraph{Geometric implication of mixing.}
Since $\sum_c w_{kc}^2$ is maximized (equal to $1$)
when expert $k$ receives a single manifold,
and minimized (equal to $1/C$) under uniform mixing,
Equation~\eqref{eq:coherence_expansion} shows, at fixed $L_k>1$:
\begin{itemize}
\item A single-manifold sampling distribution maximizes $\mathbb{E}\|G_k\|^2$.
\item Mixing reduces expected coherence in proportion to the pairwise gap $1-\rho$.
\item A smaller cross-manifold overlap $\rho$ increases this mixing penalty.
\end{itemize}

\paragraph{Initial symmetry.}
We assume symmetric initialization:
\[
p_k(0) = \frac{1}{K},
\qquad
\kappa_k(0) = \kappa_0.
\]
Thus, no expert has intrinsic advantage.
Any asymmetry emerges dynamically.

\paragraph{Continuous-time routing abstraction.}
We analyze routing via a continuous-time approximation
of soft probability updates,
treating
\[
p(t) \in \Delta^{K-1}.
\]
This allows direct stability and limit analysis.

\subsection{Baseline Reference: An Illustrative Load-Feedback Model}
\label{sec:baseline_collapse_derivation}

Under the idealized assumptions in \S\ref{sec:idealized_setting}, we first analyze the baseline case
where routing is driven solely by task loss (no alignment regularizer).
We show that even under a symmetric initialization, an arbitrarily small early advantage can be amplified
into a winner-take-all collapse (a simplex vertex attractor) via a positive feedback loop between load allocation
and learning speed under the following competence--loss and load-driven
learning dynamics.

\paragraph{(1) Uniform start with a small perturbation.}
We consider a soft routing state $p(t)\in\Delta^{K-1}$ with symmetric initialization
$p_k(0)=1/K$.
A ``small perturbation'' means that for some pair of experts $i\neq j$,
\[
p_i(0) = \frac{1}{K}+\eta,\qquad
p_j(0) = \frac{1}{K}-\eta,
\qquad
0<\eta\ll 1,
\]
and optionally a matching competence perturbation $\kappa_i(0)>\kappa_j(0)$.
Such asymmetry can arise from stochastic optimization noise or finite-sample effects.

\paragraph{Competence--loss coupling and load-driven learning.}
We assume
\begin{equation}
\ell_k(t)=\ell_0-a\,\kappa_k(t),\qquad a>0,
\label{eq:baseline_loss_skill}
\end{equation}
and load-driven competence dynamics
\begin{equation}
\dot\kappa_k(t)=\alpha\,p_k(t),\qquad \alpha>0.
\label{eq:baseline_skill}
\end{equation}

\paragraph{(2) Ratio dynamics $\Rightarrow$ positive feedback.}
We model the baseline router update by replicator dynamics with payoff $\Pi_k(t)=-\ell_k(t)$:
\begin{equation}
\dot p_k(t)=\beta\,p_k(t)\big(\bar\ell(t)-\ell_k(t)\big),\qquad \beta>0,
\label{eq:baseline_replicator}
\end{equation}
where $\bar\ell(t)=\sum_{j=1}^K p_j(t)\ell_j(t)$.
For any two experts $i,j$, define $r_{ij}(t)=p_i(t)/p_j(t)$.
A standard identity of replicator dynamics gives
\begin{equation}
\frac{d}{dt}\log r_{ij}(t)
=
\beta\big(\ell_j(t)-\ell_i(t)\big).
\label{eq:baseline_ratio_general}
\end{equation}
Substituting \eqref{eq:baseline_loss_skill} yields
\begin{equation}
\frac{d}{dt}\log r_{ij}(t)
=
\beta a\big(\kappa_i(t)-\kappa_j(t)\big).
\label{eq:baseline_ratio_skill}
\end{equation}
Meanwhile, from \eqref{eq:baseline_skill},
\begin{equation}
\frac{d}{dt}\big(\kappa_i(t)-\kappa_j(t)\big)
=
\alpha\big(p_i(t)-p_j(t)\big).
\label{eq:baseline_skill_gap}
\end{equation}
Equations \eqref{eq:baseline_ratio_skill}--\eqref{eq:baseline_skill_gap}
form a closed positive feedback loop:
a small competence advantage increases routing probability, which further increases competence.

\paragraph{(3) Strict collapse in the two-expert case ($K=2$).}
The analytically tractable case is $K=2$.
Let $p_2(t)=1-p_1(t)$ and define the competence gap $\Delta\kappa(t)=\kappa_1(t)-\kappa_2(t)$.
From \eqref{eq:baseline_skill_gap},
\begin{equation}
\dot{\Delta\kappa}(t)=\alpha(2p_1(t)-1),
\label{eq:baseline_two_skill_gap}
\end{equation}
and from \eqref{eq:baseline_ratio_skill} with $r(t)=p_1(t)/(1-p_1(t))$,
\begin{equation}
\frac{d}{dt}\log\frac{p_1(t)}{1-p_1(t)}
=
\beta a\,\Delta\kappa(t).
\label{eq:baseline_two_ratio}
\end{equation}

\begin{proposition}[Winner-take-all under a small initial advantage]
\label{prop:baseline_collapse}
Assume $K=2$, an initial routing advantage, and no initial competence deficit:
\[
p_1(0)>\tfrac12
\quad\text{and}\quad
\Delta\kappa(0)\geq0.
\]
Then $\Delta\kappa(t)$ is strictly increasing and
\[
\frac{p_1(t)}{1-p_1(t)}\to\infty
\quad\Rightarrow\quad
p_1(t)\to 1.
\]
\end{proposition}

\begin{proof}
Let $T=\sup\{t\ge0: p_1(u)>\tfrac12\text{ for all }u\in[0,t)\}$, which is
positive by continuity. On $[0,T)$, \eqref{eq:baseline_two_skill_gap} gives
$\dot{\Delta\kappa}>0$, so $\Delta\kappa(t)\ge\Delta\kappa(0)\ge0$, and
\eqref{eq:baseline_two_ratio} then gives $\frac{d}{dt}\log\frac{p_1}{1-p_1}\ge0$,
so $p_1(t)\ge p_1(0)>\tfrac12$. If $T$ were finite, continuity would give
$p_1(T)=\tfrac12$, a contradiction; hence $p_1(t)\ge p_1(0)$ for all $t\ge0$.
Consequently $\dot{\Delta\kappa}(t)\ge c:=\alpha(2p_1(0)-1)>0$, so $\Delta\kappa$ is
strictly increasing with $\Delta\kappa(t)\ge\Delta\kappa(0)+ct$, and
\[
\log\frac{p_1(t)}{1-p_1(t)}
\ge
\log\frac{p_1(0)}{1-p_1(0)}+\beta a\Big(\Delta\kappa(0)\,t+\tfrac{c}{2}t^2\Big)
\longrightarrow\infty,
\]
so $p_1(t)\to 1$.
\end{proof}

\paragraph{Extension to $K>2$ (pairwise domination).}
For general $K$, assume $p_k(0)>0$ for every expert. The ratio identity
\eqref{eq:baseline_ratio_general} then holds for every pair $(i,j)$.
A competence advantage $\kappa_{k^\star}(t)>\kappa_j(t)$ over a finite interval increases
$\log(p_{k^\star}/p_j)$ during that interval. The asymptotic behavior follows
from the integrated competence advantage: integrating \eqref{eq:baseline_ratio_skill} gives
\begin{equation}
\log\frac{p_{k^\star}(t)}{p_j(t)}
=\log\frac{p_{k^\star}(0)}{p_j(0)}
+\beta a\int_0^t\big(\kappa_{k^\star}(u)-\kappa_j(u)\big)\,du.
\label{eq:baseline_ratio_integrated}
\end{equation}
Consequently, a sufficient condition for $p_{k^\star}(t)\to1$ is
\[
\int_0^t\big(\kappa_{k^\star}(u)-\kappa_j(u)\big)\,du\longrightarrow+\infty
\qquad\text{for every }j\neq k^\star.
\]
Under this condition, all ratios $p_j(t)/p_{k^\star}(t)$ tend to zero, and
$p_{k^\star}(t)=[1+\sum_{j\neq k^\star}p_j(t)/p_{k^\star}(t)]^{-1}\to1$.
This integrated-advantage condition extends the two-expert result to arbitrary $K$.

\subsection{GAR -- Case 1: Load-Normalized Anti-Amplification under Homogeneous Mixing}
\label{sec:ours_case1}

We first analyze the load dependence of the alignment objective under the independent finite-window model.
Assume that each expert receives a statistically homogeneous mixture
of manifolds, i.e.,
\[
w_{kc} = \frac{1}{C}
\quad
\text{for all } k,c,
\]
so that all experts have identical internal conflict statistics,
but their total loads $L_k$ may differ.

\paragraph{Expected coherence under homogeneous mixing.}
Substituting $w_{kc}=1/C$ into \eqref{eq:coherence_expansion} gives the exact identity
\begin{equation}
\mathbb{E}\|G_k\|^2
=
\gamma L_k^2+(\nu-\gamma)L_k,
\qquad
\gamma:=\rho+
(1-\rho)/C .
\label{eq:case1_coherence}
\end{equation}
For independent draws, $\gamma=\|\mathbb{E}\tilde g\|^2\geq0$ and
$\nu-\gamma=\mathbb{E}\|\tilde g-\mathbb{E}\tilde g\|^2\geq0$.

\paragraph{Alignment utility scaling.}
Define the per-expert utility $U_k(P):=\|G_k(P)\|^2/(d_k(P)+\epsilon)$, the
$k$-th term of $\cA_\epsilon(P;W)$ under the gradient-Gram affinity. At a hard
assignment $d_k(P)=L_k$, so
\[
U_k
=
\frac{\|G_k\|^2}{L_k+\epsilon}.
\]
Taking expectation while keeping the stabilizer explicit gives
\begin{equation}
\mathbb{E}[U_k]
=
\frac{\gamma L_k^2+(\nu-\gamma)L_k}{L_k+\epsilon}
=
\gamma L_k+\frac{(\nu-\gamma-\gamma\epsilon)L_k}{L_k+\epsilon}.
\label{eq:case1_U}
\end{equation}

Thus, for fixed $\epsilon\geq0$, load normalization converts the quadratic
leading term into $\mathbb{E}[U_k]=\gamma L_k+O(1)$. The exact per-unit utility is
\[
\frac{\mathbb{E}[U_k]}{L_k}
=\frac{\gamma L_k+\nu-\gamma}{L_k+\epsilon}
\longrightarrow\gamma.
\]
It exceeds $\gamma$ by $(\nu-\gamma-\gamma\epsilon)/(L_k+\epsilon)$, a gap
that vanishes as the load grows.

\paragraph{Marginal alignment score.}
Routing decisions depend on the row-wise derivative of $U_k(P)$ with respect to
$p_{mk}$, evaluated where $d_k(P)=L_k$:
\begin{equation}
s_{mk}(P)
=
\frac{\partial U_k}{\partial p_{mk}}
=
\frac{2\langle G_k,\tilde g_m\rangle}{L_k+\epsilon}
-
\frac{\|G_k\|^2}{(L_k+\epsilon)^2}.
\label{eq:marginal_utility}
\end{equation}
Evaluate this derivative at a hard assignment containing observation $m$,
so that $m$ is one of the $L_k$ independent samples already included in $G_k$.
The self term must then be retained:
\[
\mathbb{E}\langle G_k,\tilde g_m\rangle
=\nu+(L_k-1)\gamma.
\]
Combining this with \eqref{eq:case1_coherence} gives the exact expectation
\begin{equation}
\mathbb{E}[s_{mk}(P)]
=
\gamma\left(1-\frac{\epsilon^2}{(L_k+\epsilon)^2}\right)
+(\nu-\gamma)\frac{L_k+2\epsilon}{(L_k+\epsilon)^2}.
\label{eq:marginal_score_saturation}
\end{equation}
In particular, at $\epsilon=0$, this equals
$\gamma+(\nu-\gamma)/L_k$, which exceeds $\gamma$ whenever the observations
have nonzero variance. For fixed $\epsilon\geq0$, the expected score still
converges to the finite limit $\gamma$ as $L_k\to\infty$. The expectation is
taken over an observation in the same finite sampling window.

\begin{proposition}[Load-normalized anti-amplification under homogeneous conflict]
\label{prop:case1_antiamplification}
Under the independent finite-window model with homogeneous mixing
($w_{kc}=1/C$), integer $L_k\geq1$, finite $\nu$, and fixed $\epsilon\geq0$,
the load-normalized alignment channel removes the quadratic load amplification in
$\mathbb{E}\|G_k\|^2$: $\mathbb{E}[U_k]=\gamma L_k+O(1)$, and the expected
in-window marginal score in \eqref{eq:marginal_score_saturation} converges to
$\gamma$. At $\epsilon=0$ it equals $\gamma+(\nu-\gamma)/L_k$, which is
non-increasing in $L_k$, so additional load does not raise the expected
marginal score.
\end{proposition}

\begin{proof}
The utility identity follows by dividing \eqref{eq:case1_coherence} by
$L_k+\epsilon$. Substitution of the in-window self term into
\eqref{eq:marginal_utility} yields \eqref{eq:marginal_score_saturation}.
Taking $L_k\to\infty$ with the mixture and $\epsilon$ fixed gives the limits. At
$\epsilon=0$, \eqref{eq:marginal_score_saturation} reduces to
$\gamma+(\nu-\gamma)/L_k$, which is non-increasing in $L_k$ because $\nu\geq\gamma$.
\end{proof}

\paragraph{Remark.}
In contrast to the baseline dynamics,
where higher load can increase learning speed and further increase routing probability,
the expected marginal score of the normalized alignment channel does not grow
with load in this sampling model. The static coherence preference and
the leading-order separation argument are treated separately in
Appendix~\ref{sec:ours_case2}--\ref{sec:ours_case3}.

\subsection{GAR -- Case 2: Static Directional Tilt and Coherence Preference}
\label{sec:ours_case2}

We compare fixed sampling mixtures at the same load. Enriching a mixture in
one manifold tilts the mean surrogate aggregate toward that manifold and
increases the expected coherence utility relative to uniform mixing.

\paragraph{Setup (mode prototype model).}
Under \S\ref{sec:idealized_setting}, $\mathbb{E}[\tilde g\mid c]=u_c$ and
$\|u_c\|=1$. Directional alignment and pairwise coherence are both governed
by the single prototype-overlap gap $1-\rho>0$.
Let expert $k$ receive $L_k$ independent observations from mixture
probabilities $w_{kc}$, with integer $L_k\geq1$ fixed. Each mode pair obeys
the common coherence and prototype-overlap conditions in
Section~\ref{sec:idealized_setting}.

\begin{lemma}[Directional tilt of the surrogate aggregate]
\label{lem:tilt}
Under the fixed-mixture independent sampling model, the detached surrogate aggregate
\[
G_k = \sum_{m=1}^{L_k}\tilde g_m
\]
satisfies
\begin{equation}
\mathbb{E}\!\left[G_k \mid w_{k\cdot}\right]
=
L_k \sum_{c=1}^C w_{kc}\,u_c .
\label{eq:tilt_mean}
\end{equation}
\end{lemma}

\begin{proof}
Each observation has mean $\sum_cw_{kc}u_c$. By linearity,
\[
\mathbb{E}[G_k\mid w_{k\cdot}]
=
\sum_{m=1}^{L_k}\sum_{c=1}^Cw_{kc}\,u_c
=
L_k\sum_c w_{kc}\,u_c.
\]
\end{proof}

\paragraph{Remark.}
Equation~\eqref{eq:tilt_mean} shows that any small bias in $w_{kc}$ immediately tilts the mean direction
of $G_k$ toward the corresponding prototype $u_c$.

Two technical lemmas used below are stated in Section~\ref{app:directional_tilt_lemmas}:
Lemma~\ref{lem:tilt_alignment} formalizes the same-manifold alignment advantage,
and Lemma~\ref{lem:purity_increases_U} shows that the per-load coherence utility
strictly increases with mixture purity.

\paragraph{Remark.}
For the symmetric perturbation below, independent samples from the enriched
manifold become more aligned with the mean $G_k$, while those from each
depleted manifold become less aligned. This comparison concerns the
alignment component of the surrogate, with the mixtures held fixed.

\paragraph{Static comparison with uniform mixing.}
Consider a manifold $c^\star$ and an expert $k$ with a tiny initial enrichment
\[
w_{k c^\star} = \frac{1}{C} + \delta,\qquad \delta>0\ \text{arbitrarily small},
\]
with the remaining mass spread across other manifolds.
Then:
\begin{itemize}
\item By Lemma~\ref{lem:tilt}, $G_k$ tilts toward $u_{c^\star}$.
\item By Lemma~\ref{lem:tilt_alignment}, mode-$c^\star$ samples have larger expected alignment
$\mathbb{E}\langle G_k,\tilde g\rangle$ than under uniform mixing.
\item By Lemma~\ref{lem:purity_increases_U}, the expert's coherence utility $U_k$ increases as its mixture becomes purer.
\end{itemize}
These comparisons quantify the objective preference for coherent observation
clusters at fixed load.

\begin{proposition}[Static directional tilt and coherence preference]
\label{prop:formal_self_reinforcing_specialization}
Under the independent sampling model, assume $C>1$, integer $L_k>1$,
$\epsilon\geq0$, and $\rho<1$.
Suppose expert $k$ is perturbed from a uniform mixture toward mode $c^\star$ by
\[
w_{k c^\star}=\frac{1}{C}+\delta,
\qquad
w_{kc}=\frac{1}{C}-\frac{\delta}{C-1}\quad(c\neq c^\star),
\]
where $0<\delta<(C-1)/C$.
Then the expected alignment of an independent mode-$c^\star$ sample with $G_k$ increases by
\[
L_k\delta(1-\rho)>0
\]
relative to uniform mixing. Moreover,
\[
\sum_c w_{kc}^2
=
\frac{1}{C}+\frac{C}{C-1}\delta^2
>
\frac{1}{C},
\]
so the expected coherence utility is larger than under uniform mixing by
\[
\frac{L_k(L_k-1)}{L_k+\epsilon}
(1-\rho)
\frac{C}{C-1}\delta^2>0.
\]
The comparison holds for the stated fixed mixtures.
\end{proposition}

\begin{proof}
For a mode-$c^\star$ sample, Lemma~\ref{lem:tilt_alignment} gives
\[
\mathbb{E}\!\big[\langle \mathbb{E}[G_k\mid w_{k\cdot}],\tilde g\rangle\mid c^\star\big]
=
L_k\left(w_{k c^\star}+
(1-w_{k c^\star})\rho\right).
\]
Relative to the uniform value obtained by setting $w_{k c^\star}=1/C$, the difference is
\[
L_k\delta(1-\rho)>0.
\]
For the purity term,
\[
\sum_c w_{kc}^2
=
\left(\frac{1}{C}+\delta\right)^2
+(C-1)\left(\frac{1}{C}-\frac{\delta}{C-1}\right)^2
=
\frac{1}{C}+\frac{C}{C-1}\delta^2.
\]
This is strictly larger than the uniform purity $1/C$ for any $\delta>0$.
Multiplying this purity increase by the slope in
Lemma~\ref{lem:purity_increases_U} gives the stated utility difference.
\end{proof}

\begin{remark}[Relation between the alignment cases]
Case~2 characterizes directional alignment and coherence at fixed load.
Case~3 varies the mode-to-expert assignment to analyze the corresponding
leading-order separation preference.
\end{remark}

\subsection{Technical Lemmas for Static Directional Tilt}
\label{app:directional_tilt_lemmas}

\begin{lemma}[Tilt implies same-manifold alignment advantage]
\label{lem:tilt_alignment}
For an independent sample from manifold $c$, the expected alignment with the mean surrogate aggregate satisfies
\begin{equation}
\mathbb{E}\!\left[\langle \mathbb{E}[G_k\mid w_{k\cdot}],\, \tilde g \rangle \,\middle|\, c \right]
=
L_k\left(
w_{kc}
+
\sum_{c'\neq c}w_{kc'}\,\rho
\right).
\label{eq:tilt_align_expected}
\end{equation}
Since $\sum_{c'}w_{kc'}=1$, the right-hand side can be written as
$L_k[w_{kc}+(1-w_{kc})\rho]$,
which is strictly increasing in $w_{kc}$ whenever
$1>\rho$.
\end{lemma}

\begin{proof}[Proof of Lemma~\ref{lem:tilt_alignment}]
Using Lemma~\ref{lem:tilt}, conditioning on $c$ and linearity,
\[
\mathbb{E}\!\left[\langle \mathbb{E}[G_k\mid w_k],\, \tilde g \rangle \mid c\right]
=
L_k\sum_{c'}w_{kc'}\,
\mathbb{E}\!\left[\langle u_{c'}, \tilde g\rangle \mid c\right].
\]
The prototype-overlap definitions give the two cases
$1$ for $c'=c$ and $\rho$ for $c'\neq c$,
which yields \eqref{eq:tilt_align_expected}. The coefficient of $w_{kc}$
is $1-\rho>0$.
\end{proof}

\begin{lemma}[Tilt increases per-load coherence utility for purer mixtures]
\label{lem:purity_increases_U}
Under the independent sampling model, for fixed integer $L_k>1$ and $\epsilon\geq0$, the finite-window expected utility
\[
\mathbb{E}[U_k]
=
\frac{L_k\nu+L_k(L_k-1)\left(
\rho+(1-\rho)\sum_cw_{kc}^2
\right)}{L_k+\epsilon}
\]
is strictly increasing in the purity measure $\sum_cw_{kc}^2$ whenever
$1>\rho$.
\end{lemma}

\begin{proof}[Proof of Lemma~\ref{lem:purity_increases_U}]
The expression follows from \eqref{eq:coherence_expansion}. For fixed $L_k$,
it is affine in $\sum_cw_{kc}^2$ with slope
\[
\frac{L_k(L_k-1)(1-\rho)}{L_k+\epsilon}>0.
\]
Thus purer mixtures have larger expected per-load coherence utility.
\end{proof}

\subsection{GAR -- Case 3: Collapsed Assignments Are Not Local Maxima}
\label{sec:ours_case3}

We now show that, under the geometric mode assumptions in \S\ref{sec:idealized_setting},
a collapsed assignment is \emph{not a local maximum} of the leading-order alignment
surrogate over continuous mode-to-expert assignments. An arbitrarily small
mode-separating perturbation strictly increases this surrogate.

Here the leading-order objective is defined at $\epsilon=0$ after dropping
finite-window corrections; its boundary is analyzed directly in that limit.
The theorem characterizes a static separation preference of this limiting
objective rather than the trajectory of a trained router.

\paragraph{Alignment surrogate as a purity functional.}
Recall $U_k=\|G_k\|^2/(L_k+\epsilon)$. Setting $\epsilon=0$ and retaining only
the leading $O(L_k)$ term after dividing \eqref{eq:coherence_simplified} by
$L_k$ gives the purity-dependent utility
\begin{equation}
\mathbb{E}[U_k]
\;\approx\;
L_k\Big(
\rho
+
(1-\rho)\sum_{c=1}^C w_{kc}^2
\Big).
\label{eq:Uk_purity_form}
\end{equation}
Thus the expected alignment reward has purity-dependent part
\begin{equation}
\cR(p)
=
\rho\sum_k L_k+
(1-\rho)\sum_k L_k\sum_cw_{kc}^2 .
\label{eq:R_def}
\end{equation}
Since $\sum_k L_k$ is fixed, maximizing $\cR$ is equivalent to maximizing
\begin{equation}
\Phi:=\sum_{k=1}^K L_k\sum_{c=1}^C w_{kc}^2.
\label{eq:purity_functional}
\end{equation}
Because $1>\rho$, increasing $\Phi$ strictly increases $\cR$.
Here loads are continuous mode masses. If $x_{kc}\geq0$ is the mass of mode $c$
assigned to expert $k$, then $L_k=\sum_c x_{kc}$ and
$L_k\sum_cw_{kc}^2=\sum_c x_{kc}^2/L_k$ for $L_k>0$.
We define the contribution of an empty expert to be zero, its continuous
extension at $L_k=0$. Locality is measured in these mode masses.

\paragraph{Collapsed vertex implies maximal mixing.}
Consider the (collapsed) vertex state in which a single expert $k^\star$ receives all load:
\[
L_{k^\star}=L,\qquad L_j=0\ \ (j\neq k^\star).
\]
Under the uniform manifold prior $\Pr(c)=1/C$, this implies the dominant expert receives the
entire mixture, i.e.,
\[
w_{k^\star c}=\frac{1}{C}\quad \forall c.
\]
Then the weighted purity equals
\begin{equation}
\Phi_{\mathrm{coll}}
=
L\sum_{c=1}^C \left(\frac{1}{C}\right)^2
=
\frac{L}{C}.
\label{eq:Phi_coll}
\end{equation}

\paragraph{A separating perturbation strictly increases purity.}
We now show that the collapsed state is not locally optimal for $\Phi$ by constructing
an arbitrarily small perturbation that increases it.

Fix any manifold $c_0$. Move an infinitesimal amount of load $\delta>0$ consisting solely of
mode-$c_0$ observations from expert $k^\star$ to an unused expert $j$.
Then
\[
L_j=\delta,\qquad L_{k^\star}=L-\delta.
\]
Expert $j$ becomes pure for $c_0$, hence
\[
w_{j c_0}=1,\qquad w_{j c}=0\ (c\neq c_0),
\quad\Rightarrow\quad
\sum_c w_{jc}^2 = 1.
\]
For the remaining dominant expert $k^\star$, the mode counts are
$L/C-\delta$ for $c_0$ and $L/C$ for every other mode, so
\begin{equation}
\sum_cw_{k^\star c}^2
=
\frac{(L/C-\delta)^2+(C-1)(L/C)^2}{(L-\delta)^2}
=
\frac{1}{C}+\left(1-\frac{1}{C}\right)\frac{\delta^2}{(L-\delta)^2}.
\label{eq:purity_kstar_expand}
\end{equation}
Therefore, the perturbed weighted purity satisfies
\begin{align}
\Phi_{\mathrm{pert}}
&=
(L-\delta)\sum_cw_{k^\star c}^2+\delta
\nonumber\\
&=
\frac{L}{C}
+\delta\left(1-\frac{1}{C}\right)
+\left(1-\frac{1}{C}\right)\frac{\delta^2}{L-\delta}.
\label{eq:Phi_pert_expand}
\end{align}
For any $C>1$ and sufficiently small $0<\delta<L/C$, every added term is positive:
\[
\Phi_{\mathrm{pert}}-\Phi_{\mathrm{coll}}>0.
\]

Hence $\Phi$ (and thus $\cR$) can be strictly increased by an arbitrarily small
separating perturbation.

\begin{theorem}[Collapsed assignments are not local maxima of the alignment surrogate]
\label{thm:vertex_unstable}
Consider the continuous mode-mass surrogate \eqref{eq:R_def}, with total mass
$L>0$, uniform mode masses $L/C$, $K=C>1$, and
$1>\rho$.
An assignment sending all mass to a single expert is not a local maximum of
$\cR$: every neighborhood contains a feasible mode-separating assignment with
strictly larger reward.
\end{theorem}

\begin{proof}
The perturbation above preserves every mode's total mass and is feasible for
$0<\delta<L/C$. Its distance from the collapsed assignment tends to zero as
$\delta\to0$. Equations~\eqref{eq:R_def} and~\eqref{eq:Phi_pert_expand} give
\[
\cR_{\mathrm{pert}}-\cR_{\mathrm{coll}}
=(1-\rho)
\left(1-\frac1C\right)
\left(\delta+\frac{\delta^2}{L-\delta}\right)>0.
\]
Thus every neighborhood contains a strictly improving feasible assignment,
which excludes a local maximum.
\end{proof}

\paragraph{Synthesis.}
In the illustrative baseline model, collapse arises from a load--learning positive feedback:
more load $\Rightarrow$ faster learning $\Rightarrow$ lower loss $\Rightarrow$ more load.
The alignment channel introduces a competing geometric pressure:
mixed experts reduce the weighted purity $\Phi$ and thereby reduce the attainable alignment reward.
The constructed separation of one mode from the collapsed mixture increases
$\Phi$ and hence $\cR$. Multiplying the reward by $\lambda>0$ preserves this
static preference.

Together, the three cases establish that the expected marginal score does not grow with load,
directional preference for coherent mixtures, and an improving separation
direction at collapsed assignments under their respective assumptions.
The implemented router combines the task-loss gradient with the weighted
alignment gradient, as specified in Algorithm~\ref{alg:training}.

\section{Hyperparameters and Reproducibility}
\label{app:hyperparameters_reproducibility}

\subsection{Configuration Principles}

The frozen LoRA-FFN top-4 comparisons, the trainable classification-head and
full-FFN settings, the Qwen3-8B comparison, and the DeBERTa single-task
controls use the selection protocol below within each setting. Both top-1 comparisons follow the same protocol
(Appendices~\ref{appsubsec:top1_roberta_protocol} and~\ref{appsubsec:top1_head_protocol}). Model-training learning
rate and applicable method-specific coefficients are selected independently for
each method. Baseline first selects its learning rate, gradient clipping,
and weight decay by a coordinate-wise search followed by joint local confirmation
within each setting and mixture. All other methods inherit its clipping and
weight-decay selections, which are excluded from their own hyperparameter optimization (HPO).
The warmup ratio is fixed at $0.1$ for every method and setting.
Baseline has three optimizer search coordinates but no method-specific coefficient.
STGC+Load independently selects its model-training learning rate, $\beta_{\mathrm{STGC}}$, and $\lambda_{\mathrm{load}}$ within each setting and mixture under the shared HPO procedure.
The 1,000-update RoBERTa LoRA-FFN and classification-head coefficient ablations
cover the five E8K4 mixtures. Each setting and mixture reuses its
selected GAR configuration; only $\lambda$ varies within each
sweep, with no additional HPO. The grouped
configuration tables report the optimizer settings and method-specific
coefficients used under these protocols. All supervised endpoint summaries
average five seeds and describe variation conditional on the selected
configurations.

The trainer uses a single AdamW parameter group for all trainable
parameters, with $(\beta_1,\beta_2)=(0.9,0.999)$ and optimizer
$\epsilon_{\rm opt}=10^{-8}$. Learning rate and weight decay follow the
configuration tables; the warmup schedule is specified below. These AdamW
moment and stability settings are fixed, not additional HPO coordinates.

For readability, continuous hyperparameters are displayed to at most two
significant digits.

Final runs use the complete corresponding training split. Forward accuracy and routing-selection evaluation use the complete benchmark-provided labeled evaluation split; gradient diagnostics use the checkpoint probes defined in Appendix~\ref{app:gradient_probe_definition}. Evaluation splits, numerical
precision, and runtime resources are detailed in
Appendix~\ref{appsubsec:validation_sets_and_runtime_metadata}.

\subsection{Frozen-Backbone LoRA-FFN Configurations}
\label{appsubsec:cagrad_preselection}

The frozen LoRA-FFN comparison covers seven methods on frozen RoBERTa, DeBERTa, and
Qwen3-1.7B on five dataset mixtures containing five to eight tasks; their constituent datasets are listed in brackets in the result and configuration tables.
All five mixtures use E8K4 routing and 2,000 optimizer
updates, with a common batch schedule and evaluation protocol across methods
within each backbone and mixture.

\begin{table}[H]
\centering\small
\caption{The five task mixtures used throughout the experiments. Each named dataset supplies one supervised classification task. The two seven-task mixtures are distinguished by PAWS or MRPC; these descriptors always refer to the complete dataset sets below.}
\label{tab:task_mixtures}
\setlength{\tabcolsep}{5pt}
\renewcommand{\arraystretch}{1.08}
\begin{tabular}{@{}p{0.27\linewidth}p{0.67\linewidth}@{}}
\toprule
Mixture & Constituent classification datasets \\
\midrule
Five tasks & QNLI, BoolQ, RTE, PAWS, WiC \\
Six tasks & QNLI, BoolQ, RTE, PAWS, ANLI, CB \\
Seven tasks (PAWS) & QNLI, BoolQ, RTE, PAWS, WiC, CoLA, SST-2 \\
Seven tasks (MRPC) & QNLI, BoolQ, RTE, WiC, CoLA, SST-2, MRPC \\
Eight tasks & QNLI, BoolQ, RTE, PAWS, WiC, CoLA, SST-2, CB \\
\bottomrule
\end{tabular}
\end{table}

\paragraph{Shared batch schedule and candidate evaluation.}
A DeBERTa CAGrad development-probe comparison selected the effective per-task
batch and micro-batch sizes for each multi-task mixture. Those sizes are reused
by the compared methods and backbones, including top-1 routing, classification-head,
full-parameter FFN, and Qwen3-8B extensions. The
transferred values are the batch sizes, not CAGrad's learning rate or method
coefficients. The shared candidate-evaluation procedure uses a fixed 10\%
training-derived probe, training on the remaining 90\%, and scores
final-checkpoint equal-task macro accuracy. All methods in every
comparison use this candidate-evaluation protocol.

A multi-task update consumes one batch from each task loader. Fixed update
budgets are used because the mixed-task schedule has no single epoch count.
For a mixture of $T$ tasks in Table~\ref{tab:task_mixtures}, each update
uses $32T$ examples: a batch of 32 per task, split into four same-task
groups of eight. These schedules apply to all compared methods and backbones.

\paragraph{Optimizer and coefficient selection.}

For each backbone, adaptation setting, and mixture, Baseline is searched
first over learning rate, gradient clipping in $[0.5,2.0]$, and weight decay
in $[0,0.05]$. These three coordinates are included in Baseline's coordinate-wise
selection and local confirmation under the candidate budget below. Its selected clipping and weight
decay are then locked for all other methods in that comparison. Different
mixtures or settings may use different selected values.
The warmup ratio is fixed globally at $0.1$ and is excluded from HPO.
Learning rate increases linearly over the first 10\% of
optimizer updates and remains constant at the method's selected rate thereafter.

Each non-Baseline method tunes its own model-training learning rate and applicable
coefficients: CAGrad's $c$, GAR's
$\lambda_{\mathrm{align}}$, the respective STGC, LoadPen, and SwitchAux
coefficients, and both $\beta_{\mathrm{STGC}}$ and $\lambda_{\mathrm{load}}$ for STGC+Load. CAGrad's inner learning rate controls its gradient-combination
solver and is distinct from the model-training learning rate; it is fixed at
$0.1$ in every setting, including both top-1 extensions and the single-task
controls, and is not searched.
Gradient clipping
and weight decay remain fixed during these subsequent searches and all final
runs within each setting and mixture; warmup is always fixed. The STGC
conflict threshold is fixed at $\tau_{\mathrm{STGC}}=0$. Baseline has no
method-specific coefficient to tune.

The model-training learning-rate reference range is
$[2\times10^{-6},5\times10^{-3}]$ on a log scale.
The coefficient reference ranges are CAGrad $c\in[0.1,1.0]$,
$\lambda_{\mathrm{align}}\in[10^{-5},3\times10^{-2}]$, STGC $[0.25,4]$,
LoadPen $[10^{-6},1.1\times10^{-2}]$, and SwitchAux
$[10^{-6},1.3\times10^{-2}]$. STGC+Load uses the STGC range for
$\beta_{\mathrm{STGC}}$ and the LoadPen range for $\lambda_{\mathrm{load}}$. Selection uses the training-derived probe
described above.

HPO uses one coordinate-wise pass with eight candidate values per free
coordinate, followed by local confirmation. Initial candidates are equally
spaced on a log scale for the learning rate, $\lambda_{\mathrm{align}}$,
$\beta_{\mathrm{STGC}}$, $\lambda_{\mathrm{load}}$, and the SwitchAux coefficient,
and on a linear scale for clipping, weight decay, and CAGrad $c$. Each candidate is
evaluated with two shared random seeds, using 500 optimizer updates for
multi-task selection or one complete training-loader epoch for single-task
selection.
Final evaluations use five shared random seeds and the final-run budgets
listed for each setting; the selection and final-run budgets are distinct.
For $d$ active coordinates ($d\leq3$ for every method), local confirmation uses a
three-point Cartesian grid on all $d$ coordinates, giving $3^d$ nominal
evaluations. Each local coordinate uses the
current best value and two reproducibly sampled nearby values within its
reference range. Thus final selections need not lie on the initial grid. Baseline has $d=3$ (learning rate, clipping, and
weight decay), giving
$8\times3=24$ coordinate-sweep and $3^3=27$ local-confirmation candidates
(51 in total). CAGrad, GAR, STGC, LoadPen, and SwitchAux each have $d=2$
(learning rate and one method coefficient), giving
$8\times2=16$ coordinate-sweep and $3^2=9$ local-confirmation candidates
(25 in total). STGC+Load has $d=3$ (learning rate,
$\beta_{\mathrm{STGC}}$, and $\lambda_{\mathrm{load}}$), giving
$8\times3=24$ coordinate-sweep and $3\times3\times3=27$
local-confirmation candidates (51 in total). Fixed clip/WD settings are not
reintroduced as search axes. All methods use the same per-coordinate search
density and per-candidate evaluation protocol; total candidate counts follow
the number of free coordinates. Each method selects its own learning rate and
applicable coefficients under this shared selection procedure.

The frozen LoRA-FFN experiments use LoRA rank $16$, LoRA scaling $\alpha=16$,
and dropout $0.1$. At the final transformer layer's feed-forward block, the
base intermediate and output modules are retained, and routed LoRA experts
add hidden-to-rank-to-hidden residual deltas, following the low-rank
parameterization of LoRA \citep{lora_hu2022}. The routing
objectives use $\epsilon=10^{-8}$ for numerical stabilization. These fixed
architectural choices are held constant across methods within a setting.

All five mixtures use E8K4 and sequence length 256. RoBERTa and
DeBERTa use FP32, and Qwen3-1.7B uses BF16.

\paragraph{Task-specific prediction heads.}
The frozen top-4 LoRA-FFN models share the backbone and routed FFN modules and
use a separate linear prediction head for each task. Label indices are local
to a task: the same index can denote entailment, acceptability, or a sentiment
category in different datasets. Separate heads avoid forcing these distinct
label meanings onto the same output weights, while retaining the shared
representation in which task gradients can interact. Each head has the
output width of its task; ANLI and CB therefore use three logits, and binary
tasks use two, with no padding logits. Task identity selects the prediction
head and corresponding loss; it is not supplied as an explicit router input.
The router receives token hidden states, so any task structure in its
assignments must be learned from those representations and the training
signals. All compared methods within this setting use the same head design.

The configuration tables are grouped by backbone and task count to match the
corresponding result tables. Each row gives the final optimizer settings and
method coefficients for that comparison.

\subsubsection{RoBERTa}
Tables~\ref{tab:appendix_hparams_5_6_roberta} and~\ref{tab:appendix_hparams_7_8_roberta}
report the selected learning rates, fixed clipping and
weight-decay settings, and applicable
method-specific coefficients for each mixture.

\begin{table}[H]
\centering
\scriptsize
\caption{Final RoBERTa configurations for the five--six-task mixtures. All rows use frozen backbones with trainable rank-16 LoRA-FFN experts, E8K4 routing, FP32, weight decay $0.01$, a per-task batch size of $32$, and 2,000 optimizer updates.}
\label{tab:appendix_hparams_5_6_roberta}
\setlength{\tabcolsep}{3.5pt}
\renewcommand{\arraystretch}{1.02}
\AppendixTable{%
\begin{tabular}{@{}lrrrl@{}}
\toprule
Method & Learning rate & Grad. clip & Weight decay & Method-specific setting \\
\midrule
\multicolumn{5}{@{}l}{\textbf{[QNLI, BoolQ, RTE, PAWS, WiC]} \textnormal{(5 tasks)}} \\*
Baseline  & $1.0\mathrm{e}{-3}$ & $1.1$ & $0.01$ & -- \\*
CAGrad    & $1.0\mathrm{e}{-3}$ & $1.1$ & $0.01$ & $c=0.32$; inner lr $=0.1$ \\*
GAR      & $1.0\mathrm{e}{-3}$ & $1.1$ & $0.01$ & $\lambda_{\mathrm{align}}=9\mathrm{e}{-5}$; load norm. \\*
STGC      & $1.0\mathrm{e}{-3}$ & $1.1$ & $0.01$ & $\beta_{\mathrm{STGC}}=0.25$; $\tau_{\mathrm{STGC}}=0$ \\*
LoadPen   & $1.2\mathrm{e}{-3}$ & $1.1$ & $0.01$ & $\lambda_{\mathrm{load}}=1.2\mathrm{e}{-3}$ \\*
SwitchAux & $1.0\mathrm{e}{-3}$ & $1.1$ & $0.01$ & $\alpha_{\mathrm{switch}}=1.2\mathrm{e}{-3}$ \\*
STGC+Load & $1.0\mathrm{e}{-3}$ & $1.1$ & $0.01$ & $\beta_{\mathrm{STGC}}=0.25$; $\lambda_{\mathrm{load}}=1.2\mathrm{e}{-3}$; $\tau_{\mathrm{STGC}}=0$ \\
\addlinespace[2pt]
\multicolumn{5}{@{}l}{\textbf{[QNLI, BoolQ, RTE, PAWS, ANLI, CB]} \textnormal{(6 tasks)}} \\*
Baseline  & $1.0\mathrm{e}{-3}$ & $0.8$ & $0.01$ & -- \\*
CAGrad    & $1.0\mathrm{e}{-3}$ & $0.8$ & $0.01$ & $c=0.26$; inner lr $=0.1$ \\*
GAR      & $1.0\mathrm{e}{-3}$ & $0.8$ & $0.01$ & $\lambda_{\mathrm{align}}=1.1\mathrm{e}{-4}$; load norm. \\*
STGC      & $1.0\mathrm{e}{-3}$ & $0.8$ & $0.01$ & $\beta_{\mathrm{STGC}}=0.44$; $\tau_{\mathrm{STGC}}=0$ \\*
LoadPen   & $1.0\mathrm{e}{-3}$ & $0.8$ & $0.01$ & $\lambda_{\mathrm{load}}=1.1\mathrm{e}{-4}$ \\*
SwitchAux & $1.2\mathrm{e}{-3}$ & $0.8$ & $0.01$ & $\alpha_{\mathrm{switch}}=1.1\mathrm{e}{-3}$ \\*
STGC+Load & $1.1\mathrm{e}{-3}$ & $0.8$ & $0.01$ & $\beta_{\mathrm{STGC}}=0.44$; $\lambda_{\mathrm{load}}=1.1\mathrm{e}{-4}$; $\tau_{\mathrm{STGC}}=0$ \\
\bottomrule
\end{tabular}
}%
\end{table}

\clearpage
\begin{table}[H]
\centering
\scriptsize
\caption{Final RoBERTa configurations for the seven--eight-task mixtures. All rows use frozen backbones with trainable rank-16 LoRA-FFN experts, E8K4 routing, FP32, weight decay $0.01$, a per-task batch size of $32$, and 2,000 optimizer updates.}
\label{tab:appendix_hparams_7_8_roberta}
\setlength{\tabcolsep}{3.5pt}
\renewcommand{\arraystretch}{1.02}
\AppendixTable{%
\begin{tabular}{@{}lrrrl@{}}
\toprule
Method & Learning rate & Grad. clip & Weight decay & Method-specific setting \\
\midrule
\multicolumn{5}{@{}l}{\textbf{[QNLI, BoolQ, RTE, PAWS, WiC, CoLA, SST-2]} \textnormal{(7 tasks)}} \\*
Baseline  & $1.0\mathrm{e}{-3}$ & $0.8$ & $0.01$ & -- \\*
CAGrad    & $1.0\mathrm{e}{-3}$ & $0.8$ & $0.01$ & $c=0.26$; inner lr $=0.1$ \\*
GAR      & $2.0\mathrm{e}{-3}$ & $0.8$ & $0.01$ & $\lambda_{\mathrm{align}}=1\mathrm{e}{-3}$; load norm. \\*
STGC      & $1.0\mathrm{e}{-3}$ & $0.8$ & $0.01$ & $\beta_{\mathrm{STGC}}=0.25$; $\tau_{\mathrm{STGC}}=0$ \\*
LoadPen   & $1.2\mathrm{e}{-3}$ & $0.8$ & $0.01$ & $\lambda_{\mathrm{load}}=1.2\mathrm{e}{-3}$ \\*
SwitchAux & $1.0\mathrm{e}{-3}$ & $0.8$ & $0.01$ & $\alpha_{\mathrm{switch}}=1.1\mathrm{e}{-3}$ \\*
STGC+Load & $1.0\mathrm{e}{-3}$ & $0.8$ & $0.01$ & $\beta_{\mathrm{STGC}}=0.5$; $\lambda_{\mathrm{load}}=1.2\mathrm{e}{-3}$; $\tau_{\mathrm{STGC}}=0$ \\
\addlinespace[2pt]
\multicolumn{5}{@{}l}{\textbf{[QNLI, BoolQ, RTE, PAWS, WiC, CoLA, SST-2, CB]} \textnormal{(8 tasks)}} \\*
Baseline  & $2.0\mathrm{e}{-3}$ & $0.8$ & $0.01$ & -- \\*
CAGrad    & $2.0\mathrm{e}{-3}$ & $0.8$ & $0.01$ & $c=0.35$; inner lr $=0.1$ \\*
GAR      & $1.0\mathrm{e}{-3}$ & $0.8$ & $0.01$ & $\lambda_{\mathrm{align}}=1.2\mathrm{e}{-2}$; load norm. \\*
STGC      & $2.0\mathrm{e}{-3}$ & $0.8$ & $0.01$ & $\beta_{\mathrm{STGC}}=0.33$; $\tau_{\mathrm{STGC}}=0$ \\*
LoadPen   & $2.0\mathrm{e}{-3}$ & $0.8$ & $0.01$ & $\lambda_{\mathrm{load}}=1.3\mathrm{e}{-4}$ \\*
SwitchAux & $1.0\mathrm{e}{-3}$ & $0.8$ & $0.01$ & $\alpha_{\mathrm{switch}}=1.2\mathrm{e}{-4}$ \\*
STGC+Load & $2.0\mathrm{e}{-3}$ & $0.8$ & $0.01$ & $\beta_{\mathrm{STGC}}=0.33$; $\lambda_{\mathrm{load}}=1\mathrm{e}{-3}$; $\tau_{\mathrm{STGC}}=0$ \\
\addlinespace[2pt]
\multicolumn{5}{@{}l}{\textbf{[QNLI, BoolQ, RTE, WiC, CoLA, SST-2, MRPC]} \textnormal{(7 tasks)}} \\*
Baseline  & $2.0\mathrm{e}{-3}$ & $0.8$ & $0.01$ & -- \\*
CAGrad    & $2.0\mathrm{e}{-3}$ & $0.8$ & $0.01$ & $c=0.26$; inner lr $=0.1$ \\*
GAR      & $1.0\mathrm{e}{-3}$ & $0.8$ & $0.01$ & $\lambda_{\mathrm{align}}=1.2\mathrm{e}{-4}$; load norm. \\*
STGC      & $2.0\mathrm{e}{-3}$ & $0.8$ & $0.01$ & $\beta_{\mathrm{STGC}}=0.25$; $\tau_{\mathrm{STGC}}=0$ \\*
LoadPen   & $1.8\mathrm{e}{-3}$ & $0.8$ & $0.01$ & $\lambda_{\mathrm{load}}=1.1\mathrm{e}{-3}$ \\*
SwitchAux & $2.0\mathrm{e}{-3}$ & $0.8$ & $0.01$ & $\alpha_{\mathrm{switch}}=1.1\mathrm{e}{-3}$ \\*
STGC+Load & $1.5\mathrm{e}{-3}$ & $0.8$ & $0.01$ & $\beta_{\mathrm{STGC}}=0.33$; $\lambda_{\mathrm{load}}=1\mathrm{e}{-3}$; $\tau_{\mathrm{STGC}}=0$ \\
\bottomrule
\end{tabular}
}%
\end{table}

\subsubsection{DeBERTa}
We use DeBERTaV3-base~\citep{deberta_v3_he2023}.
Tables~\ref{tab:appendix_hparams_5_6_deberta} and~\ref{tab:appendix_hparams_7_8_deberta}
report the selected learning rates, fixed clipping and
weight-decay settings, and applicable
method-specific coefficients for each mixture.

\begin{table}[H]
\centering
\scriptsize
\caption{Final DeBERTa configurations for the five--six-task mixtures. All rows use frozen backbones with trainable rank-16 LoRA-FFN experts, E8K4 routing, FP32, weight decay $0.01$, a per-task batch size of $32$, and 2,000 optimizer updates.}
\label{tab:appendix_hparams_5_6_deberta}
\setlength{\tabcolsep}{3.5pt}
\renewcommand{\arraystretch}{1.02}
\AppendixTable{%
\begin{tabular}{@{}lrrrl@{}}
\toprule
Method & Learning rate & Grad. clip & Weight decay & Method-specific setting \\
\midrule
\multicolumn{5}{@{}l}{\textbf{[QNLI, BoolQ, RTE, PAWS, WiC]} \textnormal{(5 tasks)}} \\*
Baseline & $5.1\mathrm{e}{-4}$ & $0.8$ & $0.01$ & -- \\*
CAGrad & $5.1\mathrm{e}{-4}$ & $0.8$ & $0.01$ & $c=0.35$; inner lr $=0.1$ \\*
GAR & $1.0\mathrm{e}{-3}$ & $0.8$ & $0.01$ & $\lambda_{\mathrm{align}}=1.1\mathrm{e}{-3}$; load norm. \\*
STGC & $1.0\mathrm{e}{-3}$ & $0.8$ & $0.01$ & $\beta_{\mathrm{STGC}}=0.75$; $\tau_{\mathrm{STGC}}=0$ \\*
LoadPen & $1.2\mathrm{e}{-3}$ & $0.8$ & $0.01$ & $\lambda_{\mathrm{load}}=1.2\mathrm{e}{-4}$ \\*
SwitchAux & $1.2\mathrm{e}{-3}$ & $0.8$ & $0.01$ & $\alpha_{\mathrm{switch}}=1.2\mathrm{e}{-3}$ \\*
STGC+Load & $1.1\mathrm{e}{-3}$ & $0.8$ & $0.01$ & $\beta_{\mathrm{STGC}}=0.8$; $\lambda_{\mathrm{load}}=1\mathrm{e}{-3}$; $\tau_{\mathrm{STGC}}=0$ \\
\addlinespace[2pt]
\multicolumn{5}{@{}l}{\textbf{[QNLI, BoolQ, RTE, PAWS, ANLI, CB]} \textnormal{(6 tasks)}} \\*
Baseline & $1.0\mathrm{e}{-3}$ & $1.1$ & $0.01$ & -- \\*
CAGrad & $1.0\mathrm{e}{-3}$ & $1.1$ & $0.01$ & $c=0.5$; inner lr $=0.1$ \\*
GAR & $2.0\mathrm{e}{-3}$ & $1.1$ & $0.01$ & $\lambda_{\mathrm{align}}=1.2\mathrm{e}{-3}$; load norm. \\*
STGC & $1.0\mathrm{e}{-3}$ & $1.1$ & $0.01$ & $\beta_{\mathrm{STGC}}=0.5$; $\tau_{\mathrm{STGC}}=0$ \\*
LoadPen & $2.0\mathrm{e}{-3}$ & $1.1$ & $0.01$ & $\lambda_{\mathrm{load}}=1.2\mathrm{e}{-3}$ \\*
SwitchAux & $1.0\mathrm{e}{-3}$ & $1.1$ & $0.01$ & $\alpha_{\mathrm{switch}}=1.3\mathrm{e}{-2}$ \\*
STGC+Load & $1.0\mathrm{e}{-3}$ & $1.1$ & $0.01$ & $\beta_{\mathrm{STGC}}=0.5$; $\lambda_{\mathrm{load}}=1.2\mathrm{e}{-3}$; $\tau_{\mathrm{STGC}}=0$ \\
\bottomrule
\end{tabular}
}%
\end{table}

\begin{table}[H]
\centering
\scriptsize
\caption{Final DeBERTa configurations for the seven--eight-task mixtures. All rows use frozen backbones with trainable rank-16 LoRA-FFN experts, E8K4 routing, FP32, weight decay $0.01$, a per-task batch size of $32$, and 2,000 optimizer updates.}
\label{tab:appendix_hparams_7_8_deberta}
\setlength{\tabcolsep}{3.5pt}
\renewcommand{\arraystretch}{1.02}
\AppendixTable{%
\begin{tabular}{@{}lrrrl@{}}
\toprule
Method & Learning rate & Grad. clip & Weight decay & Method-specific setting \\
\midrule
\multicolumn{5}{@{}l}{\textbf{[QNLI, BoolQ, RTE, PAWS, WiC, CoLA, SST-2]} \textnormal{(7 tasks)}} \\*
Baseline & $1.0\mathrm{e}{-3}$ & $1.1$ & $0.01$ & -- \\*
CAGrad & $1.0\mathrm{e}{-3}$ & $1.1$ & $0.01$ & $c=0.5$; inner lr $=0.1$ \\*
GAR & $2.0\mathrm{e}{-3}$ & $1.1$ & $0.01$ & $\lambda_{\mathrm{align}}=1\mathrm{e}{-2}$; load norm. \\*
STGC & $1.0\mathrm{e}{-3}$ & $1.1$ & $0.01$ & $\beta_{\mathrm{STGC}}=0.4$; $\tau_{\mathrm{STGC}}=0$ \\*
LoadPen & $1.2\mathrm{e}{-3}$ & $1.1$ & $0.01$ & $\lambda_{\mathrm{load}}=1.1\mathrm{e}{-3}$ \\*
SwitchAux & $2.0\mathrm{e}{-3}$ & $1.1$ & $0.01$ & $\alpha_{\mathrm{switch}}=1\mathrm{e}{-3}$ \\*
STGC+Load & $1.0\mathrm{e}{-3}$ & $1.1$ & $0.01$ & $\beta_{\mathrm{STGC}}=0.4$; $\lambda_{\mathrm{load}}=1\mathrm{e}{-3}$; $\tau_{\mathrm{STGC}}=0$ \\
\addlinespace[2pt]
\multicolumn{5}{@{}l}{\textbf{[QNLI, BoolQ, RTE, PAWS, WiC, CoLA, SST-2, CB]} \textnormal{(8 tasks)}} \\*
Baseline & $1.0\mathrm{e}{-3}$ & $1.1$ & $0.01$ & -- \\*
CAGrad & $1.0\mathrm{e}{-3}$ & $1.1$ & $0.01$ & $c=0.35$; inner lr $=0.1$ \\*
GAR & $2.0\mathrm{e}{-3}$ & $1.1$ & $0.01$ & $\lambda_{\mathrm{align}}=1\mathrm{e}{-3}$; load norm. \\*
STGC & $2.0\mathrm{e}{-3}$ & $1.1$ & $0.01$ & $\beta_{\mathrm{STGC}}=0.33$; $\tau_{\mathrm{STGC}}=0$ \\*
LoadPen & $2.0\mathrm{e}{-3}$ & $1.1$ & $0.01$ & $\lambda_{\mathrm{load}}=1.2\mathrm{e}{-4}$ \\*
SwitchAux & $2.0\mathrm{e}{-3}$ & $1.1$ & $0.01$ & $\alpha_{\mathrm{switch}}=1.1\mathrm{e}{-4}$ \\*
STGC+Load & $2.0\mathrm{e}{-3}$ & $1.1$ & $0.01$ & $\beta_{\mathrm{STGC}}=0.4$; $\lambda_{\mathrm{load}}=1\mathrm{e}{-2}$; $\tau_{\mathrm{STGC}}=0$ \\
\addlinespace[2pt]
\multicolumn{5}{@{}l}{\textbf{[QNLI, BoolQ, RTE, WiC, CoLA, SST-2, MRPC]} \textnormal{(7 tasks)}} \\*
Baseline & $1.0\mathrm{e}{-3}$ & $1.1$ & $0.01$ & -- \\*
CAGrad & $1.0\mathrm{e}{-3}$ & $1.1$ & $0.01$ & $c=0.5$; inner lr $=0.1$ \\*
GAR & $2.0\mathrm{e}{-3}$ & $1.1$ & $0.01$ & $\lambda_{\mathrm{align}}=1.2\mathrm{e}{-3}$; load norm. \\*
STGC & $2.0\mathrm{e}{-3}$ & $1.1$ & $0.01$ & $\beta_{\mathrm{STGC}}=2$; $\tau_{\mathrm{STGC}}=0$ \\*
LoadPen & $1.0\mathrm{e}{-3}$ & $1.1$ & $0.01$ & $\lambda_{\mathrm{load}}=1.1\mathrm{e}{-3}$ \\*
SwitchAux & $1.0\mathrm{e}{-3}$ & $1.1$ & $0.01$ & $\alpha_{\mathrm{switch}}=1.1\mathrm{e}{-4}$ \\*
STGC+Load & $2.0\mathrm{e}{-3}$ & $1.1$ & $0.01$ & $\beta_{\mathrm{STGC}}=2.0$; $\lambda_{\mathrm{load}}=1.2\mathrm{e}{-3}$; $\tau_{\mathrm{STGC}}=0$ \\
\bottomrule
\end{tabular}
}%
\end{table}

\subsubsection{Qwen3-1.7B}
Following the Baseline-first selection procedure above, the selected clipping
values are shared within each mixture; the selected weight decay is $0.01$
throughout. Subsequent method searches tune model learning rate and applicable
coefficients while inheriting these controls. Warmup remains fixed at $0.1$.
Tables~\ref{tab:appendix_hparams_5_6_qwen17} and~\ref{tab:appendix_hparams_7_8_qwen17}
report the final configurations.

\begin{table}[H]
\centering
\scriptsize
\caption{Final Qwen3-1.7B configurations for the five--six-task mixtures. All rows use frozen backbones with trainable rank-16 LoRA-FFN experts, E8K4 routing, BF16, weight decay $0.01$, a per-task batch size of $32$, and 2,000 optimizer updates.}
\label{tab:appendix_hparams_5_6_qwen17}
\setlength{\tabcolsep}{3.8pt}
\renewcommand{\arraystretch}{1.02}
\AppendixTable{%
\begin{tabular}{@{}lrrl@{}}
\toprule
Method & Learning rate & Grad. clip & Method-specific setting \\
\midrule
\multicolumn{4}{@{}l}{\textbf{[QNLI, BoolQ, RTE, PAWS, WiC]} \textnormal{(5 tasks)}} \\*
Baseline & $1.3\mathrm{e}{-3}$ & $1.2$ & -- \\*
CAGrad & $7.0\mathrm{e}{-4}$ & $1.2$ & $c=0.50$; inner lr $=0.1$ \\*
GAR & $1.1\mathrm{e}{-3}$ & $1.2$ & $\lambda_{\mathrm{align}}=1.8\mathrm{e}{-4}$; load norm. \\*
STGC & $1.0\mathrm{e}{-3}$ & $1.2$ & $\beta_{\mathrm{STGC}}=0.75$; $\tau_{\mathrm{STGC}}=0$ \\*
LoadPen & $7.0\mathrm{e}{-4}$ & $1.2$ & $\lambda_{\mathrm{load}}=1.2\mathrm{e}{-3}$ \\*
SwitchAux & $1.2\mathrm{e}{-3}$ & $1.2$ & $\alpha_{\mathrm{switch}}=1\mathrm{e}{-3}$ \\*
STGC+Load & $1.0\mathrm{e}{-3}$ & $1.2$ & $\beta_{\mathrm{STGC}}=0.75$; $\lambda_{\mathrm{load}}=1.2\mathrm{e}{-4}$; $\tau_{\mathrm{STGC}}=0$ \\
\addlinespace[2pt]
\multicolumn{4}{@{}l}{\textbf{[QNLI, BoolQ, RTE, PAWS, ANLI, CB]} \textnormal{(6 tasks)}} \\*
Baseline & $1.0\mathrm{e}{-3}$ & $0.6$ & -- \\*
CAGrad & $1.0\mathrm{e}{-3}$ & $0.6$ & $c=0.34$; inner lr $=0.1$ \\*
GAR & $5.1\mathrm{e}{-4}$ & $0.6$ & $\lambda_{\mathrm{align}}=1.4\mathrm{e}{-3}$; load norm. \\*
STGC & $1.2\mathrm{e}{-3}$ & $0.6$ & $\beta_{\mathrm{STGC}}=0.9$; $\tau_{\mathrm{STGC}}=0$ \\*
LoadPen & $1.0\mathrm{e}{-3}$ & $0.6$ & $\lambda_{\mathrm{load}}=1.3\mathrm{e}{-3}$ \\*
SwitchAux & $7.0\mathrm{e}{-4}$ & $0.6$ & $\alpha_{\mathrm{switch}}=1.2\mathrm{e}{-4}$ \\*
STGC+Load & $1.0\mathrm{e}{-3}$ & $0.6$ & $\beta_{\mathrm{STGC}}=0.8$; $\lambda_{\mathrm{load}}=1\mathrm{e}{-3}$; $\tau_{\mathrm{STGC}}=0$ \\
\bottomrule
\end{tabular}
}%
\end{table}

\begin{table}[H]
\centering
\scriptsize
\caption{Final Qwen3-1.7B configurations for the seven--eight-task mixtures. All rows use frozen backbones with trainable rank-16 LoRA-FFN experts, E8K4 routing, BF16, weight decay $0.01$, a per-task batch size of $32$, and 2,000 optimizer updates.}
\label{tab:appendix_hparams_7_8_qwen17}
\setlength{\tabcolsep}{3.8pt}
\renewcommand{\arraystretch}{1.02}
\AppendixTable{%
\begin{tabular}{@{}lrrl@{}}
\toprule
Method & Learning rate & Grad. clip & Method-specific setting \\
\midrule
\multicolumn{4}{@{}l}{\textbf{[QNLI, BoolQ, RTE, PAWS, WiC, CoLA, SST-2]} \textnormal{(7 tasks)}} \\*
Baseline & $5.0\mathrm{e}{-5}$ & $1.0$ & -- \\*
CAGrad & $1.0\mathrm{e}{-4}$ & $1.0$ & $c=0.32$; inner lr $=0.1$ \\*
GAR & $1.0\mathrm{e}{-4}$ & $1.0$ & $\lambda_{\mathrm{align}}=1\mathrm{e}{-3}$; load norm. \\*
STGC & $1.0\mathrm{e}{-4}$ & $1.0$ & $\beta_{\mathrm{STGC}}=0.8$; $\tau_{\mathrm{STGC}}=0$ \\*
LoadPen & $1.0\mathrm{e}{-4}$ & $1.0$ & $\lambda_{\mathrm{load}}=1.2\mathrm{e}{-3}$ \\*
SwitchAux & $1.1\mathrm{e}{-4}$ & $1.0$ & $\alpha_{\mathrm{switch}}=1\mathrm{e}{-3}$ \\*
STGC+Load & $1.0\mathrm{e}{-4}$ & $1.0$ & $\beta_{\mathrm{STGC}}=0.9$; $\lambda_{\mathrm{load}}=1.2\mathrm{e}{-4}$; $\tau_{\mathrm{STGC}}=0$ \\
\addlinespace[2pt]
\multicolumn{4}{@{}l}{\textbf{[QNLI, BoolQ, RTE, PAWS, WiC, CoLA, SST-2, CB]} \textnormal{(8 tasks)}} \\*
Baseline & $1.0\mathrm{e}{-4}$ & $1.0$ & -- \\*
CAGrad & $1.0\mathrm{e}{-4}$ & $1.0$ & $c=0.32$; inner lr $=0.1$ \\*
GAR & $4.0\mathrm{e}{-4}$ & $1.0$ & $\lambda_{\mathrm{align}}=3\mathrm{e}{-3}$; load norm. \\*
STGC & $1.2\mathrm{e}{-4}$ & $1.0$ & $\beta_{\mathrm{STGC}}=1.2$; $\tau_{\mathrm{STGC}}=0$ \\*
LoadPen & $2.0\mathrm{e}{-4}$ & $1.0$ & $\lambda_{\mathrm{load}}=1.3\mathrm{e}{-3}$ \\*
SwitchAux & $2.0\mathrm{e}{-4}$ & $1.0$ & $\alpha_{\mathrm{switch}}=1.3\mathrm{e}{-2}$ \\*
STGC+Load & $1.0\mathrm{e}{-4}$ & $1.0$ & $\beta_{\mathrm{STGC}}=1.2$; $\lambda_{\mathrm{load}}=1\mathrm{e}{-3}$; $\tau_{\mathrm{STGC}}=0$ \\
\addlinespace[2pt]
\multicolumn{4}{@{}l}{\textbf{[QNLI, BoolQ, RTE, WiC, CoLA, SST-2, MRPC]} \textnormal{(7 tasks)}} \\*
Baseline & $5.0\mathrm{e}{-5}$ & $1.0$ & -- \\*
CAGrad & $1.0\mathrm{e}{-4}$ & $1.0$ & $c=0.32$; inner lr $=0.1$ \\*
GAR & $1.0\mathrm{e}{-4}$ & $1.0$ & $\lambda_{\mathrm{align}}=1\mathrm{e}{-3}$; load norm. \\*
STGC & $1.0\mathrm{e}{-4}$ & $1.0$ & $\beta_{\mathrm{STGC}}=1.5$; $\tau_{\mathrm{STGC}}=0$ \\*
LoadPen & $1.0\mathrm{e}{-4}$ & $1.0$ & $\lambda_{\mathrm{load}}=1.1\mathrm{e}{-2}$ \\*
SwitchAux & $1.0\mathrm{e}{-4}$ & $1.0$ & $\alpha_{\mathrm{switch}}=1.2\mathrm{e}{-4}$ \\*
STGC+Load & $1.2\mathrm{e}{-4}$ & $1.0$ & $\beta_{\mathrm{STGC}}=1.5$; $\lambda_{\mathrm{load}}=1.1\mathrm{e}{-3}$; $\tau_{\mathrm{STGC}}=0$ \\
\bottomrule
\end{tabular}
}%
\end{table}

\subsection{Qwen3-8B Final Configurations}
\label{appsubsec:qwen8b_hyperparameters}
The three Qwen3-8B methods reuse the shared per-mixture batch and micro-batch sizes and the Baseline-first HPO procedure described in Appendix~\ref{appsubsec:cagrad_preselection}. Final evaluation uses five seeds with 2,000 updates and BF16. The selected weight decay is 0.010 throughout and is shared across methods within each mixture; CAGrad's inner learning rate is fixed at 0.10. GAR uses load normalization. Token routing and auxiliary group averaging follow Section~\ref{sec:methodology_implementation}.
The LoRA experts, per-task prediction heads, sequence length, and warmup follow Appendix~\ref{appsubsec:cagrad_preselection}.
Table~\ref{tab:qwen8b_hp_5_8} reports the selected values to two significant digits. All five mixtures use eight experts with top-4 routing and 32 examples per task with micro-batches of eight.
\begin{table}[H]
\centering\footnotesize
\caption{Selected Qwen3-8B configurations for the five--eight-task mixtures. Dashes indicate an inapplicable method coefficient.}
\label{tab:qwen8b_hp_5_8}
\setlength{\tabcolsep}{8pt}
\renewcommand{\arraystretch}{1.05}
\begin{tabular}{@{}lrrrr@{}}
\toprule
Method & Learning rate & Grad. clip & CAGrad $c$ & GAR $\lambda$ \\
\midrule
\multicolumn{5}{@{}l}{\textbf{[QNLI, BoolQ, RTE, PAWS, WiC]}} \\*
Baseline & $1.1\times10^{-4}$ & 1.2 & -- & -- \\*
CAGrad & $1.4\times10^{-4}$ & 1.2 & 0.22 & -- \\*
GAR & $1.4\times10^{-4}$ & 1.2 & -- & $1.8\times10^{-4}$ \\*
\addlinespace[3pt]
\multicolumn{5}{@{}l}{\textbf{[QNLI, BoolQ, RTE, PAWS, ANLI, CB]}} \\*
Baseline & $7.4\times10^{-5}$ & 1.0 & -- & -- \\*
CAGrad & $1.1\times10^{-4}$ & 1.0 & 0.40 & -- \\*
GAR & $1.1\times10^{-4}$ & 1.0 & -- & $1.4\times10^{-3}$ \\*
\addlinespace[3pt]
\multicolumn{5}{@{}l}{\textbf{[QNLI, BoolQ, RTE, PAWS, WiC, CoLA, SST-2]}} \\*
Baseline & $1.1\times10^{-4}$ & 1.2 & -- & -- \\*
CAGrad & $7.1\times10^{-5}$ & 1.2 & 0.50 & -- \\*
GAR & $1.2\times10^{-4}$ & 1.2 & -- & $1.3\times10^{-4}$ \\*
\addlinespace[3pt]
\multicolumn{5}{@{}l}{\textbf{[QNLI, BoolQ, RTE, WiC, CoLA, SST-2, MRPC]}} \\*
Baseline & $1.4\times10^{-4}$ & 1.2 & -- & -- \\*
CAGrad & $1.2\times10^{-4}$ & 1.2 & 0.34 & -- \\*
GAR & $1.2\times10^{-4}$ & 1.2 & -- & $1.5\times10^{-4}$ \\*
\addlinespace[3pt]
\multicolumn{5}{@{}l}{\textbf{[QNLI, BoolQ, RTE, PAWS, WiC, CoLA, SST-2, CB]}} \\*
Baseline & $7.1\times10^{-5}$ & 1.0 & -- & -- \\*
CAGrad & $7.1\times10^{-5}$ & 1.0 & 0.34 & -- \\*
GAR & $1.3\times10^{-4}$ & 1.0 & -- & $1.1\times10^{-4}$ \\*
\bottomrule
\end{tabular}
\end{table}

\subsection{Frozen RoBERTa Top-1 LoRA-FFN}
\label{appsubsec:top1_roberta_protocol}

This extension evaluates Baseline, CAGrad, and GAR on the same five
dataset mixtures with frozen RoBERTa-base and eight rank-16
LoRA experts ($\alpha=16$) in the final-layer FFN. One expert is selected
per token (E8K1). Training uses FP32, sequence length 256, the
per-task batch size 32 and micro-batch size 8 of Appendix~\ref{appsubsec:cagrad_preselection}, warmup ratio $0.1$, and
2,000 optimizer updates. Final evaluations use the full labeled evaluation
splits and the same five shared random seeds.

\paragraph{Straight-through top-1 routing.}
For token-router logits $z$, define $\sigma=\operatorname{softmax}(z)$ at unit
temperature and $o=\operatorname{onehot}(\arg\max_e z_e)$. The configured
gate is
\[
\pi(z)=\sg(o-\sigma)+\sigma.
\]
Its forward value is exactly one-hot, while its backward derivative is the
full-softmax Jacobian. All three methods use this rule in the task forward
computation. GAR also uses it when recomputing auxiliary gates from detached
routing inputs. The auxiliary branch averages the configured gates over each
example's valid tokens and then over the examples in each same-task
micro-batch. It pairs these group probabilities with detached expert-gradient
observations and uses the load-normalized objective of
Section~\ref{sec:methodology_implementation}, with $\epsilon=10^{-8}$.

\paragraph{Prediction heads and configurations.}
Each task has its own linear prediction head, as in the frozen top-4 setting,
with three outputs for ANLI and CB and two for the binary tasks.
Gradient clipping and weight decay are shared within each mixture;
Table~\ref{tab:appendix_top1_roberta_hparams} gives the selected
method-specific learning rates and coefficients. CAGrad's inner learning
rate is fixed at $0.1$; the selected $\lambda_{\mathrm{align}}$ values are
$1.2\mathrm{e}{-3}$, $1\mathrm{e}{-3}$, $1\mathrm{e}{-3}$,
$1.2\mathrm{e}{-3}$, and $1\mathrm{e}{-3}$ for the five-task, six-task,
seven-task (PAWS), seven-task (MRPC), and eight-task mixtures, respectively.
GAR uses load normalization.

\paragraph{Search protocol.}
The top-1 search uses the candidate-evaluation protocol of Appendix~\ref{appsubsec:cagrad_preselection}:
500 updates with two shared random seeds, training on 90\% of each training split and
scoring final-checkpoint macro accuracy on its fixed 10\% development probe.
Baseline first searches its model-training learning rate, clipping, and
weight decay over the reference ranges in
Appendix~\ref{appsubsec:cagrad_preselection}. After local confirmation,
CAGrad and GAR inherit only Baseline's selected clipping and weight decay.
CAGrad searches its own model-training learning rate and $c$, fixing its
inner learning rate at $0.1$; GAR searches its own learning rate and
$\lambda_{\mathrm{align}}$. Each free coordinate receives eight candidates,
followed by a three-point local grid on all $d$ free coordinates.
Baseline therefore uses 51 candidate configurations, and CAGrad and GAR
each use 25. The batch schedule remains fixed throughout selection and
the 2,000-update final runs.

\begin{table}[H]
\centering
\scriptsize
\caption{Final configurations for the frozen RoBERTa top-1 LoRA-FFN comparison. All five mixtures use E8K1 straight-through routing, rank 16, scaling $\alpha=16$, FP32, sequence length 256, per-task batch size 32, micro-batch size 8, warmup ratio 0.1, and 2,000 optimizer updates. Gradient clipping and weight decay are shared by the three methods within each mixture.}
\label{tab:appendix_top1_roberta_hparams}
\setlength{\tabcolsep}{4pt}
\renewcommand{\arraystretch}{1.03}
\AppendixTable{%
\begin{tabular}{@{}lrrrl@{}}
\toprule
Method & Learning rate & Grad.\ clip & Weight decay & Method-specific setting \\
\midrule
\multicolumn{5}{@{}l}{\textbf{[QNLI, BoolQ, RTE, PAWS, WiC]} \textnormal{(5 tasks)}} \\*
Baseline & $1.0\mathrm{e}{-3}$ & $1.1$ & $0.02$ & -- \\*
CAGrad & $1.5\mathrm{e}{-3}$ & $1.1$ & $0.02$ & $c=0.42$; inner lr $=0.1$ \\*
\textbf{GAR} & $1.5\mathrm{e}{-3}$ & $1.1$ & $0.02$ & $\lambda_{\mathrm{align}}=1.2\mathrm{e}{-3}$; load norm. \\
\midrule
\multicolumn{5}{@{}l}{\textbf{[QNLI, BoolQ, RTE, PAWS, ANLI, CB]} \textnormal{(6 tasks)}} \\*
Baseline & $1.2\mathrm{e}{-3}$ & $1.1$ & $0.01$ & -- \\*
CAGrad & $1.5\mathrm{e}{-3}$ & $1.1$ & $0.01$ & $c=0.26$; inner lr $=0.1$ \\*
\textbf{GAR} & $1.5\mathrm{e}{-3}$ & $1.1$ & $0.01$ & $\lambda_{\mathrm{align}}=1\mathrm{e}{-3}$; load norm. \\
\midrule
\multicolumn{5}{@{}l}{\textbf{[QNLI, BoolQ, RTE, PAWS, WiC, CoLA, SST-2]} \textnormal{(7 tasks)}} \\*
Baseline & $2.0\mathrm{e}{-3}$ & $1.0$ & $0.02$ & -- \\*
CAGrad & $2.2\mathrm{e}{-3}$ & $1.0$ & $0.02$ & $c=0.26$; inner lr $=0.1$ \\*
\textbf{GAR} & $2.0\mathrm{e}{-3}$ & $1.0$ & $0.02$ & $\lambda_{\mathrm{align}}=1\mathrm{e}{-3}$; load norm. \\
\midrule
\multicolumn{5}{@{}l}{\textbf{[QNLI, BoolQ, RTE, WiC, CoLA, SST-2, MRPC]} \textnormal{(7 tasks)}} \\*
Baseline & $1.0\mathrm{e}{-3}$ & $1.1$ & $0.02$ & -- \\*
CAGrad & $1.2\mathrm{e}{-3}$ & $1.1$ & $0.02$ & $c=0.26$; inner lr $=0.1$ \\*
\textbf{GAR} & $1.1\mathrm{e}{-3}$ & $1.1$ & $0.02$ & $\lambda_{\mathrm{align}}=1.2\mathrm{e}{-3}$; load norm. \\
\midrule
\multicolumn{5}{@{}l}{\textbf{[QNLI, BoolQ, RTE, PAWS, WiC, CoLA, SST-2, CB]} \textnormal{(8 tasks)}} \\*
Baseline & $1.2\mathrm{e}{-3}$ & $1.0$ & $0.01$ & -- \\*
CAGrad & $1.0\mathrm{e}{-3}$ & $1.0$ & $0.01$ & $c=0.26$; inner lr $=0.1$ \\*
\textbf{GAR} & $1.1\mathrm{e}{-3}$ & $1.0$ & $0.01$ & $\lambda_{\mathrm{align}}=1\mathrm{e}{-3}$; load norm. \\
\bottomrule
\end{tabular}
}%
\end{table}

\subsection{Trainable RoBERTa Classification-Head MoE}
\label{appsubsec:implb_protocol}

The compared methods use the same batch schedule, update budget, and evaluation
protocol within this extension. Baseline first selects learning rate, clipping,
and weight decay using Appendix~\ref{appsubsec:cagrad_preselection}'s procedure.
Subsequent methods inherit the selected clipping and weight decay and tune
their own learning rates and applicable coefficients. Warmup is fixed at $0.1$.

This setting uses the same routing objective and implementation pathway to test
whether the five-to-eight-task accuracy gain and the routing trade-off depend on
expert placement and trainability, and to record full routing-diagnostic
trajectories. It differs from the frozen LoRA-FFN setting on three axes at
once. First, the RoBERTa-base backbone is fully trainable. Second, the experts
are rank-16 LoRA deltas (scaling $\alpha=16$, dropout $0.1$) applied to the
classification head over a shared base head, so the combined logits are
$\mathrm{base}(h)+\sum_k p_k\,\Delta_k(h)$ for the pooled representation $h$.
Every task in a mixture uses this same logit space.  The output width is three
when the mixture contains a three-way task (ANLI or CB), and two otherwise.
Binary-task targets remain indexed by $\{0,1\}$ in a three-logit mixture; the
third logit participates in the common softmax normalization but is never a
binary-task target.  The shared-width rule is identical for every compared
method in the classification-head setting.
Third, the router is a linear map on the masked-mean pooled representation, so
routing is sequence-level, using top-4 masked-softmax routing over eight
experts (E8K4) in all five mixtures. Router numerics are FP32; the
remaining forward and backward computation uses BF16.

\paragraph{Shared-output stress test.}
We use this extension to examine routing under a shared prediction space.
It has no task-specific output head: the shared base head and the
input-conditioned mixture of expert deltas jointly produce the logits.
Task identity is used to organize task losses and gradient observations,
but is not supplied to the head or router. Both the trainable representation
and the router can learn task differences from the inputs; expert mixing is
one component of this adaptation.

Sharing output weights aliases task-local label indices with different
semantics or polarity, creating competition in the shared prediction space.
In three-logit mixtures, all three classes participate in softmax and
prediction for binary examples as well. Every compared method uses this
same shared-output convention.

The alignment pathway is the one of
Section~\ref{sec:methodology_implementation} with the routed unit being a
same-task example group: $g_m$ sums the group-loss gradients across
experts in the common LoRA parameter template and is detached as
$\tilde g_m=\sg(g_m)$ for the auxiliary loss. From detached pooled
representations $\sg(h_n)$, the auxiliary branch recomputes per-example router
probabilities $q_n$ using top-4 masked softmax over eight experts. It then sets $p_m=|m|^{-1}\sum_{n\in m}q_n$, with no
further top-$k$ operation on the group mean. Thus $p_m$ can have more than $k$
nonzero entries even when each example selects only $k$ experts. The
$\mathcal{L}_{\mathrm{norm}}$ objective with $\epsilon=10^{-8}$ has no direct
gradient to non-router parameters; detaching $h_n$ preserves its derivatives
with respect to router parameters while blocking this branch from the
backbone. Groups contain eight examples;
the per-task batch schedule and the five seeds
used throughout match the frozen LoRA-FFN setting, and final runs use
2,000 updates. The final optimizer and
method-specific values are reported in
Tables~\ref{tab:appendix_implb_hparams_5_6} and~\ref{tab:appendix_implb_hparams_7_8}.

\begin{table}[H]
\centering
\scriptsize
\caption{Final trainable RoBERTa classification-head configurations for the five--six-task mixtures: top-4 masked-softmax routing over eight experts (E8K4), sequence length 256, per-task batch size 32, and same-task groups of eight examples. All rows use rank-16 LoRA experts on the classification head over a shared base head, 2,000 multi-task updates, BF16 forward/backward with FP32 router numerics, and the CAGrad-selected batch schedule of Appendix~\ref{app:hyperparameters_reproducibility}.}
\label{tab:appendix_implb_hparams_5_6}
\setlength{\tabcolsep}{3.5pt}
\renewcommand{\arraystretch}{1.02}
\AppendixTable{%
\begin{tabular}{@{}lrrrl@{}}
\toprule
Method & Learning rate & Grad.\ clip & Weight decay & Method-specific setting \\
\midrule
\multicolumn{5}{@{}l}{\textbf{[QNLI, BoolQ, RTE, PAWS, WiC]} \textnormal{(5 tasks)}} \\*
Baseline  & $1.0\mathrm{e}{-4}$ & $1.0$ & $0.032$ & -- \\*
CAGrad    & $5.7\mathrm{e}{-5}$ & $1.0$ & $0.032$ & $c=0.22$; inner lr $=0.1$ \\*
GAR      & $1.0\mathrm{e}{-4}$ & $1.0$ & $0.032$ & $\lambda_{\mathrm{align}}=1\mathrm{e}{-3}$; load norm. \\*
STGC      & $1.0\mathrm{e}{-4}$ & $1.0$ & $0.032$ & $\beta_{\mathrm{STGC}}=0.75$; $\tau_{\mathrm{STGC}}=0$ \\*
LoadPen   & $1.0\mathrm{e}{-4}$ & $1.0$ & $0.032$ & $\lambda_{\mathrm{load}}=1.2\mathrm{e}{-4}$ \\*
SwitchAux & $1.0\mathrm{e}{-4}$ & $1.0$ & $0.032$ & $\alpha_{\mathrm{switch}}=1.2\mathrm{e}{-5}$ \\*
STGC+Load & $1.2\mathrm{e}{-4}$ & $1.0$ & $0.032$ & $\beta_{\mathrm{STGC}}=0.5$; $\lambda_{\mathrm{load}}=1.1\mathrm{e}{-4}$; $\tau_{\mathrm{STGC}}=0$ \\
\addlinespace[2pt]
\multicolumn{5}{@{}l}{\textbf{[QNLI, BoolQ, RTE, PAWS, ANLI, CB]} \textnormal{(6 tasks)}} \\*
Baseline  & $1.2\mathrm{e}{-4}$ & $0.8$ & $0.01$ & -- \\*
CAGrad    & $1.4\mathrm{e}{-4}$ & $0.8$ & $0.01$ & $c=0.26$; inner lr $=0.1$ \\*
GAR      & $1.2\mathrm{e}{-4}$ & $0.8$ & $0.01$ & $\lambda_{\mathrm{align}}=1\mathrm{e}{-2}$; load norm. \\*
STGC      & $1.3\mathrm{e}{-4}$ & $0.8$ & $0.01$ & $\beta_{\mathrm{STGC}}=0.5$; $\tau_{\mathrm{STGC}}=0$ \\*
LoadPen   & $1.4\mathrm{e}{-4}$ & $0.8$ & $0.01$ & $\lambda_{\mathrm{load}}=1.3\mathrm{e}{-3}$ \\*
SwitchAux & $1.2\mathrm{e}{-4}$ & $0.8$ & $0.01$ & $\alpha_{\mathrm{switch}}=1.2\mathrm{e}{-2}$ \\*
STGC+Load & $1.1\mathrm{e}{-4}$ & $0.8$ & $0.01$ & $\beta_{\mathrm{STGC}}=0.5$; $\lambda_{\mathrm{load}}=1.1\mathrm{e}{-3}$; $\tau_{\mathrm{STGC}}=0$ \\
\bottomrule
\end{tabular}
}%
\end{table}

\begin{table}[H]
\centering
\scriptsize
\caption{Final trainable RoBERTa classification-head configurations for the seven--eight-task mixtures: E8K4 routing, sequence length 256, per-task batch size 32, and same-task groups of eight examples. All rows use rank-16 LoRA experts on the classification head over a shared base head, 2,000 multi-task updates, BF16 forward/backward with FP32 router numerics, and the CAGrad-selected batch schedule of Appendix~\ref{app:hyperparameters_reproducibility}.}
\label{tab:appendix_implb_hparams_7_8}
\setlength{\tabcolsep}{3.5pt}
\renewcommand{\arraystretch}{1.02}
\AppendixTable{%
\begin{tabular}{@{}lrrrl@{}}
\toprule
Method & Learning rate & Grad.\ clip & Weight decay & Method-specific setting \\
\midrule
\multicolumn{5}{@{}l}{\textbf{[QNLI, BoolQ, RTE, PAWS, WiC, CoLA, SST-2]} \textnormal{(7 tasks)}} \\*
Baseline  & $7.0\mathrm{e}{-5}$ & $0.8$ & $0.01$ & -- \\*
CAGrad    & $7.4\mathrm{e}{-5}$ & $0.8$ & $0.01$ & $c=0.28$; inner lr $=0.1$ \\*
GAR      & $7.0\mathrm{e}{-5}$ & $0.8$ & $0.01$ & $\lambda_{\mathrm{align}}=1.2\mathrm{e}{-3}$; load norm. \\*
STGC      & $7.0\mathrm{e}{-5}$ & $0.8$ & $0.01$ & $\beta_{\mathrm{STGC}}=0.25$; $\tau_{\mathrm{STGC}}=0$ \\*
LoadPen   & $7.0\mathrm{e}{-5}$ & $0.8$ & $0.01$ & $\lambda_{\mathrm{load}}=1.2\mathrm{e}{-3}$ \\*
SwitchAux & $4.4\mathrm{e}{-5}$ & $0.8$ & $0.01$ & $\alpha_{\mathrm{switch}}=1\mathrm{e}{-3}$ \\*
STGC+Load & $7.0\mathrm{e}{-5}$ & $0.8$ & $0.01$ & $\beta_{\mathrm{STGC}}=0.25$; $\lambda_{\mathrm{load}}=1.1\mathrm{e}{-3}$; $\tau_{\mathrm{STGC}}=0$ \\
\addlinespace[2pt]
\multicolumn{5}{@{}l}{\textbf{[QNLI, BoolQ, RTE, PAWS, WiC, CoLA, SST-2, CB]} \textnormal{(8 tasks)}} \\*
Baseline  & $1.1\mathrm{e}{-4}$ & $1.0$ & $0.01$ & -- \\*
CAGrad    & $7.0\mathrm{e}{-5}$ & $1.0$ & $0.01$ & $c=0.30$; inner lr $=0.1$ \\*
GAR      & $7.0\mathrm{e}{-5}$ & $1.0$ & $0.01$ & $\lambda_{\mathrm{align}}=1.1\mathrm{e}{-3}$; load norm. \\*
STGC      & $9.2\mathrm{e}{-5}$ & $1.0$ & $0.01$ & $\beta_{\mathrm{STGC}}=0.25$; $\tau_{\mathrm{STGC}}=0$ \\*
LoadPen   & $7.0\mathrm{e}{-5}$ & $1.0$ & $0.01$ & $\lambda_{\mathrm{load}}=1.2\mathrm{e}{-3}$ \\*
SwitchAux & $7.0\mathrm{e}{-5}$ & $1.0$ & $0.01$ & $\alpha_{\mathrm{switch}}=1\mathrm{e}{-3}$ \\*
STGC+Load & $9\mathrm{e}{-5}$ & $1.0$ & $0.01$ & $\beta_{\mathrm{STGC}}=0.25$; $\lambda_{\mathrm{load}}=1.2\mathrm{e}{-3}$; $\tau_{\mathrm{STGC}}=0$ \\
\addlinespace[2pt]
\multicolumn{5}{@{}l}{\textbf{[QNLI, BoolQ, RTE, WiC, CoLA, SST-2, MRPC]} \textnormal{(7 tasks)}} \\*
Baseline  & $7.0\mathrm{e}{-5}$ & $1.1$ & $0.02$ & -- \\*
CAGrad    & $1.3\mathrm{e}{-4}$ & $1.1$ & $0.02$ & $c=0.28$; inner lr $=0.1$ \\*
GAR      & $7.0\mathrm{e}{-5}$ & $1.1$ & $0.02$ & $\lambda_{\mathrm{align}}=1\mathrm{e}{-4}$; load norm. \\*
STGC      & $7.0\mathrm{e}{-5}$ & $1.1$ & $0.02$ & $\beta_{\mathrm{STGC}}=1.2$; $\tau_{\mathrm{STGC}}=0$ \\*
LoadPen   & $4.6\mathrm{e}{-5}$ & $1.1$ & $0.02$ & $\lambda_{\mathrm{load}}=1.1\mathrm{e}{-4}$ \\*
SwitchAux & $7.0\mathrm{e}{-5}$ & $1.1$ & $0.02$ & $\alpha_{\mathrm{switch}}=1.2\mathrm{e}{-3}$ \\*
STGC+Load & $6.0\mathrm{e}{-5}$ & $1.1$ & $0.02$ & $\beta_{\mathrm{STGC}}=1.2$; $\lambda_{\mathrm{load}}=1\mathrm{e}{-2}$; $\tau_{\mathrm{STGC}}=0$ \\
\bottomrule
\end{tabular}
}%
\end{table}

\subsection{Trainable RoBERTa Classification-Head Top-1}
\label{appsubsec:top1_head_protocol}

This comparison keeps the trainable classification-head setting of
Appendix~\ref{appsubsec:implb_protocol} unchanged, namely the fully trainable
RoBERTa-base backbone, the rank-16 LoRA experts over the shared base head, the
shared output width, and the linear router on the masked-mean pooled
representation, and replaces the top-4 masked-softmax gate by the
straight-through one-hot gate of Appendix~\ref{appsubsec:top1_roberta_protocol}
applied to the per-example router logits (E8K1). Each example is therefore
dispatched to one expert in the forward computation, while the backward
derivative is the full-softmax Jacobian. GAR recomputes the same gate from
the detached pooled representation, averages it over the eight examples of
each same-task group, and pairs the group probability with the detached
expert-gradient observation as in Section~\ref{sec:methodology_implementation};
task losses are weighted as in Appendix~\ref{appsubsec:training_semantics}. This comparison includes
Baseline, CAGrad, and GAR.

The batch schedule, update budget, and selection protocol are those of the
top-4 classification-head comparison: per-task batch size 32 with same-task
groups of eight, 2,000 optimizer updates, five shared final seeds, and the
Baseline-first search of Appendix~\ref{appsubsec:cagrad_preselection}. Baseline
searches its learning rate, clipping, and weight decay; CAGrad and GAR inherit
the selected clipping and weight decay and search their own learning rate and
coefficient ($c$ and $\lambda_{\mathrm{align}}$, respectively), with CAGrad's
inner learning rate fixed at $0.1$. Table~\ref{tab:appendix_top1_head_hparams}
lists the selected learning rates, the shared clipping and weight decay, and
the selected coefficients.

\begin{table}[H]
\centering
\scriptsize
\caption{Final configurations for the trainable RoBERTa classification-head top-1 comparison. All five mixtures use E8K1 straight-through routing on the sequence-level router, rank-16 LoRA experts on the classification head ($\alpha=16$), BF16 forward/backward with FP32 router numerics, sequence length 256, per-task batch size 32, same-task groups of eight examples, warmup ratio 0.1, and 2,000 optimizer updates. Gradient clipping and weight decay are shared by the three methods within each mixture.}
\label{tab:appendix_top1_head_hparams}
\setlength{\tabcolsep}{5pt}
\renewcommand{\arraystretch}{1.03}
\begin{tabular}{@{}lrrrl@{}}
\toprule
Method & Learning rate & Grad.\ clip & Weight decay & Method-specific setting \\
\midrule
\multicolumn{5}{@{}l}{\textbf{[QNLI, BoolQ, RTE, PAWS, WiC]} \textnormal{(5 tasks)}} \\*
Baseline & $1.2\mathrm{e}{-4}$ & $1.1$ & $0.02$ & -- \\*
CAGrad & $1.0\mathrm{e}{-4}$ & $1.1$ & $0.02$ & $c=0.42$; inner lr $=0.1$ \\*
\textbf{GAR} & $1.0\mathrm{e}{-4}$ & $1.1$ & $0.02$ & $\lambda_{\mathrm{align}}=10^{-3}$; load norm. \\
\midrule
\multicolumn{5}{@{}l}{\textbf{[QNLI, BoolQ, RTE, PAWS, ANLI, CB]} \textnormal{(6 tasks)}} \\*
Baseline & $1.0\mathrm{e}{-4}$ & $0.8$ & $0.02$ & -- \\*
CAGrad & $1.2\mathrm{e}{-4}$ & $0.8$ & $0.02$ & $c=0.26$; inner lr $=0.1$ \\*
\textbf{GAR} & $9.0\mathrm{e}{-5}$ & $0.8$ & $0.02$ & $\lambda_{\mathrm{align}}=1.2\mathrm{e}{-3}$; load norm. \\
\midrule
\multicolumn{5}{@{}l}{\textbf{[QNLI, BoolQ, RTE, PAWS, WiC, CoLA, SST-2]} \textnormal{(7 tasks)}} \\*
Baseline & $7.0\mathrm{e}{-5}$ & $1.0$ & $0.01$ & -- \\*
CAGrad & $1.2\mathrm{e}{-4}$ & $1.0$ & $0.01$ & $c=0.26$; inner lr $=0.1$ \\*
\textbf{GAR} & $1.1\mathrm{e}{-4}$ & $1.0$ & $0.01$ & $\lambda_{\mathrm{align}}=10^{-3}$; load norm. \\
\midrule
\multicolumn{5}{@{}l}{\textbf{[QNLI, BoolQ, RTE, WiC, CoLA, SST-2, MRPC]} \textnormal{(7 tasks)}} \\*
Baseline & $9.0\mathrm{e}{-5}$ & $1.0$ & $0.01$ & -- \\*
CAGrad & $7.0\mathrm{e}{-5}$ & $1.0$ & $0.01$ & $c=0.35$; inner lr $=0.1$ \\*
\textbf{GAR} & $1.1\mathrm{e}{-4}$ & $1.0$ & $0.01$ & $\lambda_{\mathrm{align}}=10^{-4}$; load norm. \\
\midrule
\multicolumn{5}{@{}l}{\textbf{[QNLI, BoolQ, RTE, PAWS, WiC, CoLA, SST-2, CB]} \textnormal{(8 tasks)}} \\*
Baseline & $1.4\mathrm{e}{-4}$ & $0.8$ & $0.02$ & -- \\*
CAGrad & $1.1\mathrm{e}{-4}$ & $0.8$ & $0.02$ & $c=0.26$; inner lr $=0.1$ \\*
\textbf{GAR} & $1.4\mathrm{e}{-4}$ & $0.8$ & $0.02$ & $\lambda_{\mathrm{align}}=1.2\mathrm{e}{-2}$; load norm. \\
\bottomrule
\end{tabular}
\end{table}

\subsection{Trainable DeBERTa Full-Parameter FFN MoE}
\label{appsubsec:implc_protocol}

This setting retains the expert
insertion sites and token-level routing of the frozen LoRA-FFN comparison while changing
the expert parameterization and backbone trainability. The
\texttt{microsoft/deberta-v3-base} backbone is fully trainable, and each expert
is a complete feed-forward block at the final-layer feed-forward insertion
sites, without LoRA factorization, so the routed unit is a token rather than a
pooled sequence. All five mixtures use top-4 routing over eight experts
(E8K4) at sequence length 256, matching the frozen LoRA-FFN setting. This extension also retains a separate
prediction head with the appropriate output width for each task. All computation
is FP32. The alignment
pathway, the detached gradient observations, and the
$\mathcal{L}_{\mathrm{norm}}$ objective are those of
Section~\ref{sec:methodology_implementation}, with the same-task example groups
formed at the micro-batch sizes recorded in
Table~\ref{tab:appendix_implc_hparams}. This setting compares
Baseline, CAGrad, and GAR.

The three methods use the shared batch schedule and a 2,000-update budget.
Baseline first selects learning rate, clipping, and weight decay using
Appendix~\ref{appsubsec:cagrad_preselection}'s procedure. Subsequent methods
inherit the selected clipping and weight decay and tune their own learning
rates and applicable coefficients. Warmup is fixed at $0.1$.
The following table gives the configurations.

\begin{table}[H]
\centering
\scriptsize
\caption{Final trainable DeBERTa full-parameter FFN configurations for the five mixtures. All rows use \texttt{microsoft/deberta-v3-base} with a trainable backbone, full-parameter FFN experts at the final-layer feed-forward insertion sites, token-level top-$k$ routing, FP32 throughout, 2,000 optimizer updates, and the CAGrad-selected batch schedule of Appendix~\ref{appsubsec:cagrad_preselection}. The table reports the final values used at that locked schedule.}
\label{tab:appendix_implc_hparams}
\setlength{\tabcolsep}{3.5pt}
\renewcommand{\arraystretch}{1.02}
\AppendixTable{%
\begin{tabular}{@{}lrrrl@{}}
\toprule
Method & Learning rate & Grad.\ clip & Weight decay & Method-specific setting \\
\midrule
\multicolumn{5}{@{}l}{\textbf{[QNLI, BoolQ, RTE, PAWS, WiC]} \textnormal{(5 tasks; E8K4; batch $32\times5$, groups of 8; sequence length 256)}} \\*
Baseline & $1.3\mathrm{e}{-4}$ & $0.92$ & $0.02$ & -- \\*
CAGrad & $1.3\mathrm{e}{-4}$ & $0.92$ & $0.02$ & $c=0.22$; inner lr $=0.1$ \\*
\textbf{GAR} & $1.3\mathrm{e}{-4}$ & $0.92$ & $0.02$ & $\lambda_{\mathrm{align}}=1.0\mathrm{e}{-5}$; load norm. \\
\addlinespace[2pt]
\multicolumn{5}{@{}l}{\textbf{[QNLI, BoolQ, RTE, PAWS, ANLI, CB]} \textnormal{(6 tasks; E8K4; batch $32\times6$, groups of 8; sequence length 256)}} \\*
Baseline & $1.3\mathrm{e}{-4}$ & $0.86$ & $0.01$ & -- \\*
CAGrad & $1.3\mathrm{e}{-4}$ & $0.86$ & $0.01$ & $c=0.32$; inner lr $=0.1$ \\*
\textbf{GAR} & $1.3\mathrm{e}{-4}$ & $0.86$ & $0.01$ & $\lambda_{\mathrm{align}}=1.0\mathrm{e}{-3}$; load norm. \\
\addlinespace[2pt]
\multicolumn{5}{@{}l}{\textbf{[QNLI, BoolQ, RTE, PAWS, WiC, CoLA, SST-2]} \textnormal{(7 tasks; E8K4; batch $32\times7$, groups of 8; sequence length 256)}} \\*
Baseline & $6.9\mathrm{e}{-5}$ & $0.80$ & $0.01$ & -- \\*
CAGrad & $7.6\mathrm{e}{-5}$ & $0.80$ & $0.01$ & $c=0.28$; inner lr $=0.1$ \\*
\textbf{GAR} & $7.0\mathrm{e}{-5}$ & $0.80$ & $0.01$ & $\lambda_{\mathrm{align}}=1.0\mathrm{e}{-3}$; load norm. \\
\addlinespace[2pt]
\multicolumn{5}{@{}l}{\textbf{[QNLI, BoolQ, RTE, WiC, CoLA, SST-2, MRPC]} \textnormal{(7 tasks; E8K4; batch $32\times7$, groups of 8; sequence length 256)}} \\*
Baseline & $7.1\mathrm{e}{-5}$ & $0.80$ & $0.01$ & -- \\*
CAGrad & $1.3\mathrm{e}{-4}$ & $0.80$ & $0.01$ & $c=0.28$; inner lr $=0.1$ \\*
\textbf{GAR} & $7.0\mathrm{e}{-5}$ & $0.80$ & $0.01$ & $\lambda_{\mathrm{align}}=1.0\mathrm{e}{-3}$; load norm. \\
\addlinespace[2pt]
\multicolumn{5}{@{}l}{\textbf{[QNLI, BoolQ, RTE, PAWS, WiC, CoLA, SST-2, CB]} \textnormal{(8 tasks; E8K4; batch $32\times8$, groups of 8; sequence length 256)}} \\*
Baseline & $1.1\mathrm{e}{-4}$ & $0.80$ & $0.01$ & -- \\*
CAGrad & $7.0\mathrm{e}{-5}$ & $0.80$ & $0.01$ & $c=0.30$; inner lr $=0.1$ \\*
\textbf{GAR} & $7.0\mathrm{e}{-5}$ & $0.80$ & $0.01$ & $\lambda_{\mathrm{align}}=1.0\mathrm{e}{-3}$; load norm. \\
\bottomrule
\end{tabular}
}%
\end{table}

\subsection{DeBERTa Single-Task Configurations}

Both DeBERTa single-task adaptations use four experts with top-2 routing
(E4K2). The LoRA variant uses the frozen-backbone LoRA architecture described
in Appendix~\ref{appsubsec:cagrad_preselection}.
The FFN variant unfreezes the backbone and uses full-parameter FFN experts
at the same final-layer insertion sites, without a LoRA factorization.
The router, expert parameters, and classification head are trainable in both
variants; backbone parameters are also trainable in the FFN variant.
These architectural choices are shared by Baseline, CAGrad, and GAR.

Within each adaptation, the compared methods use the same experimental protocol.
For each task and adaptation, Baseline first selects learning rate, clipping,
and weight decay using Appendix~\ref{appsubsec:cagrad_preselection}'s procedure.
Subsequent methods inherit the selected clipping and weight decay and tune
their own learning rates and applicable coefficients. Warmup is fixed at $0.1$.
All runs use sequence length 128, effective batch size 16, and same-task
groups of eight examples. These sizes were selected by the same DeBERTa CAGrad
development-probe procedure, applied to MRPC, and are reused for all five tasks, both
adaptations, and all three methods; as in the multi-task settings, only the
batch sizes are transferred, not CAGrad's learning rate or coefficients. SST-2 and QQP use 3 epochs; CoLA, MRPC, and RTE use 5 epochs.
Tables~\ref{tab:appendix_hparams_4} and~\ref{tab:appendix_hparams_5}
list the LoRA and FFN configurations, respectively.

\raggedbottom
\subsubsection{Frozen-Backbone LoRA}

\begin{table}[H]
\centering
\scriptsize
\caption{Final configurations for DeBERTa single-task LoRA results (SST-2 and QQP use 3 epochs; CoLA, MRPC, and RTE use 5 epochs).}
\label{tab:appendix_hparams_4}
\begin{tabular}{@{}llrrrl@{}}
\toprule
Task & Method & Learning rate & Grad.\ clip & Weight decay & Method-specific setting \\
\midrule
CoLA & Baseline & $9.4\mathrm{e}{-6}$ & $1.1$ & $0$ & -- \\
 & CAGrad & $9.4\mathrm{e}{-6}$ & $1.1$ & $0$ & $c=0.31$; inner lr $=0.1$ \\
 & GAR & $9.4\mathrm{e}{-6}$ & $1.1$ & $0$ & $\lambda_{\mathrm{align}}=1.7\mathrm{e}{-4}$; load norm. \\
\midrule
MRPC & Baseline & $9.0\mathrm{e}{-5}$ & $1.0$ & $0.0044$ & -- \\
 & CAGrad & $7.8\mathrm{e}{-5}$ & $1.0$ & $0.0044$ & $c=0.91$; inner lr $=0.1$ \\
 & GAR & $5.5\mathrm{e}{-5}$ & $1.0$ & $0.0044$ & $\lambda_{\mathrm{align}}=1\mathrm{e}{-3}$; load norm. \\
\midrule
QQP & Baseline & $5.0\mathrm{e}{-5}$ & $1.0$ & $0.0044$ & -- \\
 & CAGrad & $7.8\mathrm{e}{-5}$ & $1.0$ & $0.0044$ & $c=0.91$; inner lr $=0.1$ \\
 & GAR & $1.4\mathrm{e}{-5}$ & $1.0$ & $0.0044$ & $\lambda_{\mathrm{align}}=1\mathrm{e}{-3}$; load norm. \\
\midrule
RTE & Baseline & $6.1\mathrm{e}{-5}$ & $1.9$ & $0.0024$ & -- \\
 & CAGrad & $4.5\mathrm{e}{-5}$ & $1.9$ & $0.0024$ & $c=0.48$; inner lr $=0.1$ \\
 & GAR & $5.5\mathrm{e}{-5}$ & $1.9$ & $0.0024$ & $\lambda_{\mathrm{align}}=1\mathrm{e}{-3}$; load norm. \\
\midrule
SST-2 & Baseline & $3.1\mathrm{e}{-5}$ & $1.4$ & $0.01$ & -- \\
 & CAGrad & $2.0\mathrm{e}{-5}$ & $1.4$ & $0.01$ & $c=0.69$; inner lr $=0.1$ \\
 & GAR & $2.2\mathrm{e}{-5}$ & $1.4$ & $0.01$ & $\lambda_{\mathrm{align}}=1\mathrm{e}{-3}$; load norm. \\
\bottomrule
\end{tabular}
\end{table}

\subsubsection{Trainable Full-Parameter FFN}

\begin{table}[H]
\centering
\scriptsize
\caption{Final configurations for DeBERTa single-task FFN results (SST-2 and QQP use 3 epochs; CoLA, MRPC, and RTE use 5 epochs).}
\label{tab:appendix_hparams_5}
\begin{tabular}{@{}llrrrl@{}}
\toprule
Task & Method & Learning rate & Grad.\ clip & Weight decay & Method-specific setting \\
\midrule
CoLA & Baseline & $2.0\mathrm{e}{-5}$ & $1.0$ & $0.017$ & -- \\
 & CAGrad & $2.0\mathrm{e}{-5}$ & $1.0$ & $0.017$ & $c=0.54$; inner lr $=0.1$ \\
 & GAR & $1.0\mathrm{e}{-5}$ & $1.0$ & $0.017$ & $\lambda_{\mathrm{align}}=2.8\mathrm{e}{-5}$; load norm. \\
\midrule
MRPC & Baseline & $1.1\mathrm{e}{-4}$ & $1.0$ & $0.011$ & -- \\
 & CAGrad & $7.8\mathrm{e}{-5}$ & $1.0$ & $0.011$ & $c=0.91$; inner lr $=0.1$ \\
 & GAR & $3.4\mathrm{e}{-5}$ & $1.0$ & $0.011$ & $\lambda_{\mathrm{align}}=7.6\mathrm{e}{-4}$; load norm. \\
\midrule
QQP & Baseline & $1.0\mathrm{e}{-5}$ & $1.0$ & $0.011$ & -- \\
 & CAGrad & $7.8\mathrm{e}{-5}$ & $1.0$ & $0.011$ & $c=0.91$; inner lr $=0.1$ \\
 & GAR & $1.7\mathrm{e}{-5}$ & $1.0$ & $0.011$ & $\lambda_{\mathrm{align}}=7.6\mathrm{e}{-4}$; load norm. \\
\midrule
RTE & Baseline & $2.6\mathrm{e}{-5}$ & $1.0$ & $0.046$ & -- \\
 & CAGrad & $4.0\mathrm{e}{-5}$ & $1.0$ & $0.046$ & $c=0.69$; inner lr $=0.1$ \\
 & GAR & $4.2\mathrm{e}{-5}$ & $1.0$ & $0.046$ & $\lambda_{\mathrm{align}}=2.4\mathrm{e}{-5}$; load norm. \\
\midrule
SST-2 & Baseline & $1.3\mathrm{e}{-5}$ & $1.4$ & $0.01$ & -- \\
 & CAGrad & $3.5\mathrm{e}{-5}$ & $1.4$ & $0.01$ & $c=0.48$; inner lr $=0.1$ \\
 & GAR & $1.3\mathrm{e}{-5}$ & $1.4$ & $0.01$ & $\lambda_{\mathrm{align}}=4\mathrm{e}{-3}$; load norm. \\
\bottomrule
\end{tabular}
\end{table}

\subsection{RoBERTa Fixed-Configuration Coefficient Ablations}
\label{appsubsec:lambda_ablation_protocol}

We report separate RoBERTa LoRA-FFN and classification-head LoRA coefficient
sweeps on the same five E8K4 dataset mixtures listed in Tables~\ref{tab:roberta_lambda_configs} and~\ref{tab:roberta_classifier_lambda_configs}. For each setting and mixture, we reuse the corresponding
selected GAR configuration, including learning rate, gradient clipping,
weight decay, batch and micro-batch sizes, sequence length, warmup ratio,
expert topology, backbone trainability, and numerical precision. The training
budget is reduced from 2,000 to 1,000 updates for every coefficient, with the
same five random seeds shared across coefficients. Within each sweep, only
$\lambda\in\{0,10^{-5},10^{-4},10^{-3},10^{-2}\}$ varies; no additional HPO
or per-coefficient configuration selection is performed. Both settings use
the common-template gradient observations defined in Section~\ref{Sec:Euc}.
At $\lambda=0$, the auxiliary backward pass is disabled. This task-loss-only
control retains the matched GAR configuration and is distinct from the
separately tuned Baseline in the main comparison.
The ablations use a common coefficient grid across mixtures. The coefficients
selected for the main comparisons appear in the configuration tables; for
[QNLI, BoolQ, RTE, PAWS, WiC, CoLA, SST-2, CB], the selected LoRA-FFN value,
$1.2\times10^{-2}$, lies above the grid maximum of $10^{-2}$.

\paragraph{Frozen-backbone LoRA-FFN.}
The experts are rank-16 LoRA adapters in the final FFN layer, with token-level
E8K4 routing, a frozen RoBERTa-base backbone, and FP32 computation.
Table~\ref{tab:roberta_lambda_configs} lists the five fixed configurations.

\begin{table}[H]
\centering
\small
\setlength{\tabcolsep}{3pt}
\caption{Fixed RoBERTa LoRA-FFN configurations for the five-mixture E8K4 coefficient sweep. Each mixture reuses its selected GAR configuration, with 1,000 updates and $\lambda\in\{0,10^{-5},10^{-4},10^{-3},10^{-2}\}$. All rows use a frozen backbone, final-layer LoRA-FFN experts, token-level routing, and FP32 computation, rank 16, scaling 16, load normalization ($\varepsilon=10^{-8}$), and warmup ratio 0.1. LR denotes learning rate, Clip the gradient-clipping threshold, WD weight decay, and Seq. the maximum sequence length. Brackets list constituent tasks; batch/group is per task.}
\label{tab:roberta_lambda_configs}
\begin{tabular}{@{}p{0.40\textwidth}rrrrr@{}}
\toprule
Mixture & LR & Clip & WD & Batch/group & Seq. \\
\midrule
{}[QNLI, BoolQ, RTE, PAWS, WiC] & $1.0\mathrm{e}{-3}$ & 1.1 & 0.010 & 32/8 & 256 \\
{}[QNLI, BoolQ, RTE, PAWS, ANLI, CB] & $1.0\mathrm{e}{-3}$ & 0.8 & 0.010 & 32/8 & 256 \\
{}[QNLI, BoolQ, RTE, PAWS, WiC, CoLA, \mbox{SST-2}] & $2.0\mathrm{e}{-3}$ & 0.8 & 0.010 & 32/8 & 256 \\
{}[QNLI, BoolQ, RTE, WiC, CoLA, \mbox{SST-2}, MRPC] & $1.0\mathrm{e}{-3}$ & 0.8 & 0.010 & 32/8 & 256 \\
{}[QNLI, BoolQ, RTE, PAWS, WiC, CoLA, \mbox{SST-2}, CB] & $1.0\mathrm{e}{-3}$ & 0.8 & 0.010 & 32/8 & 256 \\
\bottomrule
\end{tabular}
\end{table}

\paragraph{Trainable classification-head LoRA.}
The trainable backbone, classification-head LoRA experts, and sequence-level
E8K4 router follow Appendix~\ref{appsubsec:implb_protocol}. Computation uses
BF16 with FP32 router numerics. Table~\ref{tab:roberta_classifier_lambda_configs}
lists the five fixed configurations. Appendix~\ref{appsubsec:lambda_ablation_results}
reports mixture-level and per-task results for both settings without pooling
their accuracies.
The classification-head denominator ablation reuses the $\lambda=0$ and load-normalized
$\lambda=10^{-3}$ runs from this sweep. For the numerator-only objective,
each mixture uses the corresponding $\lambda=10^{-3}$ configuration and
replaces $-\sum_k\lVert G_k\rVert^2/(d_k(P)+\epsilon)$ with
$-\sum_k\lVert G_k\rVert^2$. No additional HPO or per-arm configuration
selection is performed.

\begin{table}[H]
\centering
\small
\setlength{\tabcolsep}{3pt}
\caption{Fixed RoBERTa classification-head LoRA configurations for the five-mixture E8K4 coefficient sweep. Each mixture reuses its selected GAR configuration, with 1,000 updates and $\lambda\in\{0,10^{-5},10^{-4},10^{-3},10^{-2}\}$. All rows use a trainable backbone, classification-head LoRA experts, sequence-level routing, and BF16 computation with FP32 router numerics, rank 16, scaling 16, load normalization ($\varepsilon=10^{-8}$), and warmup ratio 0.1. LR denotes learning rate, Clip the gradient-clipping threshold, WD weight decay, and Seq. the maximum sequence length. Brackets list constituent tasks; batch/group is per task.}
\label{tab:roberta_classifier_lambda_configs}
\begin{tabular}{@{}p{0.40\textwidth}rrrrr@{}}
\toprule
Mixture & LR & Clip & WD & Batch/group & Seq. \\
\midrule
{}[QNLI, BoolQ, RTE, PAWS, WiC] & $1.0\mathrm{e}{-4}$ & 1.0 & 0.032 & 32/8 & 256 \\
{}[QNLI, BoolQ, RTE, PAWS, ANLI, CB] & $1.2\mathrm{e}{-4}$ & 0.8 & 0.010 & 32/8 & 256 \\
{}[QNLI, BoolQ, RTE, PAWS, WiC, CoLA, \mbox{SST-2}] & $7.0\mathrm{e}{-5}$ & 0.8 & 0.010 & 32/8 & 256 \\
{}[QNLI, BoolQ, RTE, WiC, CoLA, \mbox{SST-2}, MRPC] & $7.0\mathrm{e}{-5}$ & 1.1 & 0.020 & 32/8 & 256 \\
{}[QNLI, BoolQ, RTE, PAWS, WiC, CoLA, \mbox{SST-2}, CB] & $7.0\mathrm{e}{-5}$ & 1.0 & 0.010 & 32/8 & 256 \\
\bottomrule
\end{tabular}
\end{table}

\subsection{Update Semantics}
\label{appsubsec:training_semantics}

In each multi-task update, we draw one batch of up to the stated effective
per-task size from each task loader. Within each experimental setting and mixture,
the compared methods share the task set, batch schedule, routed MoE
forward/dispatch semantics, and stated update budget.

For \textbf{Baseline}, the task loss averages the per-task mean losses with
equal task weight. Within task $t$, a group $m$ contributes weight
$|m|/(T B_t)$, where $B_t$ is the actual task-batch size and $T$ is the
number of tasks in the update. Thus a shorter final loader batch does not
reduce that task's total weight. GAR and the auxiliary-loss controls use the
same task-loss weights in every setting. For \textbf{CAGrad}, each input row to the gradient-combination rule is
the gradient of the mean loss within one same-task micro-batch. In multi-task
settings, all such rows from the constituent task batches are combined, with
scaling that preserves equal task weight and sample weighting within each task.
In single-task settings, the rows come from distinct micro-batches of the sole task.
CAGrad therefore still combines multiple gradient observations in a single-task
update. These rows test within-task micro-batch gradient combination.

Let $\theta_{\mathrm C}$ concatenate all trainable model parameters in a
fixed order, including the router, experts, classification head, and any
unfrozen backbone parameters. CAGrad combines the stated micro-batch loss
gradients with respect to $\theta_{\mathrm C}$ and assigns the resulting
vector to that same parameter set before global gradient clipping and the
AdamW step. Frozen parameters are excluded. The inner solver takes 10
projected-simplex gradient steps of size $0.1/\operatorname{tr}(GG^{\top})$,
where the rows of $G$ are the observations, to approximately minimize
$\langle g_w,g_0\rangle+c\|g_0\|\|g_w\|$, where $g_0$ is the mean
observation and $g_w$ is their simplex-weighted combination. It returns
$g_0+c\|g_0\|g_w/(\|g_w\|+\epsilon)$ without further rescaling.
The coefficient $c$ follows the selected configuration tables. The inner
learning rate $0.1$ is fixed; dividing it by $\operatorname{tr}(GG^{\top})$
makes the solver step invariant to the gradient scale.

Table~\ref{tab:comparator_definitions} summarizes the scalar auxiliary
objectives and their gradient paths.
\paragraph{Implemented control losses.}
Let $w_m=|m|/\sum_j|j|$ be the sample weight of micro-batch $m$, and
$v_m=|m|/(TB_t)$ its task-loss weight when $m$ comes from task $t$
(see above); the two coincide when every task batch is full.
For LoadPen, define the update-level example summaries
\[
P_e=\sum_m w_m\frac{1}{|m|}\sum_{n\in m}q_{n,e},\qquad
f_e^{(1)}=\sum_m w_m\frac{1}{|m|}\sum_{n\in m}
\mathbf{1}[e=\arg\max_j q_{n,j}].
\]
For SwitchAux, let $V_m$ contain the valid routed units in micro-batch $m$:
non-padding tokens for FFN routing and examples for classification-head routing.
Here $r_i$ is the configured routed-unit probability vector and $S_i$
is the top-4 expert set. With these probabilities and selection sets,
\[
P_{m,e}=\frac{1}{|V_m|}\sum_{i\in V_m}r_{i,e},\qquad
f_{m,e}^{(k)}=\frac{1}{k|V_m|}\sum_{i\in V_m}\mathbf{1}[e\in S_i].
\]
Both controls use the frequency--probability product form of auxiliary
load balancing \citep{switch_transformer}. Both hard-frequency vectors sum
to one and are detached. LoadPen forms its
product after aggregating examples across the update. SwitchAux averages
micro-batch products with weights $u_m$: $u_m=w_m$ for frozen FFN experts, and
$u_m=v_m$ for the trainable head, where each product is added to its
micro-batch task loss. Neither objective is the diagnostic LVar.

For adapted STGC, let $z_i$ be the router logits and $C_{m,i,e}$ indicate a
valid selected assignment whose conflict score is below the fixed threshold
$\tau_{\mathrm{STGC}}=0$. The score averages the cosine similarities of the
two LoRA virtual-bias gradient blocks to their respective means over units
assigned to expert $e$ in that micro-batch. For an expert delta
$sB_eA_e x$, these detached proxy blocks are $sB_e^\top\delta_{i,e}$ and
$s\delta_{i,e}$, where $\delta_{i,e}$ is the task-loss derivative at the expert
output. They detect conflict without introducing trainable biases.
Writing $N_m^C=\sum_{i,e}C_{m,i,e}$, the adapted conflict loss is
\[
\mathcal C_m=-\frac{1}{E\max(1,N_m^C)}
\sum_{i,e}C_{m,i,e}\log[\mathrm{softmax}(-z_i)]_e .
\]
It is zero when no selected assignment conflicts. The mask is detached;
the loss differentiates through the router logits. Each $\mathcal C_m$ is
added to its micro-batch task loss and therefore carries weight $v_m$.

\begin{table}[H]
\centering\footnotesize
\caption{Implemented scalar control objectives and their auxiliary gradient paths.
Task-loss gradients follow the ordinary training path in every row.
CAGrad is a gradient-combination control, described separately in the update semantics.}
\label{tab:comparator_definitions}
\setlength{\tabcolsep}{4pt}
\renewcommand{\arraystretch}{1.12}
\begin{tabular}{@{}lp{0.48\linewidth}p{0.31\linewidth}@{}}
\toprule
Method & Objective & Auxiliary gradient path \\
\midrule
Baseline & $\mathcal L_{\mathrm{task}}$ & None. \\
STGC & $\mathcal L_{\mathrm{task}}+\beta_{\mathrm{STGC}}\sum_mv_m\mathcal C_m$ & Through logits; conflict masks detached. \\
STGC+Load & $\mathcal L_{\mathrm{task}}+\beta_{\mathrm{STGC}}\sum_mv_m\mathcal C_m+\lambda_{\mathrm{load}}E\sum_e\sg(f_e^{(1)})P_e$ & Through logits; conflict masks and hard frequencies detached. \\
LoadPen & $\mathcal L_{\mathrm{task}}+\lambda_{\mathrm{load}}E\sum_e\sg(f_e^{(1)})P_e$ & Directly through the router; routing inputs detached. \\
SwitchAux & $\mathcal L_{\mathrm{task}}+\alpha_{\mathrm{switch}}\sum_mu_m E\sum_e\sg(f_{m,e}^{(k)})P_{m,e}$ & Directly through the router for frozen FFN; ordinary routing path for the trainable head. \\
\bottomrule
\end{tabular}
\end{table}
All resulting gradients undergo one global-norm clip before the
optimizer step.
The trainable-head STGC and SwitchAux paths can also update the backbone
through the routing inputs. The reported STGC objective consists only of task
loss and the conflict term; LoadPen and SwitchAux are separate controls.
STGC+Load combines the STGC conflict term with the LoadPen term. Its model-training learning rate, $\beta_{\mathrm{STGC}}$, and $\lambda_{\mathrm{load}}$ are selected independently for every setting and mixture under the common HPO procedure; the selected values are reported in the grouped configuration tables. Its endpoints are reported in the corresponding grouped result tables.

Each task batch is traversed in same-task example groups, called
micro-batches here. Their size follows the shared schedule in
Appendix~\ref{appsubsec:cagrad_preselection}; a final group may be shorter.
These groups fix the granularity at which gradient observations are read out;
task-loss weights follow the equal-task rule above. The auxiliary objective depends
on that granularity, which is held fixed within each method comparison.
All five mixtures use effective per-task batch size $B=32$ and groups of
eight examples, giving four groups per task per full-batch update; single-task $B=16$ runs use two groups of eight. The selected group size is excluded from subsequent method-specific
coordinate sweeps and local confirmation.

For \textbf{GAR}, $m$ indexes one such group and $\ell_m$ is the ordinary loss
averaged over that group. The gradient observation $g_m$ is formed as in
Section~\ref{Sec:Euc}: the group-loss gradients with respect to the selected trainable expert
parameters are summed entrywise across experts using matching parameter names,
shapes, and local order before inner products are computed. The template
contains one expert's LoRA parameters in the LoRA settings and one expert's
full FFN parameters in the full-parameter FFN settings.
Each $g_m$ is obtained within the same training batch and from the same task
loss as Baseline. Auxiliary router probabilities are recomputed from detached
routing inputs. Let $q_n$ denote the resulting probability vector for example
$n$; in token-routed FFN settings, $q_n$ first averages the configured token
probabilities over that example's non-padding tokens. In sequence-routed
classification-head settings, $q_n$ is the configured per-example router
probability vector. The auxiliary branch then sets
\[
p_m=\frac{1}{|m|}\sum_{n\in m}q_n,
\]
with no additional top-$k$ operation after group averaging. This construction
lies on the assignment simplex of Section~\ref{sec:variational_formulation},
although the group mean need not be $k$-sparse. The task loss uses the same
equal-task weights as the baselines, while $\mathcal{L}_{\mathrm{norm}}$ is
computed from the paired gradient observations and router summaries. The
alignment loss is backpropagated only through the router branch; task
parameters remain driven by the task loss.

\subsection{Validation Sets and Runtime Metadata}
\label{appsubsec:validation_sets_and_runtime_metadata}

Supervised accuracy results use each task's complete benchmark-provided
labeled validation or development split. Final runs use the complete
corresponding training split. The five mixtures draw from GLUE
\citep{glue_wang2019}, SuperGLUE \citep{superglue_wang2019},
PAWS \citep{paws_zhang2019}, and ANLI \citep{anli_nie2020}; ANLI uses rounds
R1--R3 for both training and validation.
Setting-specific selection details accompany the configuration tables above.
The exact task composition of each mixture matches the sets listed in
Section~\ref{sec:empirical_summary} and Appendix~\ref{app:full_empirical_results}.

RoBERTa and DeBERTa use FP32, while Qwen3-1.7B and Qwen3-8B use BF16 for throughput and
memory efficiency; the trainable RoBERTa classification-head setting uses BF16
forward and backward passes with FP32 router numerics
(Appendix~\ref{appsubsec:implb_protocol}). Within each backbone and experimental setting, all
compared methods use the same numerical precision, so no reported comparison
mixes precisions across methods. The main multi-task experiments use one
NVIDIA RTX PRO 6000 Blackwell Server Edition GPU per run with PyTorch
2.8.0+cu128 (CUDA 12.8), without distributed expert-parallel execution.
Each run records total wall-clock time and mean training-step
time.

Metric definitions and aggregation rules are centralized in
Appendix~\ref{app:experimental_protocol}.

\subsubsection{Existing Assets, Code Release, and Societal-Risk Scope}
\label{appsubsec:assets_and_impact}

All datasets and pretrained backbones used in the experiments are existing
public assets accessed through their standard benchmark or model-provider
interfaces. Code to be released at \url{https://github.com/lyclyq/MoE_arxiv} will include a reference
implementation of the objective, the training and diagnostic computation code, and
architecture presets for the reported settings; it will not redistribute raw benchmark
data or pretrained model weights. Users of the code will therefore need to
obtain the underlying assets from their original providers and comply with the
corresponding licenses, model cards, and terms of use.
Table~\ref{tab:asset_license_scope} lists the asset groups and their scope.

\begin{table}[H]
\centering
\scriptsize
\caption{Existing assets used by the experiments. We cite the original creators, describe the role of each asset, and avoid redistributing raw datasets or pretrained model weights with the code release.}
\label{tab:asset_license_scope}
\AppendixTable{%
\begin{tabular}{p{0.19\textwidth}p{0.23\textwidth}p{0.23\textwidth}p{0.27\textwidth}}
\toprule
Asset group & Role in experiments & Source/citation & License / terms identifier \\
\midrule
GLUE tasks (CoLA, MRPC, QQP, RTE, QNLI, SST-2) & Classification/paraphrase/NLI mixture components and single-task checks & GLUE benchmark \citep{glue_wang2019} & Upstream benchmark terms and original task dataset licenses; not redistributed. \\
SuperGLUE tasks (BoolQ, WiC, CB) & Five-to-eight-task mixtures & SuperGLUE benchmark \citep{superglue_wang2019} & Upstream benchmark terms and original task dataset licenses; not redistributed. \\
PAWS and ANLI & Paraphrase and adversarial-NLI components of larger mixtures & PAWS \citep{paws_zhang2019}; ANLI \citep{anli_nie2020} & License/terms specified by the upstream dataset cards or providers; not redistributed. \\
Qwen3, RoBERTa, DeBERTa & Pretrained backbones for sparse MoE adaptation & Qwen3 \citep{qwen3_yang2025}; RoBERTa \citep{roberta_liu2019}; DeBERTa \citep{deberta_he2021} & Model-card licenses and provider terms for the corresponding pretrained weights; not redistributed. \\
PyTorch, Transformers, Datasets & Training and data-loading software dependencies & Public open-source packages listed with the code release & Upstream open-source package licenses and versions documented with the code release. \\
\bottomrule
\end{tabular}
}
\end{table}

The societal-impact scope is indirect. This work is methodological and does not
introduce new datasets, user-facing systems, or new generative model
capabilities. A potential positive impact is better use of expert capacity in multi-task
sparse models. Potential risks are
deployment-mediated: better routing and training efficiency could lower the cost
of multi-task models in sensitive applications, where fairness, privacy, and
safety evaluations remain necessary before deployment.

\subsection{Runtime Overhead: DeBERTa Measurements}
\label{appsubsec:runtime_overhead_deberta_single}

The runtime measurements are single-GPU FP32 DeBERTa runs: single-task
MRPC runs for the LoRA and FFN adaptations, and a frozen LoRA-FFN five-task
mixture. Relative to Baseline, GAR adds an alignment branch whose auxiliary
gradient is directed to the router and that,
within one update, holds the detached gradient observations and the detached
routing inputs (hidden states and padding masks) needed to recompute the
differentiable router summaries of Algorithm~\ref{alg:training}; these buffers
are released after the optimizer step. Let $m$ denote the number of gradient
observations in an update, $E$ the number of experts, and $P_{\mathrm{tr}}$ the
number of selected trainable parameters in one expert. The additional
bookkeeping is $O(mEP_{\mathrm{tr}}+mE)$ arithmetic plus the route-summary
recomputation; it introduces no multi-GPU communication or expert-parallel
state. The
gradient-buffer term is smaller for LoRA because $P_{\mathrm{tr}}$ is the LoRA
adapter parameter count. Table~\ref{tab:runtime_overhead_deberta_single}
reports the step and wall-clock measurements, which also include
route-summary recomputation and host overhead. The MRPC measurements use
two gradient observations per update; the five-task mixture uses 20 (five
tasks with four same-task groups each), where GAR takes $1.04\times$
Baseline's mean step time over five seeds.

\begin{table}[H]
\centering
\scriptsize
\caption{Runtime summary for DeBERTa runs, Baseline versus GAR. Single-task
MRPC rows: the frozen-backbone LoRA rows report a one-epoch local
probe with top-$k$ gating, and the non-frozen FFN rows report the corresponding
five-seed summaries; both use effective batch size 16 and two gradient
observations per update. Five-task rows: frozen-backbone LoRA-FFN runs on
[QNLI, BoolQ, RTE, PAWS, WiC] with E8K4 routing, per-task batch size 32, and
same-task groups of eight (20 gradient observations per update), over 500
updates and five seeds; step time excludes validation and model loading, and
wall time is the corresponding 500-update training time. All rows use the same
data-loading and scheduling path as the final runs.
Relative factors compare mean step time with the corresponding baseline; for the
five-seed rows, step and wall time are reported as mean $\pm$ population
standard deviation.}
\label{tab:runtime_overhead_deberta_single}
\begin{tabular}{llccc}
\toprule
Setting & Method & Step (s) & Wall (s) & Factor \\
\midrule
\multicolumn{5}{@{}l}{\emph{Single-task MRPC}} \\
Frozen LoRA & Baseline & 0.0528 & 13.27 & 1.00$\times$ \\
Frozen LoRA & GAR & 0.0593 & 14.77 & 1.12$\times$ \\
Trainable FFN & Baseline & $0.319{\pm}0.003$ & $409.6{\pm}4.5$ & 1.00$\times$ \\
Trainable FFN & GAR & $0.329{\pm}0.002$ & $422.5{\pm}2.0$ & 1.03$\times$ \\
\midrule
\multicolumn{5}{@{}l}{\emph{Five-task mixture [QNLI, BoolQ, RTE, PAWS, WiC]}} \\
Frozen LoRA-FFN & Baseline & $0.561{\pm}0.012$ & $280.7{\pm}5.9$ & 1.00$\times$ \\
Frozen LoRA-FFN & GAR & $0.584{\pm}0.005$ & $292.0{\pm}2.3$ & 1.04$\times$ \\
\bottomrule
\end{tabular}
\end{table}

\section{Metric Definitions and Aggregation}
\label{app:experimental_protocol}

All endpoint tables report final-checkpoint summaries. Let $S$ be the
number of final seeds, $T$ the number of tasks in a mixture, and $E$ the number
of experts, denoted $K$ in Sections~\ref{sec:variational_formulation}--\ref{sec:methodology_implementation};
a label such as E8K4 gives $E$ and the number of experts selected per routed unit. For seed $s$, task accuracy is $a_t^{(s)}=c_t^{(s)}/N_t$, where $c_t^{(s)}$
counts correctly classified validation examples and $N_t$ is the size of the
task's labeled evaluation split. Each input sequence or sequence pair
contributes one prediction; accuracy is not averaged over tokens or batches.
Define the equal-task macro average
\[
a^{(s)}=\frac{1}{T}\sum_{t=1}^{T}a_t^{(s)}.
\]
Thus, each task contributes the same weight regardless of its validation-set
size. Let
$\boldsymbol{\ell}^{(s)} = (\ell_1^{(s)}, \ldots, \ell_E^{(s)})$ denote the
final expert-load vector recorded from the router, with
$\sum_{e=1}^{E} \ell_e^{(s)} = 1$.

\paragraph{Observation units across settings.}
The routed unit and the statistical aggregation are distinct. FFN experts
route non-padding tokens; classification-head experts route pooled examples.
Accuracy uses example-level predictions in both settings. Association metrics
use routed-unit selection events, whereas gradient purity and cosine metrics
use same-task micro-batch gradients averaged within each task. Training-time
group observations are defined in Section~\ref{sec:methodology_implementation};
the checkpoint gradient probes are defined below. These definitions apply to
every method within a setting.

\paragraph{Construction of the load vector.}
The FFN endpoint tables use router probability mass. With $q_{n,e}^{(s)}$
the per-example router summary defined in
Section~\ref{sec:methodology_implementation},
\[
\ell_e^{(s)}=\frac{1}{N_s}\sum_{n=1}^{N_s}q_{n,e}^{(s)}.
\]
Here $N_s$ counts validation examples pooled across the task loaders. For
token-routed FFN experts, each $q_n$ first averages the configured post-top-$k$
probabilities over that example's non-padding tokens. In the frozen top-1 extension,
these token gates are one-hot in the forward computation, so the same
averaging rule records hard expert-selection mass. Thus examples receive
equal load weight, while tasks contribute according to their validation-set
sizes; this differs from equal-task macro accuracy. No top-$k$ truncation is
applied to the averaged load vector.
The trainable classification-head diagnostics instead use normalized expert
selection counts from the contingency table defined below:
\[
\ell_e^{(s)}=
\frac{\sum_t C_{t,e}^{(s)}}{\sum_{t,j}C_{t,j}^{(s)}}.
\]
These counts include each of the $k$ selected experts once per example. LVar and Util
always use the same load vector within a setting. The FFN and classification-
head routing summaries are reported separately; their absolute load values
represent probability mass and selection frequency, respectively.

\paragraph{Accuracy and seed standard deviation.}
The reported final accuracy is
\[
\mathrm{Acc} = \frac{1}{S}\sum_{s=1}^{S} a^{(s)},
\]
and the accompanying seed standard deviation is the population standard
deviation over the completed final-evaluation seeds:
\[
\mathrm{Std} = \sqrt{\frac{1}{S}\sum_{s=1}^{S} \bigl(a^{(s)} - \mathrm{Acc}\bigr)^2}.
\]
We use this equal-task macro accuracy on the complete benchmark-provided
labeled evaluation split for reported cross-task aggregation. Task-standard
metrics, CoLA MCC and MRPC/QQP F1, are reported with the single-task results.

\paragraph{Load variance.}
Within each seed, the final expert-load variance is computed from the final
load vector as
\[
\mathrm{LVar}^{(s)} = \frac{1}{E}\sum_{e=1}^{E}\Bigl(\ell_e^{(s)} - \frac{1}{E}\Bigr)^2,
\]
which is exactly the population variance of the final expert loads. The tabled
load-variance summary is then averaged across seeds:
\[
\mathrm{LVar} = \frac{1}{S}\sum_{s=1}^{S}\mathrm{LVar}^{(s)}.
\]
This population-variance definition is used for all reported load-variance values in the
paper. The multi-task experiments use $E=8$; for the single-task E4K2 controls, $E=4$ and
$0 \leq \mathrm{LVar} \leq (E-1)/E^2 = 0.1875$; for $E=8$ the bound is
$7/64\approx0.109$.

\paragraph{Scope of load-balance metrics.}
Load variance characterizes routing organization in the matched single-GPU
experiments. Together with gradient-mass expert purity and utilization, it
describes expert usage alongside validation performance. Measured wall-clock
costs are reported separately in
Appendix~\ref{appsubsec:runtime_overhead_deberta_single}.

\paragraph{Expert utilization.}
An expert is counted as \emph{utilized} in seed $s$ if its final load is at least
half of the uniform-load baseline, that is,
\[
\ell_e^{(s)} \ge \frac{1}{2E}.
\]
The per-seed utilization ratio is therefore
\[
\mathrm{Util}^{(s)} = \frac{1}{E}\sum_{e=1}^{E}
\mathbf{1}\!\left[\ell_e^{(s)} \ge \frac{1}{2E}\right],
\]
and the reported utilization is
\[
\mathrm{Util} = \frac{1}{S}\sum_{s=1}^{S}\mathrm{Util}^{(s)}.
\]

\paragraph{Gradient diagnostic observations.}
\label{app:gradient_probe_definition}
At a checkpoint, the diagnostic implementation runs in evaluation mode and
processes leading batches within the active partition of each unshuffled validation loader. Each loader
batch is split into same-task micro-batches using the training group size.
Let $B_{t,r}^{(s)}$, $r=1,\ldots,R_t^{(s)}$, denote the micro-batches actually
processed for task $t$ and seed $s$. For FFN routing, let $V_n$ be the
non-padding token positions of example $n$, $\ell_n^{(s)}$ its supervised
sequence-classification loss, and $u_{n,j,e}^{(s)}$ the output of expert $e$
at token $j$. Define that token's contribution to the expert-parameter gradient:
\[
v_{n,j,e}^{(s)}=
\left(\frac{\partial u_{n,j,e}^{(s)}}{\partial\theta_e}\right)^{\!\top}
\nabla_{u_{n,j,e}^{(s)}}\ell_n^{(s)}.
\]
Although the sequence loss can depend on all tokens, the chain rule gives
$\nabla_{\theta_e}\ell_n^{(s)}=\sum_{j\in V_n}v_{n,j,e}^{(s)}$.
The detached FFN diagnostic observation is therefore
\[
g_{t,e}^{(s)}=\operatorname{stopgrad}\!\left[
\frac{1}{R_t^{(s)}}\sum_{r=1}^{R_t^{(s)}}
\frac{1}{|B_{t,r}^{(s)}|}\sum_{n\in B_{t,r}^{(s)}}
\sum_{j\in V_n}v_{n,j,e}^{(s)}\right].
\]
For classification-head experts, the routed unit is the pooled example, so
the innermost token sum is replaced by $\nabla_{\theta_e}\ell_n^{(s)}$.
The implementation obtains each micro-batch gradient by differentiating its
mean example loss; linearity makes this identical to the token-contribution
formula for FFN routing. Expert parameters follow a fixed local ordering.
Micro-batches receive equal weight. Token contributions are summed as vectors
before task-gradient norms or cosines are computed; no separate token-level
norm or token-level classification loss is used.
The stopping condition is checked between complete loader batches, so
$R_t^{(s)}$ counts the processed micro-batches rather than the configured
stopping threshold. This gradient pass is separate from forward evaluation.

\paragraph{Gradient-mass expert purity.}
Using the final-checkpoint observations $g_{t,e}^{(s)}$, define the normalized
task-mass distribution for an expert whose total diagnostic gradient mass
exceeds the numerical floor $10^{-12}$ as
\[
p_{t,e}^{(s)} =
\frac{\lVert g_{t,e}^{(s)} \rVert_2}
{\sum_{t'} \lVert g_{t',e}^{(s)} \rVert_2}.
\]
For experts at or below this numerical floor, the distribution above is not
formed; we set the expert purity contribution to $0$ and include that expert in
the average.
The purity of expert $e$ in seed $s$ is the dominant task share
\[
\mathrm{Purity}_{e}^{(s)} = \max_t p_{t,e}^{(s)}.
\]
We first average across experts within each seed,
\[
\mathrm{Purity}^{(s)} = \frac{1}{E}\sum_{e=1}^{E}\mathrm{Purity}_{e}^{(s)},
\]
and then average across seeds:
\[
\mathrm{Pur.} = \frac{1}{S}\sum_{s=1}^{S}\mathrm{Purity}^{(s)}.
\]
This diagnostic measures per-expert dominance in gradient mass, including
task-dependent gradient scales. We report it jointly with validation accuracy,
load variance, and utilization.

\paragraph{Task--expert NMI and ARI for supervised finetuning.}
At the final forward evaluation for seed $s$, let $C^{(s)}\in\mathbb{N}^{T\times E}$
be the contingency table of task labels and expert selections. For FFN
routing, each non-padding token contributes one count for each of its
top-$k$ selected experts, inheriting the task label of its input example.
For classification-head routing, each example contributes those counts once.
Counts are collected before token-to-example or example-to-group probability
averaging. Consequently, FFN association counts weight examples by their
numbers of valid tokens; classification-head counts weight examples equally.
Writing the corresponding empirical joint distribution as $q_{t,e}^{(s)}$ and
its marginals as $q_t^{(s)}$ and $q_e^{(s)}$, we use arithmetic-normalized
mutual information
\[
\mathrm{NMI}^{(s)}=
\frac{\sum_{t,e}q_{t,e}^{(s)}
\log\!\left(q_{t,e}^{(s)}/(q_t^{(s)}q_e^{(s)})\right)}
{\tfrac12\left(H(q_t^{(s)})+H(q_e^{(s)})\right)},
\]
with zero-mass cells omitted. ARI is the adjusted Rand index computed from the
same selection-event contingency table \citep{hubert1985ari}; it therefore
describes task association across top-$k$ expert-selection events rather than a
one-expert partition of the original examples. We compute both metrics
within each seed and then average over the five seeds. These are task--expert
association diagnostics, whereas the partition objective uses gradient
inner products; aligned gradients from different dataset tasks can therefore
favor a shared expert under the objective.
Gradient-mass purity uses per-task gradient norms, while NMI/ARI use
expert-selection counts.

\paragraph{Routed-unit structure purity.}
From the same selection-event contingency table $C^{(s)}$, define
\[
\mathrm{Purity}_{\mathrm{struct}}^{(s)} =
\frac{\sum_{e} \max_{t} C^{(s)}_{t,e}}{\sum_{t,e} C^{(s)}_{t,e}},
\]
the count-weighted share of top-$k$ selection events that belong to the
dominant task of their expert. Unlike gradient-mass purity, it counts routed
units rather than gradient norm and weights each expert by its selection mass,
so lightly used experts contribute little. It is reported for the trainable RoBERTa classification-head setting
(Appendix~\ref{appsubsec:implb_results}),
alongside NMI and ARI, and is averaged over seeds in the same way. It describes
the concentration of task labels within each expert's selection events. Since
$\sum_e\max_t C_{t,e}^{(s)}\ge\sum_e C_{t^\star,e}^{(s)}$ for the task $t^\star$
with the most routed units, structure purity is bounded below by that task's
share of the routed units (validation examples for classification-head
routing). The bound is attained when every expert is dominated by $t^\star$,
for example under task-independent routing, so values near it indicate
little task association within experts, regardless of load balance.

\paragraph{Routing entropy.}
Let $p_{n,e}^{(s)}$ denote the router probability assigned to expert $e$ for
validation example $n$ in seed $s$, with
$\sum_{e=1}^{E} p_{n,e}^{(s)} = 1$. The per-example routing entropy is
\[
H_n^{(s)} = -\sum_{e=1}^{E} p_{n,e}^{(s)} \log p_{n,e}^{(s)},
\]
For FFN routing, $p_n=q_n$ is the non-padding-token mean probability
vector defined above, so this is the entropy of that mean vector, rather
than the mean of token entropies. Classification-head routing uses its
per-example gate directly. The reported per-seed routing entropy is the validation average
\[
\mathrm{Ent}^{(s)} = \frac{1}{N_s}\sum_{n=1}^{N_s} H_n^{(s)}.
\]
When a single scalar is reported, we average over seeds:
\[
\mathrm{Ent} = \frac{1}{S}\sum_{s=1}^{S}\mathrm{Ent}^{(s)}.
\]
For the routing-entropy trend plots in Appendix~\ref{app:full_empirical_results}, the displayed value is
further normalized by $\log E$, relative to the entropy of a uniform
distribution over all $E$ experts. For the classification-head E8K4 gates used in all
reported trajectories, the per-example normalized maximum is
$\log 4/\log 8=2/3$.

\paragraph{Gradient cosine diagnostics.}
Using the same detached per-task expert gradients $g_{t,e}^{(s)}$ in the
local coordinates of each expert, the
intra-expert coherence of expert $e$ in seed $s$ is the mean pairwise cosine
similarity across task gradients within that expert:
\[
\mathrm{Intra}_{e}^{(s)} =
\frac{1}{\binom{T}{2}}
\sum_{1 \le i < j \le T}
\cos\!\bigl(g_{i,e}^{(s)}, g_{j,e}^{(s)}\bigr),
\]
where $T$ is the number of tasks in the mixture. The reported per-seed
intra-expert coherence is the average over experts,
\[
\mathrm{Intra}^{(s)} = \frac{1}{E}\sum_{e=1}^{E}\mathrm{Intra}_{e}^{(s)}.
\]
The reported intra-expert cosine includes all task pairs, assigning zero whenever
either gradient is zero. At each expert, it equals the active-pair mean multiplied
by the fraction of task pairs that are active, and is zero when no pair is active.
This all-task-pair summary reflects both coverage and directional agreement.
For inter-expert similarity, we first form the expert-level aggregate gradient
\[
\bar g_e^{(s)} = \frac{1}{T}\sum_{t=1}^{T} g_{t,e}^{(s)},
\]
and then compute the mean pairwise cosine similarity across experts, comparing
the local gradient vectors in the same parameter-entry order across the
identically structured expert blocks:
\[
\mathrm{Inter}^{(s)} =
\frac{1}{\binom{E}{2}}
\sum_{1 \le e < e' \le E}
\cos\!\bigl(\bar g_e^{(s)}, \bar g_{e'}^{(s)}\bigr).
\]
Whenever Appendix~\ref{app:full_empirical_results} shows these diagnostics as trajectories, the same formulas
are evaluated at each sampled checkpoint, averaged over mixtures within each run,
and then averaged across the five runs.

\paragraph{Task-specific F1 and MCC.}
For the single-task binary classification results, let $\mathrm{TP}$,
$\mathrm{TN}$, $\mathrm{FP}$, and $\mathrm{FN}$ denote the validation-set
confusion counts for one seed. Precision and recall are
\[
\mathrm{Prec}=\frac{\mathrm{TP}}{\max(1,\mathrm{TP}+\mathrm{FP})},
\qquad
\mathrm{Rec}=\frac{\mathrm{TP}}{\max(1,\mathrm{TP}+\mathrm{FN})}.
\]
The reported F1 score is
\[
\mathrm{F1}=
\begin{cases}
\displaystyle
\frac{2\,\mathrm{Prec}\,\mathrm{Rec}}{\mathrm{Prec}+\mathrm{Rec}},
& \text{if }\mathrm{Prec}+\mathrm{Rec}>0,\\[6pt]
0, & \text{otherwise,}
\end{cases}
\]
and the Matthews correlation coefficient is
\[
\mathrm{MCC}=
\begin{cases}
\displaystyle
\frac{\mathrm{TP}\,\mathrm{TN}-\mathrm{FP}\,\mathrm{FN}}
{\sqrt{(\mathrm{TP}+\mathrm{FP})(\mathrm{TP}+\mathrm{FN})(\mathrm{TN}+\mathrm{FP})(\mathrm{TN}+\mathrm{FN})}},
& \text{if the denominator is positive,}\\[10pt]
0, & \text{otherwise.}
\end{cases}
\]
In Appendix~\ref{appsubsec:deberta_single_task_lora_results}--\ref{appsubsec:deberta_single_task_ffn_results},
MRPC and QQP report F1 as the standard task-specific metric,
while CoLA reports MCC.

\paragraph{Higher-level summaries.}
When the main text reports backbone-level or overall multi-task summaries, it
first forms the equal-task macro average above within each seed and mixture,
then averages equally across seeds and the listed task mixtures. Any final
cross-backbone summary weights the listed backbones equally. All aggregation
and differencing use the unrounded per-seed values; decimal rounding is applied
only to the displayed table entries and prose summaries.

\paragraph{Uncertainty convention.}
Paired intervals summarize variation over five seeds conditional on the
selected configurations. We report nominal 95\% Student-$t$ intervals with
four degrees of freedom and no multiplicity adjustment. The same interval
convention applies to routing-diagnostic and coefficient-sweep comparisons.

\section{Full Empirical Results}
\label{app:full_empirical_results}

\paragraph{Metric guide.}
Accuracy is the equal-task macro validation accuracy, and seed standard deviation
describes its variation across final runs. Load variance (LVar) measures marginal
expert-load imbalance, while utilization is the fraction of experts receiving
non-negligible traffic. Gradient-mass purity is the per-expert concentration of
task-gradient norms; structure purity is the count-weighted dominant-task share
of expert selections. Normalized mutual information (NMI) and adjusted Rand
index (ARI) measure task--expert association. Intra-expert coherence and
inter-expert similarity summarize gradient direction, while normalized routing
entropy measures routing concentration. Appendix~\ref{app:experimental_protocol}
gives the formal definitions and aggregation rules.

\subsection{Multi-task results by backbone}
\label{appsubsec:task_grouped_multitask_results}

For DeBERTa, Qwen3-1.7B, and RoBERTa, the tables below report final-checkpoint
aggregates over five seeds under the frozen LoRA-FFN setting. The
DeBERTa and Qwen3-1.7B results form the two-backbone aggregate of
Table~\ref{tab:main_endpoint_controls}(b) and
Figure~\ref{fig:task_scale_accuracy_gain}; the frozen RoBERTa configuration is
reported separately in Appendix~\ref{appsubsec:roberta_multitask_tables}.
Within each backbone, the tables cover five dataset mixtures, identified by their bracketed task lists and grouped by task count. All seven
methods are evaluated on these three backbones; the Qwen3-8B extension evaluates Baseline, CAGrad, and GAR and is summarized
separately. RoBERTa additionally reports NMI and ARI for all five mixtures.

Over the five mixtures used in the main endpoint
summary, GAR has the
highest two-backbone accuracy ($0.7702$) and purity
($0.5645$), while STGC has the lowest LVar ($0.00101$).

The within-backbone diagnostic means exhibit different trade-offs. Across the
10 DeBERTa and Qwen3-1.7B combinations, GAR's LVar relative to Baseline increases
in DeBERTa's five- and six-task mixtures, and its utilization decreases in those two and in Qwen3-1.7B's
seven-task mixture containing MRPC. On DeBERTa, GAR
improves accuracy and purity relative to CAGrad, with slightly higher LVar
and lower utilization. The largest LVar reductions over Baseline occur in
the Qwen3-1.7B five- and six-task mixtures, where Baseline has highly
concentrated loads. The four simultaneous improvements in the main table
describe the equal-backbone aggregate across the five mixtures.
The two Qwen3-1.7B mixtures account for $85.2\%$ of the aggregate LVar
reduction and $69.0\%$ of the gradient-mass purity increase over Baseline.
As a leave-two-combinations-out check, the remaining 8
backbone--mixture combinations have a mean purity increase of $0.0494$
(versus $0.1277$ across all 10), a mean LVar reduction of $0.00292$, and
an accuracy gain of $1.20$ percentage points. Thus the magnitude of the
aggregate routing improvements is concentrated in these two combinations,
whereas the accuracy gain is not.

\subsubsection{RoBERTa}
\label{appsubsec:roberta_multitask_tables}

Frozen RoBERTa-base has limited downstream accuracy under the shared configuration; it is
reported separately from the two-backbone aggregate and serves as
the testbed for the top-1 extension
(Appendix~\ref{appsubsec:top1_roberta_results}) and the LoRA-FFN coefficient
sweep (Appendix~\ref{appsubsec:lambda_ablation_results}). On this
backbone, GAR improves accuracy and utilization and reduces LVar relative
to Baseline, while gradient-mass purity is essentially unchanged; adding it as a third
backbone gives paired gains of $+1.111$ $[+0.764,+1.457]$ points over Baseline
and $+0.878$ $[+0.596,+1.160]$ over CAGrad (per-mixture results in
Tables~\ref{tab:appendix_multitask_5_6_roberta} and~\ref{tab:appendix_multitask_7_8_roberta})
(Table~\ref{tab:paired_supervised_uncertainty}).

\begin{table}[H]
\centering
\scriptsize
\caption{RoBERTa five--six-task multi-task results for the seven supervised methods, with a frozen backbone, final-layer LoRA-FFN experts, and E8K4 routing. Purity is gradient-mass purity. RoBERTa additionally reports NMI and ARI task-label association diagnostics.}\label{tab:appendix_multitask_5_6_roberta}
\setlength{\tabcolsep}{3.4pt}
\renewcommand{\arraystretch}{1.03}
\AppendixTable{%
\begin{tabular}{@{}lrrrrrrr@{}}
\toprule
Method & Acc & Seed std & LVar & Purity & Util. & NMI & ARI \\
\midrule
\multicolumn{8}{@{}l}{\textbf{[QNLI, BoolQ, RTE, PAWS, WiC]} \textnormal{(5 tasks)}} \\*
Baseline & 0.6406 & 0.0061 & 0.01572 & 0.5185 & 0.550 & 0.00465 & 0.00275 \\*
CAGrad & 0.6427 & 0.0058 & 0.01135 & 0.4457 & 0.650 & 0.00713 & 0.00404 \\*
\textbf{GAR} & 0.6421 & 0.0062 & 0.00997 & 0.4817 & 0.700 & 0.00548 & 0.00333 \\*
STGC & 0.6424 & 0.0101 & 0.00322 & 0.3865 & 0.775 & 0.00285 & 0.00129 \\*
LoadPen & 0.6362 & 0.0057 & 0.01055 & 0.3961 & 0.650 & 0.00677 & 0.00348 \\*
SwitchAux & 0.6360 & 0.0093 & 0.00960 & 0.3965 & 0.625 & 0.00632 & 0.00364 \\*
STGC+Load & 0.6397 & 0.0095 & 0.00202 & 0.5139 & 0.925 & 0.00244 & 0.00177 \\
\addlinespace[2pt]
\multicolumn{8}{@{}l}{\textbf{[QNLI, BoolQ, RTE, PAWS, ANLI, CB]} \textnormal{(6 tasks)}} \\*
Baseline & 0.6384 & 0.0078 & 0.02125 & 0.5458 & 0.550 & 0.00724 & 0.00371 \\*
CAGrad & 0.6426 & 0.0054 & 0.01769 & 0.5454 & 0.550 & 0.00802 & 0.00424 \\*
\textbf{GAR} & 0.6450 & 0.0103 & 0.02055 & 0.5300 & 0.500 & 0.00717 & 0.00391 \\*
STGC & 0.6351 & 0.0093 & 0.00325 & 0.3098 & 0.850 & 0.00310 & 0.00124 \\*
LoadPen & 0.6317 & 0.0111 & 0.01924 & 0.3222 & 0.575 & 0.00891 & 0.00415 \\*
SwitchAux & 0.6430 & 0.0108 & 0.01910 & 0.3219 & 0.525 & 0.00910 & 0.00414 \\*
STGC+Load & 0.6311 & 0.0071 & 0.00169 & 0.5618 & 0.925 & 0.00259 & 0.00117 \\
\bottomrule
\end{tabular}
}%
\end{table}

\begin{table}[H]
\centering
\scriptsize
\caption{RoBERTa seven--eight-task multi-task results for the seven supervised methods, with a frozen backbone, final-layer LoRA-FFN experts, and E8K4 routing. Purity is gradient-mass purity. RoBERTa additionally reports NMI and ARI task-label association diagnostics.}\label{tab:appendix_multitask_7_8_roberta}
\setlength{\tabcolsep}{3.4pt}
\renewcommand{\arraystretch}{1.03}
\AppendixTable{%
\begin{tabular}{@{}lrrrrrrr@{}}
\toprule
Method & Acc & Seed std & LVar & Purity & Util. & NMI & ARI \\
\midrule
\multicolumn{8}{@{}l}{\textbf{[QNLI, BoolQ, RTE, PAWS, WiC, CoLA, SST-2]} \textnormal{(7 tasks)}} \\*
Baseline & 0.6809 & 0.0093 & 0.01216 & 0.3922 & 0.575 & 0.00758 & 0.00494 \\*
CAGrad & 0.6878 & 0.0159 & 0.01252 & 0.3903 & 0.650 & 0.00976 & 0.00636 \\*
\textbf{GAR} & 0.7030 & 0.0075 & 0.01155 & 0.4049 & 0.625 & 0.00604 & 0.00291 \\*
STGC & 0.6857 & 0.0151 & 0.00136 & 0.3746 & 0.975 & 0.00342 & 0.00165 \\*
LoadPen & 0.6826 & 0.0090 & 0.01572 & 0.3851 & 0.650 & 0.00614 & 0.00392 \\*
SwitchAux & 0.6932 & 0.0174 & 0.01214 & 0.3902 & 0.725 & 0.00733 & 0.00546 \\*
STGC+Load & 0.6917 & 0.0110 & 0.00175 & 0.4025 & 0.900 & 0.00288 & 0.00127 \\
\addlinespace[2pt]
\multicolumn{8}{@{}l}{\textbf{[QNLI, BoolQ, RTE, PAWS, WiC, CoLA, SST-2, CB]} \textnormal{(8 tasks)}} \\*
Baseline & 0.7081 & 0.0066 & 0.01827 & 0.3774 & 0.525 & 0.00648 & 0.00335 \\*
CAGrad & 0.6996 & 0.0062 & 0.02248 & 0.3834 & 0.575 & 0.00910 & 0.00309 \\*
\textbf{GAR} & 0.7182 & 0.0090 & 0.01575 & 0.3947 & 0.700 & 0.00627 & 0.00244 \\*
STGC & 0.7073 & 0.0167 & 0.00163 & 0.3649 & 0.950 & 0.00241 & 0.00119 \\*
LoadPen & 0.7198 & 0.0054 & 0.02704 & 0.3758 & 0.500 & 0.00746 & 0.00231 \\*
SwitchAux & 0.7108 & 0.0126 & 0.02325 & 0.3784 & 0.575 & 0.00697 & 0.00255 \\*
STGC+Load & 0.7109 & 0.0092 & 0.00211 & 0.3962 & 0.900 & 0.00246 & 0.00113 \\
\addlinespace[2pt]
\multicolumn{8}{@{}l}{\textbf{[QNLI, BoolQ, RTE, WiC, CoLA, SST-2, MRPC]} \textnormal{(7 tasks)}} \\*
Baseline & 0.7205 & 0.0090 & 0.02015 & 0.5409 & 0.575 & 0.00442 & 0.00260 \\*
CAGrad & 0.7189 & 0.0105 & 0.02032 & 0.5500 & 0.575 & 0.00391 & 0.00017 \\*
\textbf{GAR} & 0.7373 & 0.0084 & 0.00931 & 0.5647 & 0.750 & 0.00825 & 0.00416 \\*
STGC & 0.7194 & 0.0148 & 0.00101 & 0.5530 & 1.000 & 0.00207 & 0.00071 \\*
LoadPen & 0.7223 & 0.0068 & 0.01429 & 0.5467 & 0.575 & 0.00456 & 0.00143 \\*
SwitchAux & 0.7208 & 0.0069 & 0.01970 & 0.5453 & 0.600 & 0.00520 & 0.00075 \\*
STGC+Load & 0.7169 & 0.0073 & 0.00195 & 0.3604 & 0.925 & 0.00279 & 0.00094 \\
\bottomrule
\end{tabular}
}%
\end{table}

\subsubsection{DeBERTa}
\label{appsubsec:deberta_multitask_tables}

Tables~\ref{tab:appendix_multitask_5_6_deberta} and~\ref{tab:appendix_multitask_7_8_deberta}
report the five mixtures for DeBERTa.

\begin{table}[H]
\centering
\scriptsize
\caption{DeBERTa five--six-task multi-task results for the seven supervised methods, with a frozen backbone, final-layer LoRA-FFN experts, and E8K4 routing. Purity is gradient-mass purity.}\label{tab:appendix_multitask_5_6_deberta}
\setlength{\tabcolsep}{3.4pt}
\renewcommand{\arraystretch}{1.03}
\AppendixTable{%
\begin{tabular}{@{}lrrrrr@{}}
\toprule
Method & Acc & Seed std & LVar & Purity & Util. \\
\midrule
\multicolumn{6}{@{}l}{\textbf{[QNLI, BoolQ, RTE, PAWS, WiC]} \textnormal{(5 tasks)}} \\*
Baseline & 0.7763 & 0.0069 & 0.00741 & 0.5148 & 0.775 \\*
CAGrad & 0.7781 & 0.0072 & 0.01215 & 0.4988 & 0.575 \\*
\textbf{GAR} & 0.7910 & 0.0068 & 0.01698 & 0.6287 & 0.475 \\*
STGC & 0.7293 & 0.0151 & 0.00186 & 0.4135 & 0.975 \\*
LoadPen & 0.7738 & 0.0080 & 0.01083 & 0.4921 & 0.625 \\*
SwitchAux & 0.7719 & 0.0043 & 0.01175 & 0.4974 & 0.550 \\*
STGC+Load & 0.7584 & 0.0117 & 0.00103 & 0.4812 & 0.975 \\
\addlinespace[2pt]
\multicolumn{6}{@{}l}{\textbf{[QNLI, BoolQ, RTE, PAWS, ANLI, CB]} \textnormal{(6 tasks)}} \\*
Baseline & 0.7488 & 0.0078 & 0.02228 & 0.5084 & 0.500 \\*
CAGrad & 0.7523 & 0.0087 & 0.02290 & 0.4820 & 0.525 \\*
\textbf{GAR} & 0.7494 & 0.0045 & 0.03025 & 0.5870 & 0.325 \\*
STGC & 0.7285 & 0.0023 & 0.00116 & 0.4397 & 0.950 \\*
LoadPen & 0.7455 & 0.0065 & 0.01494 & 0.5227 & 0.575 \\*
SwitchAux & 0.7469 & 0.0093 & 0.01917 & 0.5275 & 0.625 \\*
STGC+Load & 0.7333 & 0.0036 & 0.00064 & 0.4353 & 1.000 \\
\bottomrule
\end{tabular}
}%
\end{table}

\begin{table}[H]
\centering
\scriptsize
\caption{DeBERTa seven--eight-task multi-task results for the seven supervised methods, with a frozen backbone, final-layer LoRA-FFN experts, and E8K4 routing. Purity is gradient-mass purity.}\label{tab:appendix_multitask_7_8_deberta}
\setlength{\tabcolsep}{3.4pt}
\renewcommand{\arraystretch}{1.03}
\AppendixTable{%
\begin{tabular}{@{}lrrrrr@{}}
\toprule
Method & Acc & Seed std & LVar & Purity & Util. \\
\midrule
\multicolumn{6}{@{}l}{\textbf{[QNLI, BoolQ, RTE, PAWS, WiC, CoLA, SST-2]} \textnormal{(7 tasks)}} \\*
Baseline & 0.8054 & 0.0064 & 0.02361 & 0.4093 & 0.450 \\*
CAGrad & 0.8060 & 0.0069 & 0.01480 & 0.4044 & 0.600 \\*
\textbf{GAR} & 0.8226 & 0.0049 & 0.01282 & 0.4496 & 0.725 \\*
STGC & 0.7988 & 0.0103 & 0.00119 & 0.3883 & 0.950 \\*
LoadPen & 0.8079 & 0.0073 & 0.01981 & 0.4242 & 0.600 \\*
SwitchAux & 0.8051 & 0.0057 & 0.01386 & 0.4152 & 0.600 \\*
STGC+Load & 0.8028 & 0.0064 & 0.00097 & 0.3953 & 0.975 \\
\addlinespace[2pt]
\multicolumn{6}{@{}l}{\textbf{[QNLI, BoolQ, RTE, PAWS, WiC, CoLA, SST-2, CB]} \textnormal{(8 tasks)}} \\*
Baseline & 0.7881 & 0.0067 & 0.01976 & 0.3923 & 0.550 \\*
CAGrad & 0.7929 & 0.0076 & 0.01114 & 0.4007 & 0.650 \\*
\textbf{GAR} & 0.8014 & 0.0132 & 0.00951 & 0.4518 & 0.650 \\*
STGC & 0.7843 & 0.0064 & 0.00113 & 0.3825 & 1.000 \\*
LoadPen & 0.7879 & 0.0110 & 0.01410 & 0.3968 & 0.575 \\*
SwitchAux & 0.7938 & 0.0067 & 0.01015 & 0.4156 & 0.650 \\*
STGC+Load & 0.8013 & 0.0099 & 0.00159 & 0.4413 & 0.925 \\
\addlinespace[2pt]
\multicolumn{6}{@{}l}{\textbf{[QNLI, BoolQ, RTE, WiC, CoLA, SST-2, MRPC]} \textnormal{(7 tasks)}} \\*
Baseline & 0.7930 & 0.0090 & 0.01290 & 0.5494 & 0.650 \\*
CAGrad & 0.7952 & 0.0112 & 0.01446 & 0.5561 & 0.650 \\*
\textbf{GAR} & 0.8060 & 0.0082 & 0.01054 & 0.5722 & 0.800 \\*
STGC & 0.7940 & 0.0115 & 0.00104 & 0.5578 & 0.975 \\*
LoadPen & 0.7953 & 0.0072 & 0.01247 & 0.5525 & 0.725 \\*
SwitchAux & 0.7917 & 0.0066 & 0.01194 & 0.5528 & 0.650 \\*
STGC+Load & 0.7764 & 0.0111 & 0.00321 & 0.3423 & 0.825 \\
\bottomrule
\end{tabular}
}%
\end{table}

\subsubsection{Qwen3-1.7B}
\label{appsubsec:qwen17_multitask_tables}

Tables~\ref{tab:appendix_multitask_5_6_qwen17} and~\ref{tab:appendix_multitask_7_8_qwen17}
report the corresponding Qwen3-1.7B endpoints.

\begin{table}[H]
\centering
\scriptsize
\caption{Qwen3-1.7B five--six-task multi-task results for the seven supervised methods, with a frozen backbone, final-layer LoRA-FFN experts, and E8K4 routing. Purity is gradient-mass purity.}\label{tab:appendix_multitask_5_6_qwen17}
\setlength{\tabcolsep}{3.4pt}
\renewcommand{\arraystretch}{1.03}
\AppendixTable{%
\begin{tabular}{@{}lrrrrr@{}}
\toprule
Method & Acc & Seed std & LVar & Purity & Util. \\
\midrule
\multicolumn{6}{@{}l}{\textbf{[QNLI, BoolQ, RTE, PAWS, WiC]} \textnormal{(5 tasks)}} \\*
Baseline & 0.7241 & 0.0086 & 0.09545 & 0.3046 & 0.150 \\*
CAGrad & 0.7235 & 0.0056 & 0.10232 & 0.3512 & 0.125 \\*
\textbf{GAR} & 0.7315 & 0.0093 & 0.03298 & 0.7981 & 0.425 \\*
STGC & 0.6714 & 0.0145 & 0.00159 & 0.3939 & 0.925 \\*
LoadPen & 0.6993 & 0.0189 & 0.01039 & 0.3953 & 0.675 \\*
SwitchAux & 0.7138 & 0.0206 & 0.00469 & 0.3923 & 0.800 \\*
STGC+Load & 0.6851 & 0.0088 & 0.00299 & 0.6009 & 0.875 \\
\addlinespace[2pt]
\multicolumn{6}{@{}l}{\textbf{[QNLI, BoolQ, RTE, PAWS, ANLI, CB]} \textnormal{(6 tasks)}} \\*
Baseline & 0.7008 & 0.0111 & 0.09173 & 0.3794 & 0.150 \\*
CAGrad & 0.7045 & 0.0046 & 0.08659 & 0.4371 & 0.200 \\*
\textbf{GAR} & 0.7072 & 0.0053 & 0.01997 & 0.7676 & 0.475 \\*
STGC & 0.7050 & 0.0180 & 0.00134 & 0.3147 & 0.950 \\*
LoadPen & 0.7010 & 0.0137 & 0.01590 & 0.3152 & 0.625 \\*
SwitchAux & 0.7028 & 0.0071 & 0.01070 & 0.3119 & 0.675 \\*
STGC+Load & 0.6694 & 0.0113 & 0.00258 & 0.4375 & 0.950 \\
\bottomrule
\end{tabular}
}%
\end{table}

\begin{table}[H]
\centering
\scriptsize
\caption{Qwen3-1.7B seven--eight-task multi-task results for the seven supervised methods, with a frozen backbone, final-layer LoRA-FFN experts, and E8K4 routing. Purity is gradient-mass purity.}\label{tab:appendix_multitask_7_8_qwen17}
\setlength{\tabcolsep}{3.4pt}
\renewcommand{\arraystretch}{1.03}
\AppendixTable{%
\begin{tabular}{@{}lrrrrr@{}}
\toprule
Method & Acc & Seed std & LVar & Purity & Util. \\
\midrule
\multicolumn{6}{@{}l}{\textbf{[QNLI, BoolQ, RTE, PAWS, WiC, CoLA, SST-2]} \textnormal{(7 tasks)}} \\*
Baseline & 0.7475 & 0.0075 & 0.00954 & 0.3834 & 0.625 \\*
CAGrad & 0.7530 & 0.0085 & 0.01552 & 0.3858 & 0.600 \\*
\textbf{GAR} & 0.7603 & 0.0078 & 0.00693 & 0.4208 & 0.850 \\*
STGC & 0.7400 & 0.0079 & 0.00026 & 0.3802 & 1.000 \\*
LoadPen & 0.7499 & 0.0045 & 0.00127 & 0.3911 & 0.975 \\*
SwitchAux & 0.7467 & 0.0066 & 0.00103 & 0.3790 & 1.000 \\*
STGC+Load & 0.7283 & 0.0060 & 0.00059 & 0.3216 & 1.000 \\
\addlinespace[2pt]
\multicolumn{6}{@{}l}{\textbf{[QNLI, BoolQ, RTE, PAWS, WiC, CoLA, SST-2, CB]} \textnormal{(8 tasks)}} \\*
Baseline & 0.7526 & 0.0096 & 0.01849 & 0.3721 & 0.575 \\*
CAGrad & 0.7535 & 0.0169 & 0.01148 & 0.3754 & 0.650 \\*
\textbf{GAR} & 0.7682 & 0.0070 & 0.00814 & 0.3978 & 0.750 \\*
STGC & 0.7483 & 0.0065 & 0.00024 & 0.3696 & 1.000 \\*
LoadPen & 0.7496 & 0.0111 & 0.00192 & 0.3768 & 0.950 \\*
SwitchAux & 0.7520 & 0.0065 & 0.00072 & 0.3714 & 0.975 \\*
STGC+Load & 0.7294 & 0.0059 & 0.00018 & 0.2882 & 1.000 \\
\addlinespace[2pt]
\multicolumn{6}{@{}l}{\textbf{[QNLI, BoolQ, RTE, WiC, CoLA, SST-2, MRPC]} \textnormal{(7 tasks)}} \\*
Baseline & 0.7562 & 0.0045 & 0.01064 & 0.5541 & 0.775 \\*
CAGrad & 0.7655 & 0.0092 & 0.00806 & 0.5548 & 0.675 \\*
\textbf{GAR} & 0.7648 & 0.0033 & 0.00608 & 0.5712 & 0.750 \\*
STGC & 0.7389 & 0.0106 & 0.00034 & 0.5602 & 1.000 \\*
LoadPen & 0.7461 & 0.0085 & 0.00154 & 0.5599 & 0.975 \\*
SwitchAux & 0.7488 & 0.0040 & 0.00055 & 0.5606 & 1.000 \\*
STGC+Load & 0.7277 & 0.0055 & 0.00012 & 0.3279 & 1.000 \\
\bottomrule
\end{tabular}
}%
\end{table}

\subsubsection{Qwen3-8B}
\label{appsubsec:qwen8b_results}
Tables~\ref{tab:qwen8b_5_6} and~\ref{tab:qwen8b_7_8}
give the mixture endpoints; Tables~\ref{tab:qwen8b_aggregate} and
\ref{tab:qwen8b_paired} summarize their means and paired accuracy uncertainty.
We additionally evaluate a frozen Qwen3-8B backbone with final-layer LoRA-FFN experts on the same five dataset mixtures. This extension compares Baseline, CAGrad, and GAR. It uses five final seeds and 2,000 optimizer updates. This extension is reported separately from the two-backbone, seven-method frozen aggregate and its Pareto comparison.
Token-level routing uses E8K4 in all five mixtures. The auxiliary group probability averages the routed token probabilities within examples and then equally across same-task examples, without an additional group-level top-$k$ truncation, as in Section~\ref{sec:methodology_implementation}. All configurations are listed in Appendix~\ref{appsubsec:qwen8b_hyperparameters}.
\begin{table}[H]
\centering\footnotesize
\caption{Qwen3-8B five--six-task mixture endpoints. Acc is equal-task macro validation accuracy; seed std is its population standard deviation over five seeds. Routing metrics use the final diagnostic checkpoint; Purity is gradient-mass purity.}
\label{tab:qwen8b_5_6}
\setlength{\tabcolsep}{6pt}
\renewcommand{\arraystretch}{1.04}
\begin{tabular*}{\linewidth}{@{\extracolsep{\fill}}lrrrrr@{}}
\toprule
Method & Acc & Seed std & LVar & Purity & Util. \\
\midrule
\multicolumn{6}{@{}l}{\textbf{[QNLI, BoolQ, RTE, PAWS, WiC]}} \\*
Baseline & 0.7268 & 0.0077 & 0.01229 & 0.5835 & 0.600 \\*
CAGrad & 0.7360 & 0.0030 & 0.00966 & 0.6466 & 0.700 \\*
GAR & 0.7326 & 0.0033 & 0.00790 & 0.6401 & 0.650 \\*
\addlinespace[3pt]
\multicolumn{6}{@{}l}{\textbf{[QNLI, BoolQ, RTE, PAWS, ANLI, CB]}} \\*
Baseline & 0.6950 & 0.0113 & 0.05777 & 0.4679 & 0.375 \\*
CAGrad & 0.7066 & 0.0047 & 0.04762 & 0.4537 & 0.475 \\*
GAR & 0.7038 & 0.0032 & 0.01301 & 0.6026 & 0.625 \\*
\bottomrule
\end{tabular*}
\end{table}
\begin{table}[H]
\centering\footnotesize
\caption{Qwen3-8B seven--eight-task mixture endpoints. Acc is equal-task macro validation accuracy; seed std is its population standard deviation over five seeds. Routing metrics use the final diagnostic checkpoint; Purity is gradient-mass purity.}
\label{tab:qwen8b_7_8}
\setlength{\tabcolsep}{6pt}
\renewcommand{\arraystretch}{1.04}
\begin{tabular*}{\linewidth}{@{\extracolsep{\fill}}lrrrrr@{}}
\toprule
Method & Acc & Seed std & LVar & Purity & Util. \\
\midrule
\multicolumn{6}{@{}l}{\textbf{[QNLI, BoolQ, RTE, PAWS, WiC, CoLA, SST-2]}} \\*
Baseline & 0.7466 & 0.0088 & 0.01096 & 0.3761 & 0.625 \\*
CAGrad & 0.7509 & 0.0075 & 0.01623 & 0.3934 & 0.600 \\*
GAR & 0.7641 & 0.0059 & 0.00713 & 0.4391 & 0.775 \\*
\addlinespace[3pt]
\multicolumn{6}{@{}l}{\textbf{[QNLI, BoolQ, RTE, WiC, CoLA, SST-2, MRPC]}} \\*
Baseline & 0.7475 & 0.0048 & 0.01020 & 0.5165 & 0.725 \\*
CAGrad & 0.7560 & 0.0094 & 0.00836 & 0.5243 & 0.625 \\*
GAR & 0.7602 & 0.0060 & 0.00554 & 0.5292 & 0.700 \\*
\addlinespace[3pt]
\multicolumn{6}{@{}l}{\textbf{[QNLI, BoolQ, RTE, PAWS, WiC, CoLA, SST-2, CB]}} \\*
Baseline & 0.7481 & 0.0084 & 0.01793 & 0.3724 & 0.475 \\*
CAGrad & 0.7513 & 0.0108 & 0.01260 & 0.3764 & 0.575 \\*
GAR & 0.7652 & 0.0085 & 0.00707 & 0.4054 & 0.700 \\*
\bottomrule
\end{tabular*}
\end{table}
\begin{table}[H]
\centering\small
\caption{Qwen3-8B aggregate endpoints. Every mixture and seed has equal weight across the five dataset mixtures, whose constituent tasks are listed in brackets in the preceding tables. These three-method means do not enter the two-backbone, seven-method frozen aggregate.}
\label{tab:qwen8b_aggregate}
\setlength{\tabcolsep}{7pt}
\begin{tabular}{@{}llrrrr@{}}
\toprule
Scope & Method & Acc & LVar & Purity & Util. \\
\midrule
Five mixtures & Baseline & 0.7328 & 0.02183 & 0.4633 & 0.560 \\
 & CAGrad & 0.7401 & 0.01890 & 0.4789 & 0.595 \\
 & GAR & 0.7452 & 0.00813 & 0.5233 & 0.690 \\
\bottomrule
\end{tabular}
\end{table}
\begin{table}[H]
\centering\small
\caption{Paired Qwen3-8B accuracy gains (percentage points). Within each shared seed, first average the mixture endpoints and then subtract the comparator. Brackets are two-sided 95\% Student-$t$ intervals over the five paired seed differences (four degrees of freedom); no multiplicity adjustment is applied.}
\label{tab:qwen8b_paired}
\begin{tabular}{@{}lcc@{}}
\toprule
Scope & GAR$-$Baseline & GAR$-$CAGrad \\
\midrule
Five mixtures & $+1.239$ [$+0.514$, $+1.965$] & $+0.505$ [$+0.059$, $+0.950$] \\
\bottomrule
\end{tabular}
\end{table}
GAR improves the five-mixture mean by 1.24 percentage points over Baseline and 0.50 over CAGrad, alongside lower load variance and higher gradient-mass purity and utilization. Across the five mixture means, GAR exceeds Baseline in all five and CAGrad in three; CAGrad is higher on [QNLI, BoolQ, RTE, PAWS, WiC] and [QNLI, BoolQ, RTE, PAWS, ANLI, CB].

\subsubsection{Paired uncertainty for the frozen LoRA-FFN comparisons}
\label{appsubsec:paired_supervised_uncertainty}

\begin{table}[H]
\centering
\small
\caption{Paired accuracy differences for the five mixtures (five to eight tasks), in percentage points. For each backbone and seed, we average the equal-task macro validation accuracy over the five matched dataset sets and then subtract the comparator's result from GAR's for the same seed. The last two rows first average the two (DeBERTa, Qwen3-1.7B) or three backbone-specific differences within each seed; the two-backbone mean is the frozen aggregate used in the main text. Brackets are two-sided 95\% Student-$t$ confidence intervals over the five paired seed differences ($n=5$, four degrees of freedom). No multiplicity correction is applied.}
\label{tab:paired_supervised_uncertainty}
\setlength{\tabcolsep}{5.5pt}
\begin{tabular}{@{}lcc@{}}
\toprule
Scope & GAR $-$ Baseline & GAR $-$ CAGrad \\
\midrule
RoBERTa & $+1.141$ [$+0.760$, $+1.523$] & $+1.078$ [$+0.632$, $+1.523$] \\
DeBERTa & $+1.176$ [$+0.687$, $+1.666$] & $+0.917$ [$+0.350$, $+1.485$] \\
Qwen3-1.7B & $+1.014$ [$+0.260$, $+1.768$] & $+0.638$ [$-0.099$, $+1.375$] \\
Two-backbone mean (DeBERTa, Qwen3-1.7B) & $+1.095$ [$+0.661$, $+1.529$] & $+0.778$ [$+0.266$, $+1.289$] \\
Three-backbone mean & $+1.111$ [$+0.764$, $+1.457$] & $+0.878$ [$+0.596$, $+1.160$] \\
\bottomrule
\end{tabular}
\end{table}

GAR$-$Baseline intervals are positive on every backbone. GAR$-$CAGrad is
positive on DeBERTa and RoBERTa and has a positive mean on Qwen3-1.7B, where
its interval includes zero.

\subsection{Frozen RoBERTa Top-1 LoRA-FFN}
\label{appsubsec:top1_roberta_results}

This extension uses the frozen RoBERTa LoRA-FFN architecture, with configurations
selected by the common protocol (Appendix~\ref{appsubsec:top1_roberta_protocol}), to
test top-1 straight-through gating.
Table~\ref{tab:appendix_top1_roberta_results} reports mixture-level endpoints
for Baseline, CAGrad, and GAR under E8K1 routing at 2,000 updates
(Appendix~\ref{appsubsec:top1_roberta_protocol}). It covers the same five
dataset mixtures and five seeds as the main comparison. The
five-mixture mean accuracies are $66.49\%$, $66.68\%$, and $67.26\%$,
respectively. GAR improves on Baseline in all five mixture means and on
CAGrad in four; CAGrad has the highest accuracy on
[QNLI, BoolQ, RTE, PAWS, WiC, CoLA, SST-2].

Averaging the five mixtures within each seed, the paired GAR$-$Baseline
accuracy gain is $+0.77$ percentage points with a 95\% Student-$t$ interval
$[+0.59,+0.94]$. The GAR$-$CAGrad difference is $+0.58$ points with interval
$[-0.24,+1.40]$. GAR also has lower aggregate LVar and higher
gradient-mass purity and utilization than both comparators.
Table~\ref{tab:appendix_top1_roberta_paired} reports the corresponding
paired differences and task-count summaries.

\begin{table}[H]
\centering
\scriptsize
\caption{Frozen RoBERTa LoRA-FFN with top-1 routing (E8K1), trained for 2,000 optimizer updates on each of the five mixtures. Columns report equal-task macro validation accuracy, its population seed standard deviation, LVar, gradient-mass purity, and utilization at the final evaluation, with metrics defined as in Appendix~\ref{app:experimental_protocol}. Five seeds per cell. For the aggregate block, each metric is first averaged equally over the five mixtures within each seed; the displayed mean and accuracy standard deviation are then computed across the five seeds.}\label{tab:appendix_top1_roberta_results}
\setlength{\tabcolsep}{5pt}
\renewcommand{\arraystretch}{1.03}
\begin{tabular}{@{}lrrrrr@{}}
\toprule
Method & Acc & Seed std & LVar & \shortstack{Gradient-mass\\purity} & Util. \\
\midrule
\multicolumn{6}{@{}l}{\textbf{[QNLI, BoolQ, RTE, PAWS, WiC]} \textnormal{(5 tasks)}} \\*
Baseline & 0.6294 & 0.0091 & 0.01398 & 0.4734 & 0.575 \\*
CAGrad & 0.6288 & 0.0093 & 0.00937 & 0.4222 & 0.700 \\*
\textbf{GAR} & 0.6367 & 0.0076 & 0.00848 & 0.4966 & 0.675 \\
\midrule
\multicolumn{6}{@{}l}{\textbf{[QNLI, BoolQ, RTE, PAWS, ANLI, CB]} \textnormal{(6 tasks)}} \\*
Baseline & 0.6375 & 0.0069 & 0.02184 & 0.4472 & 0.575 \\*
CAGrad & 0.6309 & 0.0061 & 0.03027 & 0.4952 & 0.475 \\*
\textbf{GAR} & 0.6442 & 0.0044 & 0.01755 & 0.4859 & 0.600 \\
\midrule
\multicolumn{6}{@{}l}{\textbf{[QNLI, BoolQ, RTE, PAWS, WiC, CoLA, SST-2]} \textnormal{(7 tasks)}} \\*
Baseline & 0.6672 & 0.0059 & 0.00711 & 0.3950 & 0.775 \\*
CAGrad & 0.6776 & 0.0042 & 0.00853 & 0.4707 & 0.675 \\*
\textbf{GAR} & 0.6742 & 0.0079 & 0.00669 & 0.4473 & 0.775 \\
\midrule
\multicolumn{6}{@{}l}{\textbf{[QNLI, BoolQ, RTE, WiC, CoLA, SST-2, MRPC]} \textnormal{(7 tasks)}} \\*
Baseline & 0.6818 & 0.0032 & 0.00893 & 0.4480 & 0.725 \\*
CAGrad & 0.6888 & 0.0061 & 0.00678 & 0.4516 & 0.750 \\*
\textbf{GAR} & 0.6917 & 0.0080 & 0.00503 & 0.4608 & 0.850 \\
\midrule
\multicolumn{6}{@{}l}{\textbf{[QNLI, BoolQ, RTE, PAWS, WiC, CoLA, SST-2, CB]} \textnormal{(8 tasks)}} \\*
Baseline & 0.7087 & 0.0096 & 0.02264 & 0.4088 & 0.375 \\*
CAGrad & 0.7081 & 0.0052 & 0.02292 & 0.3822 & 0.500 \\*
\textbf{GAR} & 0.7164 & 0.0082 & 0.01424 & 0.4747 & 0.550 \\
\midrule
\multicolumn{6}{@{}l}{\emph{Five-mixture aggregate}} \\*
Baseline & 0.6649 & 0.0044 & 0.01490 & 0.4345 & 0.605 \\*
CAGrad & 0.6668 & 0.0046 & 0.01557 & 0.4444 & 0.620 \\*
\textbf{GAR} & 0.6726 & 0.0037 & 0.01040 & 0.4731 & 0.690 \\
\bottomrule
\end{tabular}
\end{table}

\begin{table}[H]
\centering
\small
\caption{Paired final-evaluation differences for frozen RoBERTa LoRA-FFN with top-1 routing (E8K1; 2,000 optimizer updates). For each seed, metrics are averaged equally over the five mixtures before subtracting the comparator for that seed. The lower block repeats the accuracy comparison within the five--six- and seven--eight-task groups. Accuracy is in percentage points; LVar, gradient-mass purity, and utilization are in their native units. Brackets are unadjusted two-sided 95\% Student-$t$ confidence intervals over the five paired seed differences ($n=5$, four degrees of freedom).}
\label{tab:appendix_top1_roberta_paired}
\setlength{\tabcolsep}{5.5pt}
\begin{tabular}{@{}lcc@{}}
\toprule
Quantity & GAR $-$ Baseline & GAR $-$ CAGrad \\
\midrule
\multicolumn{3}{@{}l}{\emph{Five mixtures (five to eight tasks)}} \\
Accuracy (pp) & $+0.769$ [$+0.595$, $+0.944$] & $+0.577$ [$-0.243$, $+1.398$] \\
LVar & $-0.00450$ [$-0.00604$, $-0.00296$] & $-0.00518$ [$-0.00869$, $-0.00167$] \\
Gradient-mass purity & $+0.0386$ [$+0.0281$, $+0.0491$] & $+0.0287$ [$+0.0065$, $+0.0509$] \\
Utilization & $+0.0850$ [$+0.0071$, $+0.1629$] & $+0.0700$ [$+0.0064$, $+0.1336$] \\
\midrule
\multicolumn{3}{@{}l}{\emph{Accuracy by task-count bin (pp)}} \\
Five--six tasks & $+0.701$ [$+0.344$, $+1.058$] & $+1.057$ [$-0.385$, $+2.499$] \\
Seven--eight tasks & $+0.815$ [$+0.630$, $+0.999$] & $+0.258$ [$-0.225$, $+0.740$] \\
\bottomrule
\end{tabular}
\end{table}

\subsection{Trainable RoBERTa Classification-Head MoE}
\label{appsubsec:implb_results}

Tables~\ref{tab:appendix_implb_5_6} and~\ref{tab:appendix_implb_7_8} report the same seven methods in the trainable
RoBERTa classification-head setting (Appendix~\ref{appsubsec:implb_protocol})
for all five mixtures, split into the same five--six and
seven--eight groups as above. In this setting, every cell also reports the routed-unit structure purity, NMI, and ARI of
Appendix~\ref{app:experimental_protocol} in addition to the gradient-mass
purity. On the five
mixtures (2,000-update budget; Appendix~\ref{appsubsec:implb_protocol}), the paired GAR$-$Baseline difference is $+1.07$
$[+0.70,+1.43]$ points and GAR$-$CAGrad is $+1.84$ $[+1.29,+2.39]$
(Table~\ref{tab:appendix_implb_paired}). By bin, GAR$-$Baseline is
$+1.39$ $[+0.46,+2.32]$ and $+0.85$ $[+0.43,+1.26]$
points for the five--six and seven--eight groups, with the
largest mean gain in the five--six-task group. One LoadPen run in the five-task mixture
reached a final accuracy of $0.49$, compared with $0.76$--$0.79$ for the other
four seeds, which is why its seed standard deviation there is $0.117$ and its
paired accuracy interval in Table~\ref{tab:appendix_implb_paired} is wide.

Table~\ref{tab:appendix_implb_paired} reports the paired endpoint differences
for every routing diagnostic. GAR has lower LVar and higher utilization than
Baseline, CAGrad, and LoadPen, and higher NMI, ARI, and structure purity than
every control, while it has higher LVar and lower utilization than STGC,
SwitchAux, and STGC+Load.
The paired gradient-mass purity intervals against Baseline, CAGrad, and LoadPen
include zero; those against STGC, SwitchAux, and STGC+Load are positive. For inter-expert
similarity and normalized routing entropy, the paired intervals against
Baseline and CAGrad also include zero.

\begin{table}[H]
\centering
\scriptsize
\caption{Trainable RoBERTa classification-head LoRA MoE five--six-task results for the seven methods under E8K4 routing. Grad.\ purity is the gradient-mass expert purity of Appendix~\ref{app:experimental_protocol}; Struct.\ purity, NMI, and ARI are computed from the routed-unit selection counts recorded at the final forward evaluation (Appendix~\ref{app:experimental_protocol}). Seed std is the population standard deviation of the equal-task macro accuracy over the five seeds.}\label{tab:appendix_implb_5_6}
\setlength{\tabcolsep}{3.0pt}
\renewcommand{\arraystretch}{1.03}
\AppendixTable{%
\begin{tabular}{@{}lrrrrrrrr@{}}
\toprule
Method & Acc & Seed std & LVar & Grad.\ purity & Struct.\ purity & Util. & NMI & ARI \\
\midrule
\multicolumn{9}{@{}l}{\textbf{[QNLI, BoolQ, RTE, PAWS, WiC]} \textnormal{(5 tasks)}} \\*
Baseline & 0.7672 & 0.0205 & 0.00846 & 0.4847 & 0.4778 & 0.700 & 0.03410 & 0.01002 \\*
CAGrad & 0.7581 & 0.0095 & 0.00786 & 0.4663 & 0.4682 & 0.700 & 0.03128 & 0.00800 \\*
\textbf{GAR} & 0.7873 & 0.0100 & 0.00605 & 0.4895 & 0.4880 & 0.750 & 0.04000 & 0.01639 \\*
STGC & 0.7843 & 0.0032 & 0.00052 & 0.3796 & 0.4601 & 1.000 & 0.00652 & 0.00553 \\*
LoadPen & 0.7243 & 0.1169 & 0.00810 & 0.3634 & 0.4646 & 0.725 & 0.02247 & 0.00833 \\*
SwitchAux & 0.7809 & 0.0106 & 0.00122 & 0.4163 & 0.4648 & 0.975 & 0.02305 & 0.00903 \\*
STGC+Load & 0.7885 & 0.0043 & 0.00064 & 0.3699 & 0.4538 & 0.975 & 0.00569 & 0.00458 \\
\addlinespace[2pt]
\multicolumn{9}{@{}l}{\textbf{[QNLI, BoolQ, RTE, PAWS, ANLI, CB]} \textnormal{(6 tasks)}} \\*
Baseline & 0.7634 & 0.0111 & 0.00882 & 0.4036 & 0.4197 & 0.650 & 0.04878 & 0.01967 \\*
CAGrad & 0.7594 & 0.0051 & 0.00908 & 0.4332 & 0.4267 & 0.625 & 0.05060 & 0.01951 \\*
\textbf{GAR} & 0.7711 & 0.0090 & 0.00464 & 0.4946 & 0.4664 & 0.825 & 0.10221 & 0.05443 \\*
STGC & 0.7599 & 0.0072 & 0.00128 & 0.3666 & 0.3986 & 0.975 & 0.01418 & 0.00952 \\*
LoadPen & 0.7546 & 0.0201 & 0.00786 & 0.4807 & 0.4232 & 0.625 & 0.05521 & 0.02162 \\*
SwitchAux & 0.7652 & 0.0089 & 0.00242 & 0.4273 & 0.4166 & 0.875 & 0.05298 & 0.02422 \\*
STGC+Load & 0.7753 & 0.0093 & 0.00120 & 0.3503 & 0.3986 & 0.975 & 0.01276 & 0.00818 \\
\bottomrule
\end{tabular}
}%
\end{table}

\begin{table}[H]
\centering
\scriptsize
\caption{Trainable RoBERTa classification-head LoRA MoE seven--eight-task results for the seven methods under E8K4 routing. Grad.\ purity is the gradient-mass expert purity of Appendix~\ref{app:experimental_protocol}; Struct.\ purity, NMI, and ARI are computed from the routed-unit selection counts recorded at the final forward evaluation (Appendix~\ref{app:experimental_protocol}). Seed std is the population standard deviation of the equal-task macro accuracy over the five seeds.}\label{tab:appendix_implb_7_8}
\setlength{\tabcolsep}{3.0pt}
\renewcommand{\arraystretch}{1.03}
\AppendixTable{%
\begin{tabular}{@{}lrrrrrrrr@{}}
\toprule
Method & Acc & Seed std & LVar & Grad.\ purity & Struct.\ purity & Util. & NMI & ARI \\
\midrule
\multicolumn{9}{@{}l}{\textbf{[QNLI, BoolQ, RTE, PAWS, WiC, CoLA, SST-2]} \textnormal{(7 tasks)}} \\*
Baseline & 0.8065 & 0.0064 & 0.00635 & 0.3714 & 0.4302 & 0.800 & 0.03329 & 0.01180 \\*
CAGrad & 0.8099 & 0.0049 & 0.00728 & 0.4154 & 0.4284 & 0.725 & 0.02706 & 0.01064 \\*
\textbf{GAR} & 0.8141 & 0.0085 & 0.00439 & 0.3803 & 0.4291 & 0.800 & 0.03937 & 0.01738 \\*
STGC & 0.8065 & 0.0068 & 0.00015 & 0.3111 & 0.4089 & 1.000 & 0.00339 & 0.00213 \\*
LoadPen & 0.8094 & 0.0066 & 0.00596 & 0.4122 & 0.4379 & 0.750 & 0.03746 & 0.02076 \\*
SwitchAux & 0.7710 & 0.0081 & 0.00351 & 0.3387 & 0.4350 & 0.875 & 0.03757 & 0.01863 \\*
STGC+Load & 0.8045 & 0.0067 & 0.00025 & 0.3206 & 0.4089 & 1.000 & 0.00486 & 0.00295 \\
\addlinespace[2pt]
\multicolumn{9}{@{}l}{\textbf{[QNLI, BoolQ, RTE, PAWS, WiC, CoLA, SST-2, CB]} \textnormal{(8 tasks)}} \\*
Baseline & 0.8140 & 0.0078 & 0.01066 & 0.4746 & 0.4176 & 0.625 & 0.02816 & 0.00636 \\*
CAGrad & 0.8093 & 0.0071 & 0.00665 & 0.4426 & 0.4302 & 0.700 & 0.03401 & 0.01101 \\*
\textbf{GAR} & 0.8202 & 0.0082 & 0.00342 & 0.4235 & 0.4518 & 0.850 & 0.04862 & 0.02976 \\*
STGC & 0.8075 & 0.0061 & 0.00041 & 0.2900 & 0.4113 & 1.000 & 0.00591 & 0.00445 \\*
LoadPen & 0.8073 & 0.0075 & 0.00826 & 0.4529 & 0.4230 & 0.675 & 0.02945 & 0.01257 \\*
SwitchAux & 0.8054 & 0.0135 & 0.00378 & 0.3951 & 0.4405 & 0.850 & 0.04381 & 0.02252 \\*
STGC+Load & 0.8161 & 0.0085 & 0.00031 & 0.2891 & 0.4087 & 1.000 & 0.00459 & 0.00293 \\
\addlinespace[2pt]
\multicolumn{9}{@{}l}{\textbf{[QNLI, BoolQ, RTE, WiC, CoLA, SST-2, MRPC]} \textnormal{(7 tasks)}} \\*
Baseline & 0.7997 & 0.0072 & 0.00647 & 0.3527 & 0.4617 & 0.650 & 0.03193 & 0.01586 \\*
CAGrad & 0.7755 & 0.0168 & 0.01064 & 0.4272 & 0.4614 & 0.600 & 0.01991 & 0.00974 \\*
\textbf{GAR} & 0.8113 & 0.0081 & 0.00323 & 0.3544 & 0.4708 & 0.850 & 0.04353 & 0.02672 \\*
STGC & 0.8007 & 0.0065 & 0.00014 & 0.2825 & 0.4564 & 1.000 & 0.00398 & 0.00252 \\*
LoadPen & 0.7771 & 0.0010 & 0.00572 & 0.3749 & 0.4630 & 0.725 & 0.03267 & 0.01566 \\*
SwitchAux & 0.7981 & 0.0099 & 0.00209 & 0.3134 & 0.4619 & 0.800 & 0.03383 & 0.01677 \\*
STGC+Load & 0.7991 & 0.0046 & 0.00024 & 0.3107 & 0.4564 & 1.000 & 0.00442 & 0.00272 \\
\bottomrule
\end{tabular}
}%
\end{table}

\begin{table}[H]
\centering
\scriptsize
\caption{Paired endpoint differences (GAR minus Baseline, CAGrad, STGC, LoadPen, SwitchAux, and STGC+Load) in the trainable RoBERTa classification-head setting on the five mixtures (five to eight tasks). Metrics are first averaged over mixtures within each seed; brackets are two-sided 95\% Student-$t$ intervals over five paired differences. Accuracy is in percentage points; routing entropy is normalized by $\log 8$. Lower LVar indicates more uniform load, lower inter-expert similarity indicates greater directional separation, and lower entropy indicates more concentrated routing. Entropy, inter-expert similarity, NMI, and ARI have no universally preferred direction. No multiplicity correction is applied. The lower block compares GAR with the three controls that include an explicit load-balancing term.}
\label{tab:appendix_implb_paired}
\setlength{\tabcolsep}{5pt}
\renewcommand{\arraystretch}{1.05}
\begin{tabular}{@{}lccc@{}}
\toprule
Metric & GAR $-$ Baseline & GAR $-$ CAGrad & GAR $-$ STGC \\
\midrule
Accuracy (pp) & $+1.07$ [$+0.70$, $+1.43$] & $+1.84$ [$+1.29$, $+2.39$] & $+0.90$ [$+0.35$, $+1.46$] \\
LVar & $-0.0038$ [$-0.0050$, $-0.0026$] & $-0.0040$ [$-0.0059$, $-0.0020$] & $+0.0038$ [$+0.0030$, $+0.0047$] \\
Utilization & $+0.130$ [$+0.041$, $+0.219$] & $+0.145$ [$+0.078$, $+0.212$] & $-0.180$ [$-0.231$, $-0.129$] \\
Grad.\ purity & $+0.0111$ [$-0.0156$, $+0.0378$] & $-0.0085$ [$-0.0492$, $+0.0323$] & $+0.1025$ [$+0.0823$, $+0.1227$] \\
Struct.\ purity & $+0.0198$ [$+0.0009$, $+0.0387$] & $+0.0182$ [$+0.0016$, $+0.0349$] & $+0.0342$ [$+0.0225$, $+0.0459$] \\
NMI & $+0.0195$ [$+0.0047$, $+0.0343$] & $+0.0222$ [$+0.0035$, $+0.0409$] & $+0.0479$ [$+0.0376$, $+0.0583$] \\
ARI & $+0.0162$ [$+0.0060$, $+0.0264$] & $+0.0172$ [$+0.0056$, $+0.0287$] & $+0.0241$ [$+0.0178$, $+0.0304$] \\
Intra-expert coherence & $+0.0315$ [$+0.0019$, $+0.0611$] & $+0.0211$ [$-0.0308$, $+0.0729$] & $+0.0171$ [$-0.0518$, $+0.0859$] \\
Inter-expert similarity & $-0.0101$ [$-0.0512$, $+0.0310$] & $-0.0018$ [$-0.0162$, $+0.0126$] & $-0.0199$ [$-0.0487$, $+0.0089$] \\
Norm.\ routing entropy & $-0.0193$ [$-0.0487$, $+0.0101$] & $-0.0204$ [$-0.0444$, $+0.0037$] & $-0.0397$ [$-0.0719$, $-0.0075$] \\
\midrule
Metric & GAR $-$ LoadPen & GAR $-$ SwitchAux & GAR $-$ STGC+Load \\
\midrule
Accuracy (pp) & $+2.63$ [$-0.98$, $+6.23$] & $+1.67$ [$+1.31$, $+2.03$] & $+0.41$ [$-0.31$, $+1.13$] \\
LVar & $-0.0028$ [$-0.0047$, $-0.0010$] & $+0.0017$ [$+0.0004$, $+0.0031$] & $+0.0038$ [$+0.0029$, $+0.0048$] \\
Utilization & $+0.115$ [$+0.028$, $+0.202$] & $-0.060$ [$-0.162$, $+0.042$] & $-0.175$ [$-0.237$, $-0.113$] \\
Grad.\ purity & $+0.0117$ [$-0.0208$, $+0.0442$] & $+0.0503$ [$+0.0320$, $+0.0687$] & $+0.1004$ [$+0.0584$, $+0.1423$] \\
Struct.\ purity & $+0.0189$ [$+0.0120$, $+0.0259$] & $+0.0175$ [$+0.0104$, $+0.0246$] & $+0.0359$ [$+0.0235$, $+0.0484$] \\
NMI & $+0.0193$ [$+0.0147$, $+0.0239$] & $+0.0165$ [$+0.0110$, $+0.0220$] & $+0.0483$ [$+0.0384$, $+0.0582$] \\
ARI & $+0.0131$ [$+0.0108$, $+0.0155$] & $+0.0107$ [$+0.0081$, $+0.0133$] & $+0.0247$ [$+0.0180$, $+0.0314$] \\
Intra-expert coherence & $+0.0423$ [$+0.0050$, $+0.0796$] & $+0.0151$ [$-0.0197$, $+0.0499$] & $+0.0092$ [$-0.0395$, $+0.0580$] \\
Inter-expert similarity & $-0.0027$ [$-0.0273$, $+0.0220$] & $+0.0055$ [$-0.0369$, $+0.0480$] & $-0.0100$ [$-0.0501$, $+0.0300$] \\
Norm.\ routing entropy & $-0.0039$ [$-0.0641$, $+0.0563$] & $-0.0154$ [$-0.0540$, $+0.0233$] & $-0.0397$ [$-0.0719$, $-0.0075$] \\
\bottomrule
\end{tabular}
\end{table}

\subsubsection{Classification-Head Routing Trajectories}
\label{appsubsec:implb_mechanism_trajectories}

\begin{figure}[H]
\centering
\includegraphics[width=0.99\textwidth]{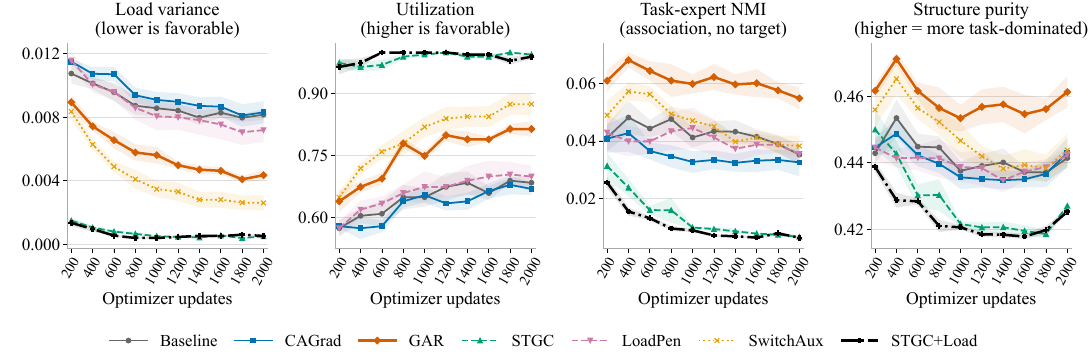}
\caption{Routing trajectories for all seven methods in the trainable RoBERTa classification-head setting (E8K4), aggregated over the five mixtures. At each checkpoint, mixtures are averaged within each run before the five run-level means are averaged. Checkpoints are 200, 400, $\ldots$, 2,000 optimizer updates; bands are $\pm1$ standard error over the five run-level mixture means. Left to right: load variance, utilization, task--expert NMI, and routed-unit structure purity.}
\label{fig:appendix_implb_mechanism_main}
\end{figure}

\begin{figure}[H]
\centering
\includegraphics[width=0.99\textwidth]{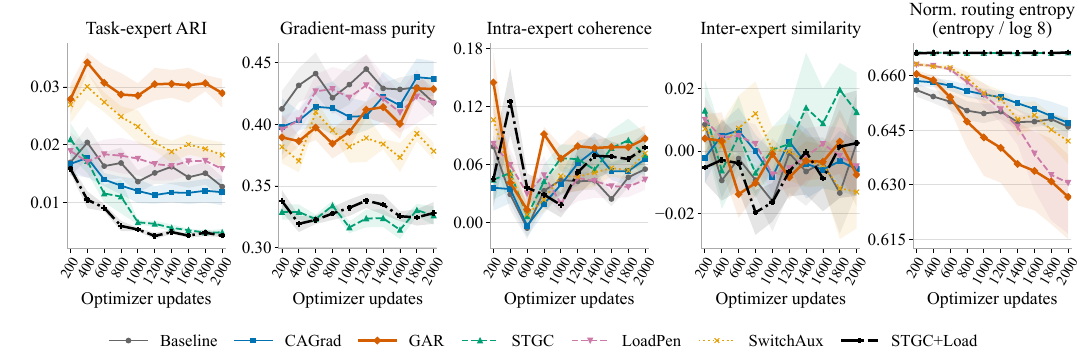}
\caption{Complementary routing diagnostics for all seven methods in the trainable RoBERTa classification-head setting, using the same five-mixture and checkpoint aggregation as Figure~\ref{fig:appendix_implb_mechanism_main}. Left to right: task--expert ARI, gradient-mass purity, intra-expert coherence, inter-expert similarity, and routing entropy normalized by $\log 8$. Bands are $\pm1$ standard error over the five run-level mixture means.}
\label{fig:appendix_implb_mechanism_extra}
\end{figure}

Figures~\ref{fig:appendix_implb_mechanism_main}
and~\ref{fig:appendix_implb_mechanism_extra} plot the routing diagnostics
at ten checkpoints, from 200 to 2,000 optimizer updates in increments of
200. The 2,000-update values are the final checkpoints summarized in the
endpoint tables. Figure~\ref{fig:appendix_implb_mechanism_main} includes all
seven methods over the five mixtures, with the same five-mixture aggregation
for every method. Figure~\ref{fig:appendix_implb_mechanism_extra}
reports the complementary diagnostics for all seven methods.
At each checkpoint, mixtures are first averaged within each run, and the
displayed mean and standard error are then computed across the five run-level
mixture means.

The first figure places the two usage axes and the two association axes side by
side. STGC and STGC+Load exhibit low load variance and high
utilization from the first checkpoint on. The STGC task--expert NMI and
structure purity decrease overall and remain below Baseline. At the endpoint,
STGC and STGC+Load lie within $0.008$ of the structure-purity lower bound of
Appendix~\ref{app:experimental_protocol} in every mixture, whereas GAR exceeds
it by $0.014$--$0.072$. GAR has lower
mean load variance and higher mean utilization, NMI, and structure purity than
Baseline and CAGrad at every recorded checkpoint. In the second figure, ARI
follows NMI. At the endpoint, GAR's gradient-mass purity exceeds that of STGC, SwitchAux,
and STGC+Load (Table~\ref{tab:appendix_implb_paired}). GAR has the highest mean
all-task-pair intra-expert coherence at the endpoint. The paired GAR$-$Baseline
difference is $+0.0315$ $[+0.0019,+0.0611]$, and the interval against LoadPen
also excludes zero, whereas the intervals against CAGrad, STGC, SwitchAux, and
STGC+Load include zero (Table~\ref{tab:appendix_implb_paired}). For
inter-expert similarity, all reported GAR--comparator paired intervals include
zero. Mean normalized routing entropy declines most for GAR. These trajectories
characterize the evolution of expert usage, gradient coherence, and task
association. For the five-task, six-task, and both seven-task mixtures, GAR and
Baseline select the same learning rate; clipping and weight decay are
inherited from Baseline in every mixture.

\subsection{Trainable RoBERTa Classification-Head Top-1}
\label{appsubsec:top1_head_results}

Table~\ref{tab:appendix_top1_head_results} reports Baseline, CAGrad, and GAR
in the trainable classification-head setting with top-1 straight-through
routing (E8K1; Appendix~\ref{appsubsec:top1_head_protocol}) on the five
mixtures and five seeds. The five-mixture mean accuracies are $80.39\%$,
$80.66\%$, and $81.45\%$; GAR has the highest mixture mean in all five
mixtures. Averaging the five mixtures within each seed, the paired
GAR$-$Baseline gain is $+1.06$ percentage points $[+0.83,+1.28]$ and the
GAR$-$CAGrad gain is $+0.79$ $[+0.30,+1.27]$; both task-count bins are
positive (Table~\ref{tab:appendix_top1_head_paired}).

Under one-hot dispatch, task-loss-only routing concentrates on a single expert:
Baseline's final utilization equals $1/8$ in 17 of its 25 runs, with LVar at
or near its maximum in the five- and seven-task mixtures containing PAWS.
GAR's utilization equals $1/8$ in 1 of its 25 runs, and its mixture means
correspond to 3.0--4.4 utilized experts; its paired LVar,
gradient-mass purity, and utilization differences against both Baseline and
CAGrad exclude zero. The accuracy ordering and the routing ordering therefore
agree in this setting.

\begin{table}[H]
\centering
\scriptsize
\caption{Trainable RoBERTa classification-head LoRA MoE with top-1 routing (E8K1), trained for 2,000 optimizer updates on each of the five mixtures. Columns report equal-task macro validation accuracy, its population seed standard deviation, LVar, gradient-mass purity, and utilization at the final evaluation, with metrics defined as in Appendix~\ref{app:experimental_protocol}. Five seeds per cell. For the aggregate block, each metric is first averaged equally over the five mixtures within each seed; the displayed mean and accuracy standard deviation are then computed across the five seeds.}\label{tab:appendix_top1_head_results}
\setlength{\tabcolsep}{5pt}
\renewcommand{\arraystretch}{1.03}
\begin{tabular}{@{}lrrrrr@{}}
\toprule
Method & Acc & Seed std & LVar & \shortstack{Gradient-mass\\purity} & Util. \\
\midrule
\multicolumn{6}{@{}l}{\textbf{[QNLI, BoolQ, RTE, PAWS, WiC]} \textnormal{(5 tasks)}} \\*
Baseline & 0.7864 & 0.0101 & 0.10806 & 0.1976 & 0.125 \\*
CAGrad & 0.7897 & 0.0081 & 0.03731 & 0.5230 & 0.375 \\*
\textbf{GAR} & 0.8038 & 0.0074 & 0.02377 & 0.5832 & 0.525 \\
\midrule
\multicolumn{6}{@{}l}{\textbf{[QNLI, BoolQ, RTE, PAWS, ANLI, CB]} \textnormal{(6 tasks)}} \\*
Baseline & 0.7818 & 0.0064 & 0.04267 & 0.2903 & 0.350 \\*
CAGrad & 0.7839 & 0.0094 & 0.04720 & 0.3746 & 0.350 \\*
\textbf{GAR} & 0.7908 & 0.0041 & 0.03245 & 0.5475 & 0.450 \\
\midrule
\multicolumn{6}{@{}l}{\textbf{[QNLI, BoolQ, RTE, PAWS, WiC, CoLA, SST-2]} \textnormal{(7 tasks)}} \\*
Baseline & 0.8157 & 0.0033 & 0.10936 & 0.0436 & 0.125 \\*
CAGrad & 0.8207 & 0.0054 & 0.03761 & 0.3748 & 0.375 \\*
\textbf{GAR} & 0.8249 & 0.0047 & 0.02986 & 0.3560 & 0.425 \\
\midrule
\multicolumn{6}{@{}l}{\textbf{[QNLI, BoolQ, RTE, WiC, CoLA, SST-2, MRPC]} \textnormal{(7 tasks)}} \\*
Baseline & 0.8105 & 0.0034 & 0.07550 & 0.0609 & 0.275 \\*
CAGrad & 0.8133 & 0.0025 & 0.02096 & 0.4103 & 0.425 \\*
\textbf{GAR} & 0.8216 & 0.0027 & 0.01853 & 0.5435 & 0.550 \\
\midrule
\multicolumn{6}{@{}l}{\textbf{[QNLI, BoolQ, RTE, PAWS, WiC, CoLA, SST-2, CB]} \textnormal{(8 tasks)}} \\*
Baseline & 0.8251 & 0.0090 & 0.09535 & 0.2086 & 0.200 \\*
CAGrad & 0.8253 & 0.0058 & 0.07846 & 0.3067 & 0.250 \\*
\textbf{GAR} & 0.8311 & 0.0048 & 0.05912 & 0.3694 & 0.375 \\
\midrule
\multicolumn{6}{@{}l}{\emph{Five-mixture aggregate}} \\*
Baseline & 0.8039 & 0.0036 & 0.08619 & 0.1602 & 0.215 \\*
CAGrad & 0.8066 & 0.0032 & 0.04431 & 0.3979 & 0.355 \\*
\textbf{GAR} & 0.8145 & 0.0024 & 0.03275 & 0.4799 & 0.465 \\
\bottomrule
\end{tabular}
\end{table}

\begin{table}[H]
\centering
\small
\caption{Paired final-evaluation differences for the trainable RoBERTa classification-head LoRA MoE with top-1 routing (E8K1; 2,000 optimizer updates). For each seed, metrics are averaged equally over the five mixtures before subtracting the comparator for that seed. The lower block repeats the accuracy comparison within the five--six- and seven--eight-task groups. Accuracy is in percentage points; LVar, gradient-mass purity, and utilization are in their native units. Brackets are unadjusted two-sided 95\% Student-$t$ confidence intervals over the five paired seed differences ($n=5$, four degrees of freedom).}
\label{tab:appendix_top1_head_paired}
\setlength{\tabcolsep}{5.5pt}
\begin{tabular}{@{}lcc@{}}
\toprule
Quantity & GAR $-$ Baseline & GAR $-$ CAGrad \\
\midrule
\multicolumn{3}{@{}l}{\emph{Five mixtures (five to eight tasks)}} \\
Accuracy (pp) & $+1.056$ [$+0.831$, $+1.282$] & $+0.786$ [$+0.300$, $+1.272$] \\
LVar & $-0.05344$ [$-0.07021$, $-0.03667$] & $-0.01156$ [$-0.02261$, $-0.00051$] \\
Gradient-mass purity & $+0.3197$ [$+0.2526$, $+0.3868$] & $+0.0821$ [$+0.0109$, $+0.1532$] \\
Utilization & $+0.2500$ [$+0.1543$, $+0.3457$] & $+0.1100$ [$+0.0629$, $+0.1571$] \\
\midrule
\multicolumn{3}{@{}l}{\emph{Accuracy by task-count bin (pp)}} \\
Five--six tasks & $+1.323$ [$+0.725$, $+1.921$] & $+1.049$ [$+0.441$, $+1.656$] \\
Seven--eight tasks & $+0.879$ [$+0.451$, $+1.307$] & $+0.611$ [$+0.094$, $+1.127$] \\
\bottomrule
\end{tabular}
\end{table}

\subsection{Trainable DeBERTa Full-Parameter FFN MoE}
\label{appsubsec:implc_results}

Table~\ref{tab:appendix_implc_results} reports Baseline, CAGrad, and GAR with a fully trainable DeBERTa backbone whose
experts are full-parameter feed-forward blocks at the final-layer insertion
sites, with token-level top-$k$ routing
(Appendix~\ref{appsubsec:implc_protocol}). Both this setting and the
classification-head setting of Appendix~\ref{appsubsec:implb_results} unfreeze
the backbone. This setting replaces LoRA with full-parameter FFN experts while
retaining the insertion sites and token-level routing; the classification-head
setting retains LoRA experts but moves them to the head and uses sequence-level
routing. This setting compares the three methods
that carry the main accuracy comparison. All five mixtures and the five seeds
used throughout are evaluated.

On the five mixtures, the equal-mixture
means are $0.8478$ for Baseline, $0.8507$ for CAGrad, and $0.8594$ for GAR.
At the 2,000-update budget of Appendix~\ref{appsubsec:implc_protocol}, the
paired GAR$-$Baseline difference is $+1.155$ $[+0.730,+1.579]$ points and
GAR$-$CAGrad is $+0.873$ $[+0.562,+1.183]$
(Table~\ref{tab:appendix_implc_paired}). Both intervals exclude zero, and the
GAR$-$Baseline gain is close to the frozen two-backbone gain of $+1.10$
$[+0.66,+1.53]$ and the classification-head gain of $+1.07$ $[+0.70,+1.43]$.
The full-parameter, trainable DeBERTa setting also shows an aggregate accuracy
gain. GAR has the highest accuracy in four of the five mixtures; the
exception is the six-task mixture, where CAGrad reaches $0.8534$ against
$0.8456$ for GAR and $0.8311$ for Baseline.

The aggregate LVar and utilization changes have the same directions as in the frozen LoRA-FFN aggregate.
GAR reduces LVar against both comparators and raises utilization by $+0.0700$
$[+0.0295,+0.1105]$ over Baseline and $+0.0900$ $[+0.0381,+0.1419]$ over
CAGrad, with all four intervals excluding zero. GAR's gradient-mass purity exceeds CAGrad's by
$+0.0165$ $[+0.0037,+0.0294]$ and Baseline's by $+0.0120$
$[-0.0135,+0.0375]$. By task-count bin, GAR$-$Baseline is $+1.048$ $[+0.292,+1.805]$
and $+1.226$ $[+0.969,+1.482]$ points for the five--six and seven--eight groups. The GAR$-$CAGrad interval in
the five--six bin, $+0.166$ $[-0.512,+0.843]$, includes zero and reflects the
six-task mixture above. These three-method comparisons use the
final-checkpoint values of each run. In the five- and six-task mixtures,
GAR and Baseline select the same learning rate (clipping and weight decay are
inherited from Baseline); the accuracy gains are $+0.65$ and $+1.45$ points, respectively.

\begin{table}[H]
\centering
\scriptsize
\caption{Trainable DeBERTa full-parameter FFN multi-task results for Baseline, CAGrad, and GAR across the five mixtures. Columns are equal-task macro validation accuracy, the population seed standard deviation, LVar, gradient-mass purity, and utilization, each defined as in Appendix~\ref{app:experimental_protocol} and read at the final evaluation. All five mixtures use E8K4, matching the grouping of Appendix~\ref{appsubsec:task_grouped_multitask_results}. Five seeds per cell.}\label{tab:appendix_implc_results}
\setlength{\tabcolsep}{5pt}
\renewcommand{\arraystretch}{1.03}
\begin{tabular}{@{}lrrrrr@{}}
\toprule
Method & Acc & Seed std & LVar & Purity & Util. \\
\midrule
\multicolumn{6}{@{}l}{\textbf{[QNLI, BoolQ, RTE, PAWS, WiC]} \textnormal{(5 tasks)}} \\*
Baseline & 0.8464 & 0.0027 & 0.00365 & 0.4015 & 0.825 \\*
CAGrad & 0.8418 & 0.0059 & 0.00393 & 0.4208 & 0.825 \\*
\textbf{GAR} & 0.8528 & 0.0043 & 0.00284 & 0.4102 & 0.900 \\
\addlinespace[2pt]
\multicolumn{6}{@{}l}{\textbf{[QNLI, BoolQ, RTE, PAWS, ANLI, CB]} \textnormal{(6 tasks)}} \\*
Baseline & 0.8311 & 0.0110 & 0.00325 & 0.3537 & 0.875 \\*
CAGrad & 0.8534 & 0.0036 & 0.00444 & 0.4117 & 0.775 \\*
\textbf{GAR} & 0.8456 & 0.0067 & 0.00263 & 0.3569 & 0.925 \\
\addlinespace[2pt]
\multicolumn{6}{@{}l}{\textbf{[QNLI, BoolQ, RTE, PAWS, WiC, CoLA, SST-2]} \textnormal{(7 tasks)}} \\*
Baseline & 0.8599 & 0.0026 & 0.00242 & 0.3706 & 0.875 \\*
CAGrad & 0.8574 & 0.0039 & 0.00239 & 0.3290 & 0.925 \\*
\textbf{GAR} & 0.8706 & 0.0035 & 0.00170 & 0.3916 & 0.925 \\
\addlinespace[2pt]
\multicolumn{6}{@{}l}{\textbf{[QNLI, BoolQ, RTE, WiC, CoLA, SST-2, MRPC]} \textnormal{(7 tasks)}} \\*
Baseline & 0.8408 & 0.0046 & 0.00310 & 0.3333 & 0.875 \\*
CAGrad & 0.8293 & 0.0012 & 0.00484 & 0.3040 & 0.775 \\*
\textbf{GAR} & 0.8543 & 0.0050 & 0.00261 & 0.3378 & 0.925 \\
\addlinespace[2pt]
\multicolumn{6}{@{}l}{\textbf{[QNLI, BoolQ, RTE, PAWS, WiC, CoLA, SST-2, CB]} \textnormal{(8 tasks)}} \\*
Baseline & 0.8610 & 0.0073 & 0.00287 & 0.3120 & 0.850 \\*
CAGrad & 0.8714 & 0.0054 & 0.00254 & 0.2832 & 0.900 \\*
\textbf{GAR} & 0.8736 & 0.0076 & 0.00185 & 0.3347 & 0.975 \\
\bottomrule
\end{tabular}
\end{table}

\begin{table}[H]
\centering
\small
\caption{Paired endpoint differences in the trainable DeBERTa full-parameter FFN setting. The upper block covers the five mixtures (five to eight tasks): within each seed we average the quantity over those five mixtures and then subtract the comparator for the same seed. The lower block repeats the accuracy difference within the five--six- and seven--eight-task groups. Accuracy is in percentage points; LVar, gradient-mass purity, and utilization are in their native units. Brackets are two-sided 95\% Student-$t$ confidence intervals over the five paired seed differences ($n=5$, four degrees of freedom). No multiplicity correction is applied.}
\label{tab:appendix_implc_paired}
\setlength{\tabcolsep}{5.5pt}
\begin{tabular}{@{}lcc@{}}
\toprule
Quantity & GAR $-$ Baseline & GAR $-$ CAGrad \\
\midrule
\multicolumn{3}{@{}l}{\emph{Five mixtures (five to eight tasks)}} \\
Accuracy (pp) & $+1.155$ [$+0.730$, $+1.579$] & $+0.873$ [$+0.562$, $+1.183$] \\
LVar & $-0.00073$ [$-0.00145$, $-0.00001$] & $-0.00130$ [$-0.00205$, $-0.00056$] \\
Gradient-mass purity & $+0.0120$ [$-0.0135$, $+0.0375$] & $+0.0165$ [$+0.0037$, $+0.0294$] \\
Utilization & $+0.0700$ [$+0.0295$, $+0.1105$] & $+0.0900$ [$+0.0381$, $+0.1419$] \\
\midrule
\multicolumn{3}{@{}l}{\emph{Accuracy by task-count bin (pp)}} \\
Five--six tasks & $+1.048$ [$+0.292$, $+1.805$] & $+0.166$ [$-0.512$, $+0.843$] \\
Seven--eight tasks & $+1.226$ [$+0.969$, $+1.482$] & $+1.344$ [$+0.781$, $+1.906$] \\
\bottomrule
\end{tabular}
\end{table}

\subsection{Controlled DeBERTa Single-Task Checks}
\label{appsubsec:controlled_deberta_checks}

Table~\ref{tab:single_task_endpoint_summary} summarizes the DeBERTa single-task
checks; Tables~\ref{tab:appendix_empirical_4} and~\ref{tab:appendix_empirical_5}
give the LoRA and FFN runs referenced from the main text.
All methods share the batch construction and routed forward
computation. CAGrad combines loss gradients from distinct micro-batches of the
sole task, serving as a within-task gradient-combination control; update rules are detailed in
Appendix~\ref{app:hyperparameters_reproducibility}. Purity is $1.0$ in the reported
single-task rows. With one task, each expert with nonzero diagnostic gradient
mass contributes $1$ and a zero-mass expert contributes $0$; the expert mean
equals $1$ when all experts have nonzero mass. These rows evaluate
single-task adaptation.

\noindent
All DeBERTa single-task results in the next two subsections use 3 epochs for
SST-2 and QQP, and 5 epochs for CoLA, MRPC, and RTE.
Within each task, CAGrad combines multiple micro-batch gradients, while
GAR uses the corresponding gradient observations to train the router.

\begin{table}[H]
\centering\small
\caption{DeBERTa single-task E4K2 endpoint summary: frozen-backbone LoRA-FFN versus full-parameter FFN experts with an unfrozen backbone. Each adaptation averages [CoLA, MRPC, RTE, SST-2, QQP]; Mean weights the two adaptations equally. Accuracy excludes MCC/F1. Seed std averages per-dataset population standard deviations; LVar and utilization average endpoint diagnostics. Purity is 1.0 in the reported single-task rows and is omitted.}
\label{tab:single_task_endpoint_summary}
\setlength{\tabcolsep}{6pt}
\begin{tabular}{@{}llrrrr@{}}
\toprule
Adapt. & Method & Acc $\uparrow$ & Seed std $\downarrow$ & LVar $\downarrow$ & Util. $\uparrow$ \\
\midrule
LoRA & Baseline & 0.8783 & 0.0081 & 0.0300 & 0.690 \\
     & CAGrad   & 0.8802 & 0.0110 & 0.0324 & \textbf{0.720} \\
     & GAR     & \textbf{0.8896} & \textbf{0.0065} & \textbf{0.0274} & 0.710 \\
\midrule
FFN  & Baseline & 0.8857 & 0.0090 & 0.0700 & 0.530 \\
     & CAGrad   & 0.8797 & 0.0103 & 0.0847 & 0.460 \\
     & GAR     & \textbf{0.8929} & \textbf{0.0075} & \textbf{0.0653} & \textbf{0.540} \\
\midrule
Mean & Baseline & 0.8820 & 0.0086 & 0.0500 & 0.610 \\
     & CAGrad   & 0.8799 & 0.0107 & 0.0586 & 0.590 \\
     & GAR     & \textbf{0.8913} & \textbf{0.0070} & \textbf{0.0463} & \textbf{0.625} \\
\bottomrule
\end{tabular}
\end{table}

\subsubsection{DeBERTa Single-Task LoRA Results}
\label{appsubsec:deberta_single_task_lora_results}
\begin{table}[H]
\centering
\small
\caption{DeBERTa single-task frozen-backbone LoRA-FFN results with E4K2 routing (SST-2 and QQP use 3 epochs; CoLA, MRPC, and RTE use 5 epochs). CoLA additionally reports MCC, while MRPC and QQP additionally report F1. LVar is the population variance of the normalized final expert-load shares, as defined in Appendix~\ref{app:experimental_protocol}. Purity is gradient-mass purity, which equals 1 with a single task.}
\label{tab:appendix_empirical_4}
\AppendixTable{%
\begin{tabular}{p{0.16\textwidth}lcccccc}
\toprule
Setting & Method & Final val acc $\uparrow$ & Task metric $\uparrow$ & Seed std $\downarrow$ & LVar $\downarrow$ & Purity $\uparrow$ & Utilization $\uparrow$ \\
\midrule
CoLA & Baseline & 0.8619 & 0.6664 (MCC) & 0.0126 & 0.0366 & 1.0000 & 0.6500 \\
 & CAGrad & 0.8656 & 0.6758 (MCC) & 0.0121 & 0.0308 & 1.0000 & 0.7000 \\
 & GAR & 0.8667 & 0.6786 (MCC) & 0.0099 & 0.0323 & 1.0000 & 0.6000 \\
\midrule
MRPC & Baseline & 0.8686 & 0.9042 (F1) & 0.0061 & 0.0412 & 1.0000 & 0.6000 \\
 & CAGrad & 0.8608 & 0.8995 (F1) & 0.0196 & 0.0520 & 1.0000 & 0.5500 \\
 & GAR & 0.8882 & 0.9193 (F1) & 0.0053 & 0.0247 & 1.0000 & 0.7500 \\
\midrule
QQP & Baseline & 0.9080 & 0.8778 (F1) & 0.0018 & 0.0105 & 1.0000 & 0.8500 \\
 & CAGrad & 0.8959 & 0.8590 (F1) & 0.0038 & 0.0133 & 1.0000 & 0.9000 \\
 & GAR & 0.9219 & 0.8964 (F1) & 0.0011 & 0.0152 & 1.0000 & 0.8500 \\
\midrule
RTE & Baseline & 0.8152 & -- & 0.0110 & 0.0377 & 1.0000 & 0.6500 \\
 & CAGrad & 0.8361 & -- & 0.0132 & 0.0295 & 1.0000 & 0.7000 \\
 & GAR & 0.8253 & -- & 0.0122 & 0.0311 & 1.0000 & 0.7000 \\
\midrule
SST-2 & Baseline & 0.9376 & -- & 0.0090 & 0.0242 & 1.0000 & 0.7000 \\
 & CAGrad & 0.9424 & -- & 0.0062 & 0.0362 & 1.0000 & 0.7500 \\
 & GAR & 0.9459 & -- & 0.0039 & 0.0336 & 1.0000 & 0.6500 \\
\bottomrule
\end{tabular}
}
\end{table}

\subsubsection{DeBERTa Single-Task FFN Results}
\label{appsubsec:deberta_single_task_ffn_results}
\begin{table}[H]
\centering
\small
\caption{DeBERTa single-task full-parameter FFN results with an unfrozen backbone and E4K2 routing (SST-2 and QQP use 3 epochs; CoLA, MRPC, and RTE use 5 epochs). Experts occupy the same final-layer insertion sites as in LoRA-FFN, without LoRA factorization. CoLA reports MCC; MRPC and QQP report F1. LVar is the population variance of normalized final expert-load shares, as defined in Appendix~\ref{app:experimental_protocol}. Purity is gradient-mass purity, which equals 1 with a single task.}
\label{tab:appendix_empirical_5}
\AppendixTable{%
\begin{tabular}{p{0.16\textwidth}lcccccc}
\toprule
Setting & Method & Final val acc $\uparrow$ & Task metric $\uparrow$ & Seed std $\downarrow$ & LVar $\downarrow$ & Purity $\uparrow$ & Utilization $\uparrow$ \\
\midrule
CoLA & Baseline & 0.8639 & 0.6711 (MCC) & 0.0114 & 0.0942 & 1.0000 & 0.5000 \\
 & CAGrad & 0.8577 & 0.6554 (MCC) & 0.0080 & 0.1181 & 1.0000 & 0.4000 \\
 & GAR & 0.8658 & 0.6764 (MCC) & 0.0105 & 0.0658 & 1.0000 & 0.5000 \\
\midrule
MRPC & Baseline & 0.8648 & 0.9021 (F1) & 0.0108 & 0.1054 & 1.0000 & 0.3000 \\
 & CAGrad & 0.8775 & 0.9103 (F1) & 0.0103 & 0.0792 & 1.0000 & 0.4000 \\
 & GAR & 0.8946 & 0.9244 (F1) & 0.0069 & 0.0878 & 1.0000 & 0.3500 \\
\midrule
QQP & Baseline & 0.9202 & 0.8950 (F1) & 0.0023 & 0.0156 & 1.0000 & 0.9000 \\
 & CAGrad & 0.8985 & 0.8656 (F1) & 0.0019 & 0.0423 & 1.0000 & 0.6500 \\
 & GAR & 0.9212 & 0.8950 (F1) & 0.0018 & 0.0220 & 1.0000 & 0.8000 \\
\midrule
RTE & Baseline & 0.8368 & -- & 0.0185 & 0.0699 & 1.0000 & 0.4500 \\
 & CAGrad & 0.8354 & -- & 0.0186 & 0.0680 & 1.0000 & 0.4500 \\
 & GAR & 0.8347 & -- & 0.0150 & 0.0828 & 1.0000 & 0.5000 \\
\midrule
SST-2 & Baseline & 0.9427 & -- & 0.0021 & 0.0649 & 1.0000 & 0.5000 \\
 & CAGrad & 0.9294 & -- & 0.0127 & 0.1160 & 1.0000 & 0.4000 \\
 & GAR & 0.9484 & -- & 0.0032 & 0.0681 & 1.0000 & 0.5500 \\
\bottomrule
\end{tabular}
}
\end{table}

\subsubsection{Single-Task Equal-Adaptation Paired Uncertainty}
\label{appsubsec:single_task_paired_uncertainty}

Table~\ref{tab:single_task_paired_uncertainty} reports paired uncertainty
after averaging the two adaptations equally within each seed.

\begin{table}[H]
\centering
\small
\caption{Equal-adaptation paired accuracy uncertainty for the controlled DeBERTa single-task checks. Within each shared seed, we first average validation accuracy over CoLA, MRPC, RTE, SST-2, and QQP separately for LoRA and FFN, then average the two adaptation results equally, and finally subtract the matched comparator from GAR. Brackets are two-sided 95\% Student-$t$ confidence intervals over the five paired seed differences ($n=5$, four degrees of freedom); $p$ values are from the corresponding two-sided paired $t$ tests. No multiplicity correction or significance symbols are used.}
\label{tab:single_task_paired_uncertainty}
\setlength{\tabcolsep}{7pt}
\begin{tabular}{@{}lccc@{}}
\toprule
Comparison & Mean difference (pp) & 95\% paired CI (pp) & Paired $t$ $p$ \\
\midrule
GAR $-$ Baseline & $+0.929$ & $[+0.207,\,+1.651]$ & $0.0234$ \\
GAR $-$ CAGrad & $+1.133$ & $[+0.420,\,+1.847]$ & $0.0116$ \\
\bottomrule
\end{tabular}
\end{table}

\subsection{RoBERTa Fixed-Configuration Coefficient Ablations}
\label{appsubsec:lambda_ablation_results}

Both RoBERTa coefficient sweeps cover the same five E8K4 dataset mixtures listed by their constituent tasks in the following tables. Each setting reports five
coefficients over five shared random seeds at 1,000 updates (125 endpoints
per setting). Each setting and mixture reuses its selected GAR configuration,
except for the shorter budget and the swept
coefficient, as specified in Appendix~\ref{appsubsec:lambda_ablation_protocol}.
Mixture macros average tasks equally; the five-mixture summaries average
mixtures within each seed before computing the mean and population standard
deviation. The two settings are reported separately.

\paragraph{Frozen-backbone LoRA-FFN.}
Table~\ref{tab:appendix_ffn_lambda_mixtures} summarizes mixture accuracy.
Table~\ref{tab:appendix_ffn_lambda_tasks} gives all 33 mixture--task results,
macro accuracies, seed variability, and routing diagnostics.
The five-mixture mean rises from $66.45\%$ at $\lambda=0$ to $67.21\%$
at $10^{-3}$, then falls to $66.74\%$ at $10^{-2}$. All five mixture means
increase at $10^{-3}$ relative to zero, but task responses differ: RTE improves
in all five mixtures, whereas WiC declines in both seven-task mixtures and the eight-task mixture.
[QNLI, BoolQ, RTE, WiC, CoLA, SST-2, MRPC] has its highest mixture mean at $10^{-4}$ rather than $10^{-3}$.

\begin{table}[H]
\centering
\small
\setlength{\tabcolsep}{4.5pt}
\caption{RoBERTa LoRA-FFN coefficient sensitivity on the five E8K4 mixtures at 1,000 updates. Entries are accuracy (\%) as mean $\pm$ population standard deviation over five seeds. Tasks are equally weighted within each mixture. The five-mixture row averages mixtures within each seed before computing the mean and standard deviation. Fixed configurations are in Table~\ref{tab:roberta_lambda_configs}.}
\label{tab:appendix_ffn_lambda_mixtures}
\begin{tabular}{@{}p{0.27\textwidth}ccccc@{}}
\toprule
Mixture & $\lambda=0$ & $10^{-5}$ & $10^{-4}$ & $10^{-3}$ & $10^{-2}$ \\
\midrule
{}[QNLI, BoolQ, RTE, PAWS, WiC] & $60.63\pm1.18$ & $61.53\pm0.74$ & $61.84\pm1.07$ & $62.18\pm0.62$ & $62.08\pm0.67$ \\
{}[QNLI, BoolQ, RTE, PAWS, ANLI, CB] & $64.54\pm0.58$ & $64.87\pm0.59$ & $65.22\pm0.73$ & $65.38\pm0.69$ & $64.89\pm0.69$ \\
{}[QNLI, BoolQ, RTE, PAWS, WiC, CoLA, \mbox{SST-2}] & $66.92\pm0.70$ & $67.17\pm0.50$ & $67.53\pm0.43$ & $67.63\pm0.71$ & $66.92\pm0.61$ \\
{}[QNLI, BoolQ, RTE, WiC, CoLA, \mbox{SST-2}, MRPC] & $68.86\pm0.58$ & $68.90\pm0.55$ & $69.34\pm0.54$ & $69.23\pm0.43$ & $68.84\pm0.78$ \\
{}[QNLI, BoolQ, RTE, PAWS, WiC, CoLA, \mbox{SST-2}, CB] & $71.29\pm0.50$ & $71.56\pm0.65$ & $71.42\pm0.45$ & $71.62\pm0.45$ & $70.97\pm0.88$ \\
\midrule
Five-mixture mean & $66.45\pm0.37$ & $66.81\pm0.37$ & $67.07\pm0.26$ & $67.21\pm0.36$ & $66.74\pm0.17$ \\
\bottomrule
\end{tabular}
\end{table}

\begingroup
\small
\setcounter{LTchunksize}{200}
\setlength{\tabcolsep}{4.5pt}
\begin{longtable}{@{}lrrrrr@{}}
\caption{Per-task results and routing diagnostics for RoBERTa LoRA-FFN coefficient sensitivity on the five E8K4 mixtures. Task and mixture macro accuracies are five-seed means in percent; seed std is the population standard deviation of mixture macro accuracy in percentage points. LVar, gradient-mass purity, and utilization are five-seed means at 1,000 updates on their original scales. All 33 mixture--task combinations and five coefficients are reported. Mixture definitions and fixed configurations are in Table~\ref{tab:roberta_lambda_configs}.}
\label{tab:appendix_ffn_lambda_tasks}\\
\toprule
Task / metric & $\lambda=0$ & $10^{-5}$ & $10^{-4}$ & $10^{-3}$ & $10^{-2}$ \\
\midrule
\endfirsthead
\caption[]{RoBERTa LoRA-FFN coefficient sensitivity: per-task results and diagnostics (continued).}\\
\toprule
Task / metric & $\lambda=0$ & $10^{-5}$ & $10^{-4}$ & $10^{-3}$ & $10^{-2}$ \\
\midrule
\endhead
\bottomrule
\endfoot
\multicolumn{6}{@{}l}{\textbf{[QNLI, BoolQ, RTE, PAWS, WiC]} \textnormal{(5 tasks)}} \\*
QNLI & 77.04 & 78.54 & 79.02 & 79.84 & 78.89 \\*
BoolQ & 61.36 & 61.85 & 62.13 & 61.90 & 61.41 \\*
RTE & 54.15 & 55.74 & 55.88 & 57.18 & 56.61 \\*
PAWS & 56.36 & 56.83 & 57.00 & 57.26 & 59.14 \\*
WiC & 54.26 & 54.67 & 55.14 & 54.73 & 54.36 \\*
\cmidrule(lr){1-6}
\emph{Macro} & 60.63 & 61.53 & 61.84 & 62.18 & 62.08 \\*
Seed std (pp) & 1.18 & 0.74 & 1.07 & 0.62 & 0.67 \\*
LVar & 0.03960 & 0.03644 & 0.03657 & 0.01654 & 0.01727 \\*
Purity & 0.5205 & 0.4770 & 0.5002 & 0.5329 & 0.6185 \\*
Util. & 0.425 & 0.500 & 0.500 & 0.600 & 0.475 \\
\midrule
\multicolumn{6}{@{}l}{\textbf{[QNLI, BoolQ, RTE, PAWS, ANLI, CB]} \textnormal{(6 tasks)}} \\*
QNLI & 80.30 & 80.07 & 80.56 & 81.30 & 80.19 \\*
BoolQ & 64.43 & 64.28 & 64.35 & 64.54 & 64.03 \\*
RTE & 62.31 & 63.10 & 65.13 & 65.05 & 63.47 \\*
PAWS & 59.55 & 59.45 & 60.32 & 60.14 & 60.79 \\*
ANLI & 35.28 & 35.56 & 35.61 & 35.53 & 35.52 \\*
CB & 85.36 & 86.79 & 85.36 & 85.71 & 85.36 \\*
\cmidrule(lr){1-6}
\emph{Macro} & 64.54 & 64.87 & 65.22 & 65.38 & 64.89 \\*
Seed std (pp) & 0.58 & 0.59 & 0.73 & 0.69 & 0.69 \\*
LVar & 0.01979 & 0.01770 & 0.01823 & 0.01594 & 0.01526 \\*
Purity & 0.4962 & 0.5048 & 0.4583 & 0.4989 & 0.5317 \\*
Util. & 0.550 & 0.525 & 0.525 & 0.600 & 0.725 \\
\midrule
\multicolumn{6}{@{}l}{\textbf{[QNLI, BoolQ, RTE, PAWS, WiC, CoLA, SST-2]} \textnormal{(7 tasks)}} \\*
QNLI & 76.15 & 76.53 & 76.53 & 77.42 & 74.06 \\*
BoolQ & 60.75 & 61.39 & 60.83 & 60.82 & 61.33 \\*
RTE & 56.03 & 57.55 & 58.70 & 59.49 & 57.83 \\*
PAWS & 57.05 & 57.20 & 57.43 & 57.13 & 55.87 \\*
WiC & 54.36 & 53.57 & 54.23 & 54.14 & 53.86 \\*
CoLA & 76.64 & 76.18 & 76.78 & 76.11 & 77.53 \\*
SST-2 & 87.43 & 87.78 & 88.19 & 88.30 & 87.98 \\*
\cmidrule(lr){1-6}
\emph{Macro} & 66.92 & 67.17 & 67.53 & 67.63 & 66.92 \\*
Seed std (pp) & 0.70 & 0.50 & 0.43 & 0.71 & 0.61 \\*
LVar & 0.06611 & 0.05918 & 0.06653 & 0.05193 & 0.05728 \\*
Purity & 0.5075 & 0.5161 & 0.5001 & 0.5311 & 0.7129 \\*
Util. & 0.250 & 0.275 & 0.275 & 0.325 & 0.225 \\
\midrule
\multicolumn{6}{@{}l}{\textbf{[QNLI, BoolQ, RTE, WiC, CoLA, SST-2, MRPC]} \textnormal{(7 tasks)}} \\*
QNLI & 75.11 & 75.48 & 75.62 & 75.79 & 74.05 \\*
BoolQ & 61.00 & 60.89 & 61.49 & 61.09 & 61.25 \\*
RTE & 56.53 & 55.96 & 56.97 & 58.34 & 54.22 \\*
WiC & 53.01 & 52.79 & 53.54 & 52.41 & 54.45 \\*
CoLA & 76.72 & 76.68 & 76.93 & 76.01 & 76.64 \\*
SST-2 & 86.58 & 86.88 & 87.27 & 87.64 & 86.93 \\*
MRPC & 73.04 & 73.63 & 73.53 & 73.33 & 74.31 \\*
\cmidrule(lr){1-6}
\emph{Macro} & 68.86 & 68.90 & 69.34 & 69.23 & 68.84 \\*
Seed std (pp) & 0.58 & 0.55 & 0.54 & 0.43 & 0.78 \\*
LVar & 0.02453 & 0.02193 & 0.02358 & 0.01018 & 0.01376 \\*
Purity & 0.4032 & 0.4478 & 0.4037 & 0.4239 & 0.6630 \\*
Util. & 0.525 & 0.475 & 0.500 & 0.725 & 0.750 \\
\midrule
\multicolumn{6}{@{}l}{\textbf{[QNLI, BoolQ, RTE, PAWS, WiC, CoLA, SST-2, CB]} \textnormal{(8 tasks)}} \\*
QNLI & 79.69 & 79.60 & 79.87 & 79.79 & 77.77 \\*
BoolQ & 63.86 & 64.21 & 64.22 & 63.71 & 62.50 \\*
RTE & 59.93 & 61.95 & 60.72 & 63.97 & 63.54 \\*
PAWS & 59.26 & 59.54 & 58.94 & 59.06 & 58.98 \\*
WiC & 56.93 & 56.58 & 56.74 & 56.24 & 56.87 \\*
CoLA & 77.56 & 77.56 & 77.43 & 77.79 & 77.14 \\*
SST-2 & 89.17 & 89.43 & 89.52 & 89.20 & 88.44 \\*
CB & 83.93 & 83.57 & 83.93 & 83.21 & 82.50 \\*
\cmidrule(lr){1-6}
\emph{Macro} & 71.29 & 71.56 & 71.42 & 71.62 & 70.97 \\*
Seed std (pp) & 0.50 & 0.65 & 0.45 & 0.45 & 0.88 \\*
LVar & 0.01455 & 0.01361 & 0.01218 & 0.00917 & 0.01732 \\*
Purity & 0.4211 & 0.4448 & 0.4391 & 0.4562 & 0.5694 \\*
Util. & 0.525 & 0.575 & 0.600 & 0.750 & 0.525 \\
\end{longtable}
\endgroup

\paragraph{Trainable classification-head LoRA.}
Tables~\ref{tab:appendix_classifier_lambda_mixtures}
and~\ref{tab:appendix_classifier_lambda_tasks} give the corresponding five
mixture summaries and all 33 mixture--task combinations. The five-mixture
mean is highest at $10^{-3}$ among the tested coefficients ($79.87\%$,
versus $78.97\%$ at zero). [QNLI, BoolQ, RTE, PAWS, WiC] at $\lambda=0$ has one low-accuracy seed
(about $73.02\%$), producing a population standard deviation of $2.16$
percentage points.

\begin{table}[H]
\centering
\small
\setlength{\tabcolsep}{4.5pt}
\caption{RoBERTa classification-head LoRA coefficient sensitivity on the five E8K4 mixtures at 1,000 updates. Entries are accuracy (\%) as mean $\pm$ population standard deviation over five seeds. Tasks are equally weighted within each mixture. The five-mixture row averages mixtures within each seed before computing the mean and standard deviation. Fixed configurations are in Table~\ref{tab:roberta_classifier_lambda_configs}.}
\label{tab:appendix_classifier_lambda_mixtures}
\begin{tabular}{@{}p{0.27\textwidth}ccccc@{}}
\toprule
Mixture & $\lambda=0$ & $10^{-5}$ & $10^{-4}$ & $10^{-3}$ & $10^{-2}$ \\
\midrule
{}[QNLI, BoolQ, RTE, PAWS, WiC] & $76.86\pm2.16$ & $78.03\pm1.35$ & $77.87\pm1.36$ & $78.91\pm0.92$ & $77.64\pm1.46$ \\
{}[QNLI, BoolQ, RTE, PAWS, ANLI, CB] & $76.39\pm1.13$ & $76.31\pm0.88$ & $76.77\pm0.63$ & $77.14\pm0.81$ & $77.05\pm1.03$ \\
{}[QNLI, BoolQ, RTE, PAWS, WiC, CoLA, \mbox{SST-2}] & $80.73\pm0.60$ & $80.82\pm0.55$ & $81.20\pm0.83$ & $81.01\pm0.90$ & $81.11\pm0.73$ \\
{}[QNLI, BoolQ, RTE, WiC, CoLA, \mbox{SST-2}, MRPC] & $80.11\pm0.14$ & $80.31\pm0.40$ & $80.50\pm0.36$ & $80.90\pm0.10$ & $80.66\pm0.36$ \\
{}[QNLI, BoolQ, RTE, PAWS, WiC, CoLA, \mbox{SST-2}, CB] & $80.76\pm0.77$ & $80.38\pm0.76$ & $80.96\pm0.61$ & $81.37\pm0.94$ & $80.84\pm0.92$ \\
\midrule
Five-mixture mean & $78.97\pm0.78$ & $79.17\pm0.57$ & $79.46\pm0.52$ & $79.87\pm0.58$ & $79.46\pm0.70$ \\
\bottomrule
\end{tabular}
\end{table}

\begingroup
\small
\setcounter{LTchunksize}{200}
\setlength{\tabcolsep}{4.5pt}
\begin{longtable}{@{}lrrrrr@{}}
\caption{Per-task results and routing diagnostics for RoBERTa classification-head LoRA coefficient sensitivity on the five E8K4 mixtures. Task and mixture macro accuracies are five-seed means in percent; seed std is the population standard deviation of mixture macro accuracy in percentage points. LVar, gradient-mass purity, and utilization are five-seed means at 1,000 updates on their original scales. All 33 mixture--task combinations and five coefficients are reported. Mixture definitions and fixed configurations are in Table~\ref{tab:roberta_classifier_lambda_configs}.}
\label{tab:appendix_classifier_lambda_tasks}\\
\toprule
Task / metric & $\lambda=0$ & $10^{-5}$ & $10^{-4}$ & $10^{-3}$ & $10^{-2}$ \\
\midrule
\endfirsthead
\caption[]{RoBERTa classification-head LoRA coefficient sensitivity: per-task results and diagnostics (continued).}\\
\toprule
Task / metric & $\lambda=0$ & $10^{-5}$ & $10^{-4}$ & $10^{-3}$ & $10^{-2}$ \\
\midrule
\endhead
\bottomrule
\endfoot
\multicolumn{6}{@{}l}{\textbf{[QNLI, BoolQ, RTE, PAWS, WiC]} \textnormal{(5 tasks)}} \\*
QNLI & 86.12 & 87.99 & 88.10 & 88.15 & 87.93 \\*
BoolQ & 72.06 & 72.51 & 71.91 & 73.02 & 71.86 \\*
RTE & 71.05 & 72.71 & 72.85 & 73.57 & 72.64 \\*
PAWS & 91.05 & 91.34 & 91.42 & 91.75 & 90.53 \\*
WiC & 64.04 & 65.61 & 65.05 & 68.06 & 65.27 \\*
\cmidrule(lr){1-6}
\emph{Macro} & 76.86 & 78.03 & 77.87 & 78.91 & 77.64 \\*
Seed std (pp) & 2.16 & 1.35 & 1.36 & 0.92 & 1.46 \\*
LVar & 0.00959 & 0.00781 & 0.00862 & 0.00573 & 0.00118 \\*
Purity & 0.3862 & 0.4508 & 0.5118 & 0.4990 & 0.4627 \\*
Util. & 0.700 & 0.700 & 0.675 & 0.750 & 0.975 \\
\midrule
\multicolumn{6}{@{}l}{\textbf{[QNLI, BoolQ, RTE, PAWS, ANLI, CB]} \textnormal{(6 tasks)}} \\*
QNLI & 88.27 & 88.08 & 88.45 & 88.56 & 87.81 \\*
BoolQ & 73.43 & 73.76 & 73.43 & 74.57 & 75.19 \\*
RTE & 74.80 & 75.02 & 75.60 & 75.38 & 75.60 \\*
PAWS & 91.11 & 91.40 & 91.44 & 91.12 & 91.97 \\*
ANLI & 41.43 & 41.40 & 41.72 & 42.11 & 42.46 \\*
CB & 89.29 & 88.21 & 90.00 & 91.07 & 89.29 \\*
\cmidrule(lr){1-6}
\emph{Macro} & 76.39 & 76.31 & 76.77 & 77.14 & 77.05 \\*
Seed std (pp) & 1.13 & 0.88 & 0.63 & 0.81 & 1.03 \\*
LVar & 0.01109 & 0.00839 & 0.00716 & 0.00443 & 0.00340 \\*
Purity & 0.3858 & 0.4441 & 0.4630 & 0.4918 & 0.5920 \\*
Util. & 0.525 & 0.675 & 0.725 & 0.825 & 0.850 \\
\midrule
\multicolumn{6}{@{}l}{\textbf{[QNLI, BoolQ, RTE, PAWS, WiC, CoLA, SST-2]} \textnormal{(7 tasks)}} \\*
QNLI & 87.81 & 87.87 & 88.43 & 88.04 & 87.53 \\*
BoolQ & 72.90 & 73.36 & 73.50 & 73.18 & 73.46 \\*
RTE & 75.31 & 74.58 & 76.46 & 76.53 & 77.47 \\*
PAWS & 90.20 & 90.42 & 90.23 & 90.29 & 90.21 \\*
WiC & 64.76 & 65.14 & 64.80 & 64.64 & 64.17 \\*
CoLA & 81.82 & 82.09 & 82.09 & 81.69 & 82.07 \\*
SST-2 & 92.29 & 92.27 & 92.87 & 92.71 & 92.84 \\*
\cmidrule(lr){1-6}
\emph{Macro} & 80.73 & 80.82 & 81.20 & 81.01 & 81.11 \\*
Seed std (pp) & 0.60 & 0.55 & 0.83 & 0.90 & 0.73 \\*
LVar & 0.00720 & 0.00675 & 0.00687 & 0.00439 & 0.00173 \\*
Purity & 0.3989 & 0.4507 & 0.3797 & 0.3803 & 0.3999 \\*
Util. & 0.725 & 0.750 & 0.725 & 0.800 & 0.925 \\
\midrule
\multicolumn{6}{@{}l}{\textbf{[QNLI, BoolQ, RTE, WiC, CoLA, SST-2, MRPC]} \textnormal{(7 tasks)}} \\*
QNLI & 87.33 & 87.37 & 88.05 & 87.80 & 87.22 \\*
BoolQ & 72.70 & 72.71 & 72.04 & 72.75 & 72.64 \\*
RTE & 74.51 & 75.96 & 75.88 & 76.68 & 75.96 \\*
WiC & 65.74 & 65.45 & 65.08 & 66.55 & 66.18 \\*
CoLA & 81.32 & 80.96 & 81.67 & 81.76 & 81.84 \\*
SST-2 & 92.41 & 92.78 & 93.03 & 93.49 & 93.33 \\*
MRPC & 86.76 & 86.96 & 87.75 & 87.30 & 87.45 \\*
\cmidrule(lr){1-6}
\emph{Macro} & 80.11 & 80.31 & 80.50 & 80.90 & 80.66 \\*
Seed std (pp) & 0.14 & 0.40 & 0.36 & 0.10 & 0.36 \\*
LVar & 0.00674 & 0.00668 & 0.00500 & 0.00323 & 0.00285 \\*
Purity & 0.3936 & 0.3700 & 0.3329 & 0.3544 & 0.4134 \\*
Util. & 0.675 & 0.750 & 0.775 & 0.850 & 0.850 \\
\midrule
\multicolumn{6}{@{}l}{\textbf{[QNLI, BoolQ, RTE, PAWS, WiC, CoLA, SST-2, CB]} \textnormal{(8 tasks)}} \\*
QNLI & 87.28 & 87.61 & 87.72 & 87.59 & 87.41 \\*
BoolQ & 70.34 & 69.46 & 70.25 & 72.01 & 70.51 \\*
RTE & 74.01 & 74.44 & 75.67 & 75.09 & 74.73 \\*
PAWS & 89.49 & 88.50 & 89.22 & 89.33 & 89.14 \\*
WiC & 64.42 & 64.42 & 64.26 & 65.58 & 64.08 \\*
CoLA & 80.75 & 80.35 & 80.63 & 80.63 & 80.23 \\*
SST-2 & 92.27 & 92.16 & 92.41 & 92.87 & 92.73 \\*
CB & 87.50 & 86.07 & 87.50 & 87.86 & 87.86 \\*
\cmidrule(lr){1-6}
\emph{Macro} & 80.76 & 80.38 & 80.96 & 81.37 & 80.84 \\*
Seed std (pp) & 0.77 & 0.76 & 0.61 & 0.94 & 0.92 \\*
LVar & 0.00714 & 0.00763 & 0.00665 & 0.00342 & 0.00166 \\*
Purity & 0.4538 & 0.4586 & 0.4111 & 0.4235 & 0.5039 \\*
Util. & 0.725 & 0.700 & 0.700 & 0.850 & 0.925 \\
\end{longtable}
\endgroup

The classification-head denominator ablation compares $\lambda=0$, the
load-normalized objective at $\lambda=10^{-3}$, and a numerator-only objective at
$\lambda=10^{-3}$. The first two are the coefficient-sweep entries above;
Table~\ref{tab:appendix_classifier_loadnorm_ablation} reports the numerator-only objective. Averaging the five
mixtures within each seed, the numerator-only objective reaches $79.39\%$: its paired
difference from $\lambda=0$ is $+0.42$ $[-0.40,+1.25]$ points, and the
load-normalized objective exceeds it by $+0.47$ $[+0.08,+0.87]$ points
($n=5$, four degrees of freedom).

\begin{table}[H]
\centering
\small
\setlength{\tabcolsep}{4.2pt}
\caption{Numerator-only RoBERTa classification-head LoRA results on the five E8K4 mixtures at 1,000 updates. Each row inherits the corresponding fixed $\lambda=10^{-3}$ configuration in Table~\ref{tab:roberta_classifier_lambda_configs} and removes only the load denominator. Accuracy is in percent and its seed standard deviation in percentage points; Purity is gradient-mass purity; routing diagnostics are five-run means on their original scales. The $\lambda=0$ and load-normalized results appear in Table~\ref{tab:appendix_classifier_lambda_tasks}.}
\label{tab:appendix_classifier_loadnorm_ablation}
\begin{tabular}{@{}p{0.43\textwidth}rrrrr@{}}
\toprule
Mixture & Acc $\uparrow$ & Seed std $\downarrow$ & LVar $\downarrow$ & Purity $\uparrow$ & Util. $\uparrow$ \\
\midrule
{}[QNLI, BoolQ, RTE, PAWS, WiC] & 78.98 & 1.07 & 0.00561 & 0.5020 & 0.750 \\
{}[QNLI, BoolQ, RTE, PAWS, ANLI, CB] & 76.95 & 0.95 & 0.00615 & 0.5586 & 0.775 \\
{}[QNLI, BoolQ, RTE, PAWS, WiC, CoLA, \mbox{SST-2}] & 80.46 & 0.71 & 0.00605 & 0.3964 & 0.700 \\
{}[QNLI, BoolQ, RTE, WiC, CoLA, \mbox{SST-2}, MRPC] & 80.17 & 0.56 & 0.00692 & 0.3761 & 0.750 \\
{}[QNLI, BoolQ, RTE, PAWS, WiC, CoLA, \mbox{SST-2}, CB] & 80.41 & 0.42 & 0.00602 & 0.6280 & 0.725 \\
\bottomrule
\end{tabular}
\end{table}

For a paired comparison of $10^{-3}$ against zero, we first
average the five mixtures within each seed and then compute the five matched
seed differences. The LoRA-FFN gain is $+0.76$ percentage points with a
two-sided 95\% Student-$t$ interval $[-0.02,+1.54]$; the classification-head
gain is $+0.90$ with interval $[+0.43,+1.37]$ ($n=5$, four degrees of freedom,
sample standard deviations for the intervals). The LoRA-FFN interval includes zero; the classification-head interval is positive.
These intervals are unadjusted for the coefficient sweep.

The ablations characterize the coefficient response of the complete
auxiliary training rule at fixed configurations and a 1,000-update budget.

\section{Training Flow Schematic}
\label{app:training_flow}

\begin{figure}[H]
\centering
\begin{tikzpicture}[
scale=0.7, transform shape,
node distance=6mm and 15mm,
box/.style={draw, rounded corners, align=center, minimum width=3.5cm, minimum height=8mm},
smallbox/.style={draw, rounded corners, align=center, minimum width=3cm, minimum height=7mm},
arrow/.style={-Stealth, thick}
]

\node[box] (data1) {Multi-task Dataloaders};
\node[box, below=of data1] (sample1) {Draw Matched Task Batches};
\node[box, below=of sample1] (batch1) {Mixed Training Batch};
\node[box, below=of batch1] (router1) {Per-Unit Router Probabilities};
\node[box, below=of router1] (experts1) {Configured Expert Dispatch};
\node[box, below=of experts1] (task1) {Weighted Task Loss $\mathcal{L}_{\text{task}}$};

\node[smallbox, right=35mm of batch1] (grad1) {Detached Template Gradients $\tilde g_m$\\and Routing Inputs $\sg(h)$};
\node[smallbox, below=of grad1] (summary1) {Recompute Auxiliary $p_m$\\by Group Averaging};
\node[smallbox, below=of summary1] (agg1) {$G_k = \sum_m p_{mk}\tilde g_m$};
\node[smallbox, below=of agg1] (align1) {$\mathcal{L}_{\text{norm}}$ \\ (Direct gradient: router)};

\node[box, below=15mm of task1] (update1) {Shared Global Clip + Optimizer Step \\
Direct Non-router Gradient: $\nabla_{\theta_{\text{task}}}\mathcal{L}_{\text{task}}$ \\
Router Gradient: $\nabla_{\theta_r}(\mathcal{L}_{\text{task}} + \lambda\mathcal{L}_{\text{norm}})$};

\draw[arrow] (data1) -- (sample1);
\draw[arrow] (sample1) -- (batch1);
\draw[arrow] (batch1) -- (router1);
\draw[arrow] (router1) -- (experts1);
\draw[arrow] (experts1) -- (task1);
\draw[arrow] (task1) -- (update1);

\draw[arrow] (batch1.east) -- (grad1.west);
\draw[arrow] (grad1) -- (summary1);
\draw[arrow] (summary1) -- (agg1);
\draw[arrow] (agg1) -- (align1);
\draw[arrow] (align1.west) |- (update1.east);

\node[above=3mm of data1] {\textbf{GAR: Gradient-Aligned Routing}};

\node[box, right=85mm of data1] (data2) {Multi-task Dataloaders};
\node[box, below=of data2] (sample2) {Draw Matched Task Batches};
\node[box, below=of sample2] (batch2) {Mixed Training Batch};
\node[box, below=of batch2] (router2) {Per-Unit Router Probabilities};
\node[box, below=of router2] (experts2) {Configured Expert Dispatch};
\node[box, below=of experts2] (loss2) {Weighted Task Loss $\mathcal{L}_{\text{task}}$};
\node[box, below=of loss2] (update2) {Shared Global Clip + Optimizer Step};

\draw[arrow] (data2) -- (sample2);
\draw[arrow] (sample2) -- (batch2);
\draw[arrow] (batch2) -- (router2);
\draw[arrow] (router2) -- (experts2);
\draw[arrow] (experts2) -- (loss2);
\draw[arrow] (loss2) -- (update2);

\node[above=3mm of data2] {\textbf{Baseline: Standard MoE Routing}};

\end{tikzpicture}

\caption{
Full training-flow comparison between standard MoE routing and
gradient-aligned routing (GAR). The auxiliary branch detaches gradient observations
and routing inputs, then recomputes group probabilities using the group-mean
construction in Algorithm~\ref{alg:training}. Gradients shown at the optimizer
node specify the direct pre-clipping gradient inputs; task parameters include
all trainable non-router parameters. The main settings dispatch with top-4 routing over eight experts;
the top-1 extensions use the straight-through rule of
Appendix~\ref{appsubsec:top1_roberta_protocol}.
}
\label{fig:training_flow}
\end{figure}
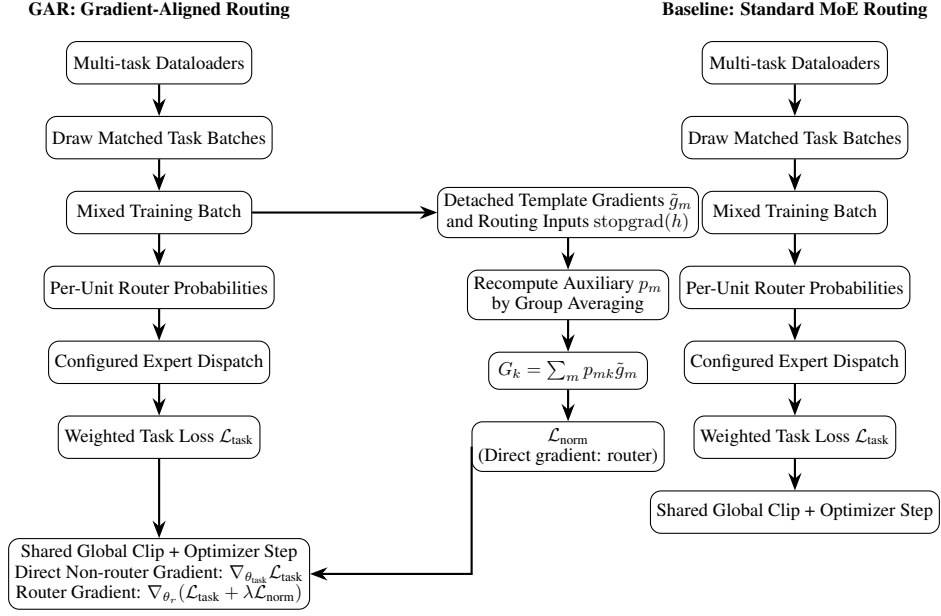

\clearpage
\subsection{Training Step}

\begin{algorithm}[H]
\caption{Gradient-space partitioning training with group observations}
\label{alg:training}
\footnotesize
\begin{algorithmic}[1]
\STATE Split each task batch into same-task example groups using the configured group size shared by all methods; the final group may be shorter. Index these routed units by $m$.
\STATE Run the ordinary forward pass with the configured per-unit gating and evaluate $\mathcal{L}_{\text{task}}$ with equal task weights and within-task sample weights (Appendix~\ref{appsubsec:training_semantics}). Retain routing inputs $h$ and padding masks for the auxiliary branch.
\STATE From the same training batch and task loss as Baseline, sum corresponding expert-gradient entries in a common parameter template to form $g_m=\sum_e\nabla_{\theta_e}\ell_m$, and detach: $\tilde g_m=\sg(g_m)$.
\STATE Recompute differentiable router probabilities from $\sg(h)$ using the configured routing rule. For token-routed FFN experts, average the configured token probabilities over each example's non-padding tokens to obtain $q_n$; for sequence-routed classification-head experts, $q_n$ is the configured per-example probability vector.
\STATE Form $p_m=|m|^{-1}\sum_{n\in m}q_n$, with no additional top-$k$ truncation after group averaging. Thus $\sum_k p_{mk}=1$.
\STATE Form surrogate aggregates $G_k=\sum_m p_{mk}\tilde g_m$.
\STATE Compute $\mathcal{L}_{\text{norm}}=-\sum_k \|G_k\|^2/(d_k(P)+\epsilon)$, where $d_k(P)=\sum_m p_{mk}$.
\STATE Form the pre-clipping gradients: all trainable non-router parameters receive gradients only from $\mathcal{L}_{\text{task}}$; router parameters receive gradients from $\mathcal{L}_{\text{task}}+\lambda\mathcal{L}_{\text{norm}}$.
\STATE Apply one global-norm clip jointly to all trainable gradients, then take the optimizer step.
\end{algorithmic}
\end{algorithm}

\subsection{Route-Aligned Observation Construction}
\label{app:route_aligned_observations}

The partition criterion accepts paired gradient and gate observations at any
granularity. Besides the same-task groups used in all reported experiments, it
applies to individual examples or tokens as described below; no reported
experiment uses these finer units.

\paragraph{Classification-head observations.}
For example $i$, let $\ell_i$ be its classification loss and let
$\tilde g_i=\sg(\sum_e\nabla_{\theta_e}\ell_i)$ sum corresponding
expert-gradient entries in the common template. The matching row $p_i$ is the example's
configured routing distribution, recomputed from its detached pooled routing
input. All examples in the optimizer update contribute one observation each.
Thus $M$ is the number of examples and group averaging is replaced by the
individual gate rows.

\paragraph{FFN observations.}
For valid token $t$ in example $i$, we use its local contribution to the
common-template gradient of $\ell_i$, summing corresponding parameter
contributions across experts. Each token row is paired with its
configured token gate $p_{it}$. For the position-wise expert layers, summing
these contributions over tokens recovers the example's common-template
gradient. Padding positions are excluded. Token observations can be subsampled
uniformly without replacement across the optimizer update, with the same
indices selecting the gradient factors and gate rows.

\paragraph{Exact factorized Gram computation.}
For expert $e$ and local linear-block slot $b$, let $x_{meb}$ be the actual
input and $\delta_{meb}$ the output derivative with respect to the example
loss. The template weight-gradient slot is
$v_{mb}=\sum_e\delta_{meb}x_{meb}^{\top}$. Its Gram expansion gives
\[
W_{ij}=\sum_b\sum_{e,f}
\langle\delta_{ieb},\delta_{jfb}\rangle\langle x_{ieb},x_{jfb}\rangle
+\sum_{b\,\text{with trainable bias}}\sum_{e,f}
\langle\delta_{ieb},\delta_{jfb}\rangle.
\]
The sums include all expert pairs within each corresponding template slot.
Repeated applications of a physical block contribute their cross-use terms.
Inputs retain the actual dropout realization; backward factors include
expert scales and gates. Factors are collected per example rather than for the
micro-batch mean loss. This evaluates the common-template Gram without
materializing an observation-by-parameter gradient matrix.

For LoRA expert $e$, write $y_{me}=s_eB_eA_ex_{me}$,
$z_{me}=A_ex_{me}$, and $\delta_{me}=\partial\ell_m/\partial y_{me}$.
The separate $A$ and $B$ template slots give
\begin{align*}
W^{\rm LoRA}_{ij}=\sum_{e,f}s_es_f\big[&
\langle B_e^{\top}\delta_{ie},B_f^{\top}\delta_{jf}\rangle
\langle x_{ie},x_{jf}\rangle\\
&+\langle\delta_{ie},\delta_{jf}\rangle
\langle z_{ie},z_{jf}\rangle\big].
\end{align*}
Each $x_{me}$ includes that expert's input dropout. This is the Gram of the
summed template gradients defined in Section~\ref{Sec:Euc}.

The collected Gram matrix is detached. Differentiable per-unit gate rows
then define the same objective,
\[
\mathcal L_{\rm norm}(P;W)
=-\sum_k\frac{[P^\top W P]_{kk}}{\sum_m p_{mk}+\epsilon}.
\]
Recomputing gates from detached pre-router inputs, after any input
normalization, directs its auxiliary
gradient to the router. Task-gradient accumulation, equal-task weighting,
global clipping, and the optimizer step retain the training semantics of
Appendix~\ref{appsubsec:training_semantics}.

\end{document}